\documentclass[twoside,11pt]{article}

\usepackage{blindtext}

\usepackage[preprint]{jmlr2e}

\usepackage{amsmath}
\usepackage{enumitem}
\usepackage[ruled,vlined]{algorithm2e}
\usepackage{subcaption}
\usepackage{booktabs}
\usepackage{multirow}

\usepackage{lastpage}
\jmlrheading{00}{0000}{1-\pageref{LastPage}}{00/00; Revised 00/00}{00/00}{21-0000}{Zheng Wang and Nizar Bouguila}
\makeatletter
\def\ps@jmlrtps{%
  \let\@mkboth\@gobbletwo
  \def\@oddhead{}%
  \def\@oddfoot{}
  \def\@evenhead{}%
  \def\@evenfoot{}%
}
\makeatother

\ShortHeadings{Vectorized Bayes for LDTA}{Wang and Bouguila}
\firstpageno{1}

\begin{document}

\title{Vectorized Bayesian Inference for Latent Dirichlet-Tree Allocation}

\author{\name Zheng Wang \email zheng.wang.20211@mail.concordia.ca \\
       \addr Concordia Institute for Information Systems Engineering\\
       Concordia University\\
       Montreal, QC H3G 1M8, Canada
       \AND
       \name Nizar Bouguila \email nizar.bouguila@concordia.ca \\
       \addr Concordia Institute for Information Systems Engineering\\
       Concordia University\\
       Montreal, QC H3G 1M8, Canada}

\editor{My editor}

\maketitle

\begin{abstract}
Latent Dirichlet Allocation (LDA) is a foundational model for discovering latent thematic structure in discrete data, but its Dirichlet prior cannot represent the rich correlations and hierarchical relationships often present among topics.
We introduce the framework of Latent Dirichlet-Tree Allocation (LDTA), a generalization of LDA that replaces the Dirichlet prior with an arbitrary Dirichlet-Tree (DT) distribution.
LDTA preserves LDA’s generative structure but enables expressive, tree-structured priors over topic proportions.
To perform inference, we develop universal mean-field variational inference and Expectation Propagation, providing tractable updates for all DT.
We reveal the vectorized nature of the two inference methods through theoretical development, and perform fully vectorized, GPU-accelerated implementations.
The resulting framework substantially expands the modeling capacity of LDA while maintaining scalability and computational efficiency.
The source code is available at https://github.com/intaiyjw/ldta.
\end{abstract}


\section{Introduction}

Probabilistic topic modeling has become one of the central frameworks for understanding latent structure in large, unorganized collections of discrete data.
As digital text, image annotations, scientific literature, biomedical data, and social media content continue to grow at unprecedented scale, tools for discovering hidden themes and interpretable patterns have become indispensable in machine learning and data mining \citep{blei2012probabilistic}.
Among these tools, Latent Dirichlet Allocation (LDA) has played a foundational role.
Since its introduction by \cite{blei:03}, LDA has remained one of the most widely used generative models for unsupervised learning of thematic structure. 
Its influence extends far beyond text analysis: LDA has been applied in genetics \citep{Liu2016TopicModeling}, marketing \citep{hofmann1999plsi}, computer vision \citep{feifei2005bayesian}, recommendation systems \citep{wang2011collab}, and many other domains where observed data can be interpreted as arising from mixtures of latent components.

The original LDA model assumes that each document is represented by a probability vector over latent topics, and each topic is a probability distribution over a vocabulary of terms \citep{blei:03}. 
Crucially, LDA places a Dirichlet prior over the document-level topic proportions. 
This choice of prior has several attractive properties: it is conjugate to the multinomial distribution, it facilitates tractable posterior updates, and it induces a simple geometric interpretation on the simplex \citep{bishop:06, gelman2013bayesian}. 
However, the Dirichlet prior also imposes structural limitations. 
Although its parameters can control concentration and sparsity, the Dirichlet distribution cannot express fine-grained correlations among topic components beyond the inherent negative dependence due to the simplex constraint \citep{wallach2009rethinking}. 
As a result, while LDA is simple and effective, it is often too rigid to capture the complex hierarchical or clustered topic structures that arise in real data.

This limitation has led to numerous extensions and generalizations of LDA. 
Correlated Topic Models \citep{blei2007ctm} introduced a logistic–normal prior to encode richer covariance structures, while hierarchical topic models such as the Hierarchical Pachinko Allocation Model \citep{li2006pachinko} and the nested Chinese restaurant process \citep{blei2010nested} allow for topic hierarchies and tree structures. 
Although these models provide greater flexibility, they typically sacrifice conjugacy and introduce substantial computational overhead. 
Many such models require non-conjugate variational inference or sampling procedures, which can be computationally expensive and difficult to scale \citep{hoffman2013stochastic, teh2006hdp}.

A powerful yet underutilized alternative to the Dirichlet distribution is the Dirichlet-Tree (DT) distribution, conceptualized for the first time by \cite{connor1969dirichlet} and later revived in works by \cite{dennis:91} and \cite{minka:99}.
The Dirichlet-Tree distribution provides a natural way to encode structured correlations through a tree of nested Dirichlet random variables. 
Each internal node of the tree defines a sub-Dirichlet distribution governing how probability mass is split among its children, and the full distribution over the leaves (which lie on the simplex) arises from the multiplicative combination of these local branching decisions. 
This formulation has two particularly attractive properties: (1) it can express arbitrarily rich correlation patterns and hierarchical structures by choosing appropriate tree shapes, and (2) it retains many of the computational conveniences of the ordinary Dirichlet distribution, including local conjugacy and analytical tractability.

Given the advantages of the Dirichlet-Tree distribution, it is natural to ask whether the traditional LDA model can be generalized by replacing the Dirichlet prior with a Dirichlet-Tree prior. 
Doing so would allow topic models to encode complex prior structures—such as that certain topics should co-occur, that some topics form meaningful groups, or that certain topics follow hierarchical relationships—without abandoning the general framework of LDA.
Indeed, the special cases of Dirichlet-Tree, including Beta-Liouville \citep{GuptaRichards1987Liouville, BakhtiariBouguila2016BetaLiouville} and Generalized Dirichlet \citep{Fang1990Symmetric, Wong1998GenDirichlet, BakhtiariBouguila2014VariationalCount}, have been explored and utilized in recent topic modeling literature.
This idea motivates the framework developed in this paper, which we call Latent Dirichlet–Tree Allocation (LDTA). 
LDTA generalizes LDA by allowing an arbitrary Dirichlet-Tree to govern the topic proportions. 
It preserves the generative semantics of LDA, retains interpretability, and supports a broad class of structured priors suitable for modeling realistic topic correlations.
However, generalizing LDA to LDTA also introduces significant challenges in posterior inference. 
While LDA benefits from the conjugacy of the Dirichlet and multinomial distributions, LDTA introduces a hierarchical and potentially deep tree structure whose internal dependencies complicate updates. 
Efficient and scalable inference algorithms are therefore essential. 
Two influential approaches to approximate Bayesian inference are particularly relevant: Mean-Field Variational Inference (MFVI) and Expectation Propagation (EP).

Mean-Field Variational Inference (MFVI) \citep{jordan1999variational, wainwright2008graphical, blei2017variational}, introduced in the late 1990s and now widely used in modern machine learning, replaces the exact posterior with a factorized approximation chosen to optimize a lower bound on the log marginal likelihood (the evidence lower bound, or ELBO). 
Variational inference is deterministic, often fast, and well-suited for models where conjugacy provides closed-form updates. 
It has become a standard tool for scalable Bayesian modeling and is used in diverse settings ranging from Bayesian neural networks to probabilistic graphical models. 
For LDTA, variational inference enables the derivation of structured coordinate ascent algorithms that update each factor of the variational distribution in closed or semi-closed form.

On the other hand, Expectation Propagation (EP), introduced by \cite{minka2001expectation, Minka2002EPAspectModel}, takes a different approach.
Rather than optimizing a lower bound, EP iteratively approximates each non-conjugate factor in the model by projecting it onto a chosen exponential family distribution through moment matching. 
EP often provides more accurate approximations than MFVI, especially in models where the posterior is not well approximated by fully factorized distributions (mean-field) \citep{Minka2005Divergence}. 
EP has been used in Gaussian process classification \citep{HernandezLobato2016ScalableGPC_EP}, Bayesian logistic regression \citep{Vehtari2020EPWayOfLife}, and many other settings where non-conjugacy arises. 
For LDTA, EP offers a principled way to approximate the interactions between different nodes of the Dirichlet-tree, yielding improved posterior estimates at the cost of increased algorithmic complexity.

As modern machine learning applications frequently involve very large corpora, it is critical that inference algorithms for topic models leverage vectorized computation and GPU acceleration. 
Frameworks such as PyTorch \citep{paszke2019pytorch} have made it easy to express complex probabilistic updates in terms of batched tensor operations that can run efficiently on CUDA-capable GPUs. 
While traditional implementations of LDA rely on nested loops and sequential updates, a vectorized implementation can exploit parallelism across documents, topics, and tree nodes. 
Such acceleration is essential for large datasets where non-vectorized methods would be prohibitively slow.

This paper develops both mean-field variational inference and Expectation Propagation algorithms for LDTA and implements them in a fully vectorized fashion that exploits GPU parallelism. 
The resulting algorithms can perform efficient inference even for large corpora and complex Dirichlet-tree structures.
Compared with scalar or loop-based implementations, the vectorized approach drastically reduces runtime and allows the algorithm to scale to modern datasets.
The remainder of the paper proceeds as follows.
We develop an improved and comprehensive formulation of Dirichlet-Tree in Section 2, revealing the properties of Dirichlet-Tree and preparing it to integrate into LDA, mean-field variational inference and expectation propagation.
In Section 3, we describe the LDTA model, which is a generalization of LDA by extending the Dirichlet prior to an arbitrary Dirichlet-Tree distribution.
In Section 4, a vectorized universal mean-field variational inference method is derived for any LDTA model.
In Section 5, we interpret the idea and demonstrate the implementation steps of the vectorized Expectation Propagation for any LDTA model.
In Section 6, we choose three typical Dirichlet-Tree priors: the conventional Dirichlet, the Beta-Liouville and the Generalized Dirichlet to implement our models.
We validate and compare our methods by a range of experiments including document modeling, document and image classification, and RNA-sequencing in bioinformatics.
Finally, in Section 7, we conclude our paper and discuss the potential future directions.

\section{Dirichlet-Tree}

This section serves as a formal and comprehensive introduction to the Dirichlet-Tree distribution.
Dirichlet-Tree was developed as a generalization of Dirichlet distribution by \cite{dennis:91} and further explored by \cite{minka:99}.
We further prove the exponential form of Dirichlet-Tree distribution, and the triple forms of Dirichlet-Tree are concluded: the node form, the general form and the exponential form.
As a result, we show that Dirichlet-Tree remains a conjugate prior to the multinomial likelihood with a transformed parameter updating.
Finally, several important concepts including Bayesian operator are proposed to further assist the derivation of vectorized inferences for Latent Dirichlet-Tree Allocation.

\subsection{Tree}

The tree is a well-studied data structure in computer science.
It is useful to introduce several tree-related notations before deriving the Dirichlet-Tree distribution.
Figure~\ref{fig:tree} shows a general tree with height values of 3.
A tree consists of multiple nodes and branches (or edges).
We refer to any node $\lambda$ that has multiple child nodes, including the root node, as an internal node; and the set for all internal nodes is denoted by $\boldsymbol{\Lambda}$.
In particular, $\boldsymbol{\Lambda}^{\backslash r}$ denotes the set of all internal nodes except the root node.
Any node $\omega$ that does not have a child node is called a terminal node (or leaf) and the set of all terminal nodes is $\boldsymbol{\Omega}$.

Between a parent node and one of its child nodes is a branch (or edge).
We use $t \mid s$ to denote the branch starting from node $s$ and pointing to node $t$.
An indicator function $\delta_{t \mid s}(\omega)$ whose domain is $\boldsymbol{\Omega}$ is defined and assigned to each branch:

\begin{equation*}
\delta_{t \mid s}(\omega) =
\begin{cases}
1, & \text{if the branch } t|s \text{ leads to leaf } \omega, \\[6pt]
0,        & \text{otherwise}
\end{cases}
\quad \omega \in \boldsymbol{\Omega}, \; s \in \boldsymbol{\Lambda}
\label{eq:tree branch indicator}
\end{equation*}

Sometimes, the starting node of a branch does not matter, and we denote this branch that directly points to node $s$ as $s|*$.
We define the function $l(s),\;s \in \boldsymbol{\Lambda} \cup \boldsymbol{\Omega}$ as the number of terminal nodes that a node $s$ travels to. 
When $s$ is the root node, we have:
\begin{equation*}
    l(r) = \sum_{t|r} \sum_{\omega \in \boldsymbol{\Omega}} \delta_{t|r}(\omega) = K
\end{equation*}
where $K = |\boldsymbol{\Omega}|$ is the number of all terminal nodes.
For $s \in \boldsymbol{\Omega}$, we have:
\begin{equation*}
    l(s) = \sum_{\omega \in \boldsymbol{\Omega}} \delta_{s|*}(\omega) = 1, \qquad s \in \boldsymbol{\Omega}
\end{equation*}
For any $s \in \boldsymbol{\Lambda}^{\backslash r}$:
\begin{equation*}
    l(s) = \sum_{\omega \in \boldsymbol{\Omega}} \delta_{s|*}(\omega) = \sum_{t|s} \sum_{\omega \in \boldsymbol{\Omega}} \delta_{t|s}(\omega), \qquad s \in \boldsymbol{\Lambda}^{\backslash r}
\end{equation*}
We further define the function $c(s),\;s \in \boldsymbol{\Lambda}$ as the number of direct child nodes (or branches) under an internal node.
For example, in Figure~\ref{fig:tree}, $c(\lambda_2) = 3$.
The following theorem is obvious and useful to mention:
\begin{theorem}\label{thm:tree_relation}
In a tree structure, for any internal node $s \in \boldsymbol{\Lambda}$,
\begin{equation}
l(s) = c(s) + \sum_{t \mid s} \left( l(t) - 1 \right), 
\qquad s \in \boldsymbol{\Lambda}.
\label{eq:tree_relation}
\end{equation}
\end{theorem}

\begin{figure}[htbp]
    \centering
    \includegraphics[width=0.75\linewidth]{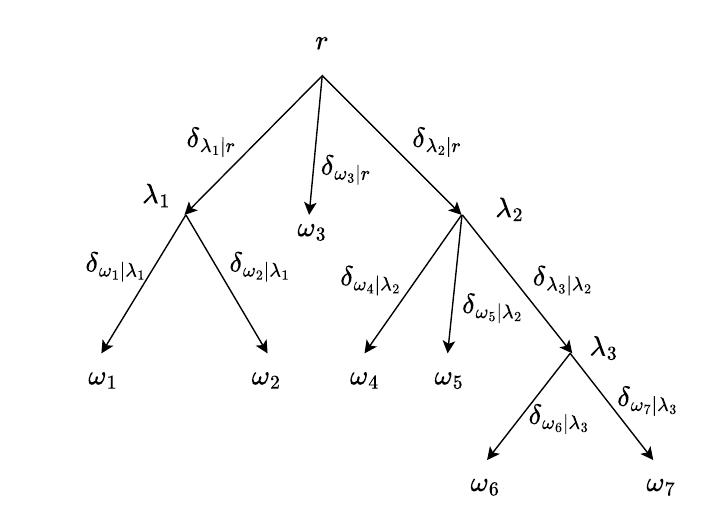}
    \caption{Notation for a general tree structure}
    \label{fig:tree}
\end{figure}

\subsection{Hierarchical Multinomial and Dirichlet-Tree}

\begin{figure}[htbp]
    \centering
    \includegraphics[width=1.0\linewidth]{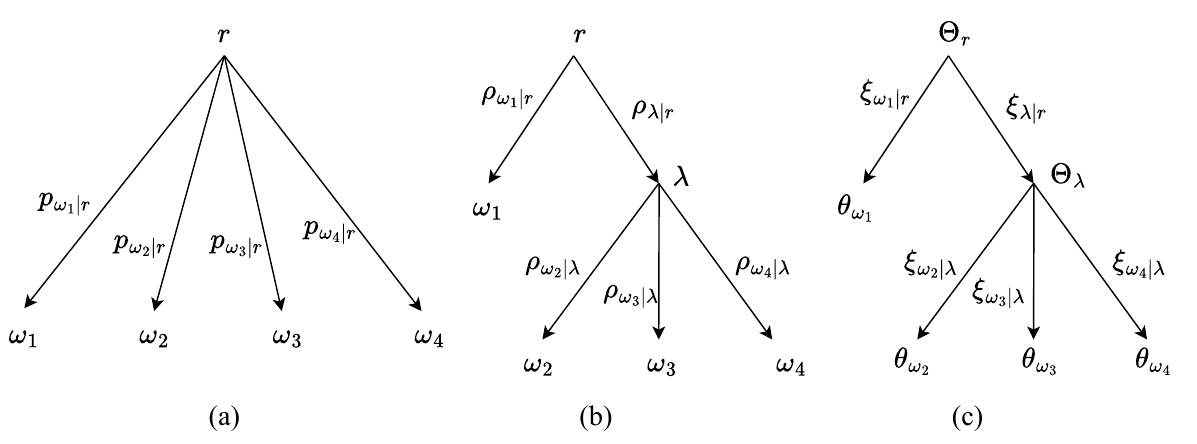}
    \caption{(a) Conventional Multinomial; (b) Hierarchical Multinomial; (c) Dirichlet-Tree}
    \label{fig:h_multinomial_to_dt}
\end{figure}

The Dirichlet distribution has been widely used as a conjugate prior to multinomial under conventional parameterization \citep{minka:99, blei:03, bishop:06}.
\cite{minka:99} described the sample of a conventional multimomial as the outcome of a K-sided die.
In this analogy, a K-sided die is rolled once and one out of $K$ possible outcomes is produced, where each outcome is assigned to a possibility matter $p_k, \; k=1,2,\cdots,K$.
As in Figure~\ref{fig:h_multinomial_to_dt} (a), we can formally represent the conventional multinomial by assigning each $p_k$ to a branch in a tree with height 1.
The probability mass function (PMF) is given by:
\begin{equation*}
    p(\omega \mid \boldsymbol{p}) = 
    \prod_{\omega \mid r} {p_{\omega \mid r}}^{\delta_{\omega \mid r}(\omega)},
    \qquad \omega \in \boldsymbol{\Omega},\;\boldsymbol{p} \in \Delta^{K-1},\;r \text{ is root}
\end{equation*}
Correspondingly, its conjugate Dirichlet can be represented by the same tree with each branch assigned a parameter $\alpha_{\omega \mid r}$, and each leaf assigned the component of the $(K-1)$-dimensional simplex.
The probability density function (PDF) follows:
\begin{gather*}
    \text{Dir}(\boldsymbol{\theta} \mid \boldsymbol{\alpha}_r) = 
    \frac{1}{B(\boldsymbol{\alpha}_r)} \prod_{\omega \in \boldsymbol{\Omega}} {\theta_{\omega}}^{\alpha_{\omega \mid r} - 1},
    \qquad \boldsymbol{\theta} \in \Delta^{|\boldsymbol{\Omega}|-1}, \; \boldsymbol{\alpha}_{r} \in \mathbb{R}_{>0}^{K}, \; r \text{ is root} \\[6pt]
    B(\boldsymbol{\alpha}_r) = \frac{\prod_{\omega \mid r} \Gamma\left(\alpha_{\omega \mid r}\right)}{\Gamma\left(\sum_{\omega \mid r} \alpha_{\omega \mid r}\right)}
\end{gather*}
where $\Delta^{|\boldsymbol{\Omega}|-1}$ is the $(|\boldsymbol{\Omega}|-1)$-dimensional simplex defined by:
\begin{equation*}
    \Delta^{|\boldsymbol{\Omega}|-1} = \left\{ \boldsymbol{\theta} \in \mathbb{R}^{|\boldsymbol{\Omega}|} \;\middle|\;
    \theta_{\omega} \ge 0, \; \sum_{\omega \in \boldsymbol{\Omega}} \theta_{\omega} = 1 \right\}
\end{equation*}
In the case of Dirichlet, $|\boldsymbol{\Omega}| = K$.
Despite its convenience as a conjugate prior for multinomial, Dirichlet suffers from the following limitations:
(i) all components share a common variance parameter although each component has its own mean \citep{minka:99}, and
(ii) the covariance between any two components is strictly negative \citep{mosimann:62}.
To solve these limitations, \cite{dennis:91} developed the Dirichlet-Tree distribution as a generalization of Dirichlet.

Instead of rolling a $K$-dimensional die once, the sample of a multinomial can also be thought of as rolling several different dice for finite times, as shown in Figure~\ref{fig:h_multinomial_to_dt} (b).
This finite stochastic process results in a hierarchical multinomial under a tree-like parameterization with the following PMF:
\begin{equation}
    p(\omega \mid \boldsymbol{\rho}) = 
    \prod_{s \in \boldsymbol{\Lambda}} \prod_{t \mid s} {\rho_{t \mid s}}^{\delta_{t \mid s}(\omega)}, \qquad
    \omega \in \boldsymbol{\Omega}, \; \boldsymbol{\rho}_s \in \Delta_{s}^{c(s)-1}, \; s \in \boldsymbol{\Lambda}
    \label{eq:hierarchical_multinomial}
\end{equation}
Consequently, the corresponding prior for the hierarchical multinomial is based on the collection of $|\boldsymbol{\Lambda}|$ independent Dirichlets endowed with a tree structure (Figure~\ref{fig:h_multinomial_to_dt} (c)).
And the PDF is given by:
\begin{equation}
    p(\boldsymbol{\rho} \mid \boldsymbol{\xi}) = 
    \prod_{s \in \boldsymbol{\Lambda}} \frac{1}{B(\boldsymbol{\xi}_s)} \prod_{t \mid s} {\rho_{t \mid s}}^{\xi_{t | s} - 1}, \qquad \boldsymbol{\rho}_s \in \Delta_{s}^{c(s) - 1}, \; \boldsymbol{\xi}_s \in \mathbb{R}_{>0}^{c(s)}, \; s \in \boldsymbol{\Lambda}
    \label{eq:pdf_dirichlets}
\end{equation}
This is not the final form of Dirichlet-Tree distribution yet.
What we want is a distribution over $(|\boldsymbol{\Omega}| - 1)$-dimensional simplex assigned to the set of terminal nodes.
A transformation from $\boldsymbol{\rho} \in \prod_{s \in \boldsymbol{\Lambda}} \Delta_{s}^{c(s) - 1}$ to $\boldsymbol{\theta} \in \Delta^{|\boldsymbol{\Omega}| - 1}$ is needed to achieve this distribution.
As \eqref{eq:hierarchical_multinomial} implies, the component $\theta_{\omega}$ assigned to a terminal node is equal to the product of all probability mass assigned to the branches that lead to the terminal node.
We denote this mapping as follows:
\begin{equation}
\begin{aligned}
T: \prod_{s \in \boldsymbol{\Lambda}} \Delta_{s}^{c(s) - 1} 
&\to \Delta^{|\boldsymbol{\Omega}| - 1}, 
\quad \boldsymbol{\rho} \mapsto T(\boldsymbol{\rho}) \\[1ex]
\theta_{\omega}
= \prod_{s \in \boldsymbol{\Lambda}}\prod_{t \mid s} &{\rho_{t \mid s}}^{\delta_{t\mid s}(\omega)}, \qquad \omega \in \boldsymbol{\Omega}\\[1ex]
T(\boldsymbol{\rho}) 
&= \bigcirc_{s \in \boldsymbol{\Lambda}} T_s(\boldsymbol{\rho}_s)
\label{eq:trans_rho_theta}
\end{aligned}
\end{equation}
where $T$ is the composition of all sub-mappings $T_s$ under each internal node.
Conversely, to show the inverse transform $T^{-1}$, we first define the node mass $\Theta_s$:
\begin{equation*}
\Theta_s(\boldsymbol{\theta}) =
\begin{cases}
1, & s \text{ is the root}, \\[6pt]
\sum_{\omega \in \boldsymbol{\Omega}} \delta_{s \mid *}\left( \omega \right) \theta_{\omega}, & s \in \boldsymbol{\Lambda}^{\backslash r} \cup \boldsymbol{\Omega} 
\end{cases}
\end{equation*}
Obviously, $\Theta_{\omega} = \theta_{\omega}$ when $\omega \in \boldsymbol{\Omega}$.
Thus, we have the inverse transform as follows:
\begin{equation}
\begin{aligned}
T^{-1}: \Delta^{|\boldsymbol{\Omega}| - 1} 
&\to \prod_{s \in \boldsymbol{\Lambda}} \Delta_{s}^{c(s) - 1}, 
\quad \boldsymbol{\theta} \mapsto T^{-1}(\boldsymbol{\theta}) \\[1ex]
\rho_{t \mid s} 
&= \frac{\Theta_t}{\Theta_s}, \quad s \in \boldsymbol{\Lambda}
\end{aligned}
\label{eq:trans_theta_rho}
\end{equation}
Now, we have a pair of inverse transforms and assume the transform to be smooth.
We go back to \eqref{eq:trans_rho_theta} and derive the Jacobian of $T$.
For $\forall s \in \boldsymbol{\Lambda}$, we denote the sub-transform by free-coordinates representation:
\begin{equation*}
    T_s(\boldsymbol{\rho}_s): (\rho_{1 \mid s}, \dots, \rho_{c(s) - 1 \mid s}, \Theta_s) \;\mapsto\; (\Theta_{1}, \dots, \Theta_{c(s)})
\end{equation*}
Note that we enumerate the set of immediate child nodes of node $s$ by $1, \dots, c(s)$.
For $i = 1, \dots, c(s)-1$ and $j = 1, \dots, c(s)-1$:
\begin{equation*}
    \frac{\partial \Theta_i}{\partial \rho_{j \mid s}} = \frac{\partial \left( \rho_{i \mid s} \Theta_s \right)}{\partial \rho_{j \mid s}} = \Theta_s \mathbf{1}_{\{i=j\}}, \qquad \frac{\partial \Theta_{c(s)}}{\partial \rho_{j \mid s}} = \frac{\partial \left( \left( 1 - \sum_{l=1}^{c(s)-1} \rho_{l \mid s} \right) \Theta_s \right)}{\partial \rho_{j \mid s}} = -\Theta_s
\end{equation*}
For the column corresponding to $\Theta_s$:
\begin{equation*}
    \frac{\partial \Theta_i}{\partial \Theta_s} = \rho_{i \mid s}, \qquad i = 1, \dots, c(s)
\end{equation*}
Thus, we have the Jacobian with the form:
\begin{equation*}
J_{T_s} =
\begin{pmatrix}
\displaystyle \Theta_s I_{c(s)-1} & b \\[1ex]
\displaystyle c^{\top} & d
\end{pmatrix}
\end{equation*}
where $I_{c(s)-1}$ is ($c(s)-1$)-dimensional identity, $b = (\rho_{1 \mid s}, \dots, \rho_{c(s)-1 \mid s})^{\top}$, $c^{\top} = (-\Theta_s, \dots, -\Theta_s)$, $d = \rho_{c(s) \mid s}$.
Thus, the determinant derives as follows:
\begin{equation}
    \det(J_{T_s}) = \det(\Theta_s I_{c(s)-1}) \left( d - c^{\top} \left(\Theta_s I_{c(s)-1}\right)^{-1}b \right) = {\Theta_s}^{c(s)-1}
    \label{eq:det_jacobian}
\end{equation}
According to the change-of-variable theorem, what we call the node form of Dirichlet-Tree is derived as follows by combining \eqref{eq:pdf_dirichlets}, \eqref{eq:trans_theta_rho} and \eqref{eq:det_jacobian}:
\begin{equation}
\begin{aligned}
    \text{DT}(\boldsymbol{\theta} \mid \boldsymbol{\xi}) 
    &= p\left(T^{-1}(\boldsymbol{\theta})\mid\boldsymbol{\xi}\right)\;\left|\det(J_T)\right|^{-1} \\[1ex]
    &= \prod_{s \in \boldsymbol{\Lambda}} \frac{1}{B(\boldsymbol{\xi}_{s})} \left( \frac{1}{\Theta_s} \right)^{c(s) - 1} \prod_{t \mid s} \left( \frac{\Theta_t}{\Theta_s} \right)^{\xi_{t \mid s} - 1},
    \qquad \boldsymbol{\theta} \in \Delta^{|\boldsymbol{\Omega}|-1},\;\boldsymbol{\xi} \in \mathbb{R}_{>0}^{\sum_{s \in \boldsymbol{\Lambda}}c(s)}
\end{aligned}
\label{eq:node_form}
\end{equation}
It is easy to derive a more general form of Dirichlet-Tree from \eqref{eq:node_form}.
Based on this general form, we give a formal definition of Dirichlet-Tree distribution.
\begin{definition}[General form of Dirichlet-Tree]
With respect to a tree structure where $r$ is the root node, $\boldsymbol{\Omega}$ is the set of terminal nodes and $\boldsymbol{\Lambda}^{\backslash r}$ is the set of the remaining nodes, a Dirichlet-Tree distribution is a probability measure on $\Delta^{|\boldsymbol{\Omega}|-1}$ with PDF
\begin{equation}
    \mathrm{DT}(\boldsymbol{\theta} \mid \boldsymbol{\xi}) =
    \frac{1}{B(\boldsymbol{\xi}_r)} \prod_{\omega \in \boldsymbol{\Omega}} {\theta_{\omega}}^{\xi_{\omega \mid *} - 1} \left( \prod_{s \in \boldsymbol{\Lambda}^{\backslash r}} \frac{1}{B(\boldsymbol{\xi}_s)} {\Theta_s}^{\xi_{s \mid *} - \xi_{* \mid s}} \right)
\end{equation}
where $\xi_{t \mid s} \in \mathbb{R}_{>0}$ is assigned to each branch for $s \in \boldsymbol{\Lambda}$, and
\begin{equation*}
    \xi_{* \mid s} = \sum_{t \mid s} \xi_{t \mid s}
\end{equation*}
and $B(\cdot)$ is the multivariate beta function.
\end{definition}

\subsection{Exponential Form and Conjugacy}

We further prove that a Dirichlet-Tree distribution belongs to the exponential family.
The exponential form of a Dirichlet-Tree is given as follows:
\begin{theorem}[Exponential form of Dirichlet-Tree]
\begin{equation}
\begin{aligned}
    \mathrm{DT}(\boldsymbol{\theta} \mid \boldsymbol{\xi}) 
    &= \prod_{s \in \boldsymbol{\Lambda}} \frac{1}{B(\boldsymbol{\xi}_s)} \prod_{t \mid s} \left( \frac{\Theta_t}{\Theta_s} \right)^{\xi_{t \mid s} - l(t)} \\[1ex]
    &= \exp\left(\sum_{s \in \boldsymbol{\Lambda}}\sum_{t \mid s}\left(\xi_{t \mid s}-l(t)\right)\log\left(\frac{\Theta_t}{\Theta_s}\right) - \log\prod_{s \in \boldsymbol{\Lambda}}B(\boldsymbol{\xi}_s)\right)
\end{aligned}
\label{eq:expo_form}
\end{equation}
where $l(t)$ is the number of leaf nodes that the node $t$ leads to.
\end{theorem}
\begin{proof}
Dividing \eqref{eq:node_form} by \eqref{eq:expo_form} yields
\begin{equation*}
\begin{aligned}
    &\prod_{s \in \boldsymbol{\Lambda}}\left(\frac{1}{\Theta_s}\right)^{c(s)-1} \prod_{t \mid s}\left(\frac{\Theta_t}{\Theta_s}\right)^{l(t)-1} \\[1ex]
    =& \left(\prod_{s \in \boldsymbol{\Lambda}}\left(\frac{1}{\Theta_s}\right)^{c(s)+\sum_{t \mid s}\left(l(t)-1\right) -1}\right) \left(\prod_{t \in \boldsymbol{\Lambda}^{\backslash r}\cup\boldsymbol{\Omega}}\Theta_t^{l(t)-1}\right) \\[1ex]
    =& \left(\prod_{s \in \boldsymbol{\Lambda}}\left(\frac{1}{\Theta_s}\right)^{l(s)-1}\right) \left(\prod_{t \in \boldsymbol{\Lambda}^{\backslash r}} \Theta_t ^{l(t)-1}\right) \\[1ex]
    =& \;1
\end{aligned}
\end{equation*}
The derivation uses Theorem~\ref{thm:tree_relation} and the fact that 
\begin{equation*}
    \Theta_{\text{root}} = 1
\end{equation*}
\end{proof}
The natural parameter, sufficient statistics, normalizer, and base measure are respectively denoted as:
\begin{gather*}
    \eta_{t\mid s}(\boldsymbol{\xi}) = \xi_{t\mid s} - l(t), \quad s\in\boldsymbol{\Lambda} \\[1ex]
    u_{t\mid s}(\boldsymbol{\theta}) = \log\left(\frac{\Theta_t}{\Theta_s}\right), \quad s\in\boldsymbol{\Lambda} \\[1ex]
    g\left(\boldsymbol{\eta}(\boldsymbol{\xi})\right) = \prod_{s\in\boldsymbol{\Lambda}}B(\boldsymbol{\xi}_s) \\[1ex]
    h(\boldsymbol{\theta}) = 1
\end{gather*}
Given the exponential form, all properties of the exponential family listed in Appendix~\ref{app:properties_expo} will apply.
\begin{corollary}[Expectation of sufficient statistics]\label{cor:mean_u}
\begin{gather*}
    \mathbb{E}_{\mathrm{DT}(\boldsymbol{\theta}\mid\boldsymbol{\xi})}\left[\log\left(\frac{\Theta_t}{\Theta_s}\right)\right]
    = \frac{\partial \left(\log\prod_{s\in\boldsymbol{\Omega}}B(\boldsymbol{\xi}_s)\right)}{\partial \xi_{t\mid s}}
    = \psi\left(\xi_{t\mid s}\right) - \psi\left(\xi_{*\mid s}\right), \quad s \in \boldsymbol{\Lambda} \\[1ex]
    \xi_{*\mid s} = \sum_{t\mid s} \xi_{t\mid s}
\end{gather*}
\end{corollary}
We proceed to show the conjugacy of Dirichlet-Tree to any conventional $|\boldsymbol{\Omega}|$-dimensional multinomial.
Combining the pair of the transforms \eqref{eq:trans_rho_theta} and \eqref{eq:trans_theta_rho}, we can represent each component $\theta_{\omega}$ assigned to a terminal node by node mass and indicator functions assigned to non-terminal nodes and branches as follows:
\begin{equation}
    \label{eq:theta_by_node_mass}
    \theta_{\omega} = \prod_{s \in \boldsymbol{\Lambda}}\prod_{t \mid s}\left(\frac{\Theta_t}{\Theta_s}\right)^{\delta_{t \mid s}(\omega)}, \quad \omega \in \boldsymbol{\Omega}
\end{equation}
Let a Dirichlet-Tree distribution with $\boldsymbol{\Omega}$ be the prior to a $|\boldsymbol{\Omega}|$-dimensional multinomial.
Given a collection of multinomial samples $\boldsymbol{n} = (n_1, \dots, n_K)$ where $K=|\boldsymbol{\Omega}|$, the posterior is given as:
\begin{equation*}
\begin{aligned}
    p(\boldsymbol{\theta}\mid\boldsymbol{n}) 
    &\propto p(\boldsymbol{n}\mid\boldsymbol{\theta})p(\boldsymbol{\theta}) \\[1ex]
    &\propto \left(\prod_{\omega\in\boldsymbol{\Omega}}{\theta_{\omega}}^{n_\omega}\right)\left(\prod_{s\in\boldsymbol{\Lambda}}\prod_{t\mid s}\left(\frac{\Theta_t}{\Theta_s}\right)^{\xi_{t\mid s}-l(t)}\right) \\[1ex]
    &= \exp\left\{\sum_{\omega\in\boldsymbol{\Omega}}n_{\omega}\log\theta_{\omega} + \sum_{s\in \boldsymbol{\Lambda}}\sum_{t\mid s}\left(\xi_{t\mid s}-l(t)\right)\log\left(\frac{\Theta_t}{\Theta_s}\right)\right\} \\[1ex]
    &= \exp\left\{ \sum_{\omega\in\boldsymbol{\Omega}}n_{\omega}\sum_{s\in\boldsymbol{\Lambda}}\sum_{t\mid s}\delta_{t\mid s}(\omega)\log\left(\frac{\Theta_t}{\Theta_s}\right) + \sum_{s\in \boldsymbol{\Lambda}}\sum_{t\mid s}\left(\xi_{t\mid s}-l(t)\right)\log\left(\frac{\Theta_t}{\Theta_s}\right)\right\} \\[1ex]
    &= \exp\left\{ \sum_{s\in \boldsymbol{\Lambda}}\sum_{t\mid s}\left(\xi_{t\mid s} + \sum_{\omega\in\boldsymbol{\Omega}}n_{\omega}\delta_{t\mid s}(\omega) - l(t)\right)\log\left(\frac{\Theta_t}{\Theta_s}\right) \right\}
\end{aligned}
\end{equation*}
Now it is obvious that the posterior remains the same kind of Dirichlet-Tree with a transformed parameter updating.
For convenience, we denote the posterior Dirichlet-Tree as
\begin{equation*}
    \mathrm{DT}(\boldsymbol{\theta}\mid\boldsymbol{\xi}^{\prime}) = 
    \mathrm{DT}(\boldsymbol{\theta}\mid\boldsymbol{\xi}\oplus_{p}\boldsymbol{n}) = 
    \mathrm{DT}(\boldsymbol{\theta}\mid\boldsymbol{\xi}+\mathbf{D}_{p}\boldsymbol{n})
\end{equation*}
where we call $\oplus_p$ the Bayesian addition; and $\mathbf{D}_p$ the Dirichlet selection operator as $\mathbf{D}_p$ selects the components from the partitioned samples to add to the Dirichlet-Tree parameter corresponding to each branch in the tree structure; and
\begin{equation*}
    {\xi_{t\mid s}}^{\prime} = \;\xi_{t\mid s} + \sum_{\omega\in\boldsymbol{\Omega}}n_{\omega}\delta_{t\mid s}(\omega),\quad s\in \boldsymbol{\Lambda}
\end{equation*}
We have the following corollary by combining Corollary~\ref{cor:mean_u} and \eqref{eq:theta_by_node_mass}:
\begin{corollary}[Log-space expectation]\label{mean_log_theta}
\begin{equation*}
    \mathbb{E}_{\mathrm{DT}(\boldsymbol{\theta}\mid\boldsymbol{\xi})}\left[\log \theta_{\omega}\right]
    = \mathbb{E}_{\mathrm{DT}(\boldsymbol{\theta\mid\boldsymbol{\xi}})}\left[\sum_{s\in\boldsymbol{\Lambda}}\sum_{t\mid s}\delta_{t\mid s}(\omega)\log\left(\frac{\Theta_t}{\Theta_s}\right)\right]
    = \sum_{s\in\boldsymbol{\Lambda}}\sum_{t\mid s}\delta_{t\mid s}(\omega)\left(\psi(\xi_{t\mid s})- \psi(\xi_{*\mid s})\right)
\end{equation*}
\end{corollary}
Intuitively, the expectation of $\log\theta_{\omega}$ is equal to the sum of $\mathbb{E}_{\mathrm{DT}(\boldsymbol{\theta}\mid\boldsymbol{\xi})}\left[u_{t\mid s}\right]$ corresponding to all the branches that leads to $\omega$.

\subsection{Bayesian operator}

In this section, we propose several concepts based on the exponential form and conjugacy of Dirichlet-Tree to further assist the derivations of two approximate Bayesian inference algorithms introduced in the following sections.
In the Bayesian inference of generative graphical models like Latent Dirichlet Allocation \citep{blei:03}, it is common to have a substantial number of parameter updating processes.
Bayesian operator is proposed to represent this kind of processes in a compact way.
We first give the formal definition of Bayesian operator:
\begin{definition}[Bayesian operator]\label{def:bayes_operator}
Let $\mathcal{T}$ be the statistical manifold of Dirichlet-Tree.
In conjugacy to multinomial, the Bayesian operator for $\forall p \in \mathcal{T}$ transforms the $p$ as a prior to a posterior $p^{\prime}\in\mathcal{T}$ given the multinomial observation $\boldsymbol{n}$:
\begin{equation*}
    p^{\prime} = \mathbf{B}(\boldsymbol{n})\;p
\end{equation*}
\end{definition}
It is easy to derive:
\begin{equation*}
    \mathbf{B}(\boldsymbol{n}) = \left(\frac{g(\boldsymbol{\eta}(\boldsymbol{\xi}+\mathbf{D}_{p}\boldsymbol{n}))}{g(\boldsymbol{\eta}(\boldsymbol{\xi}))}\right)^{-1} \exp\left(\boldsymbol{n}^{\top}\log\boldsymbol{\theta}\right)
\end{equation*}
The Bayesian operator exhibits a closure nature and provides a new expression of conjugacy.
The mathematical expression in our following derivation of Expectation Propagation is largely simplified with the Bayesian operator.
Given the Bayesian operator, the relationship between the Dirichlet selection operator and the Bayesian operator is built, which is our central theorem:
\begin{theorem}[The central theorem]
    \begin{equation*}
        p(\boldsymbol{\theta}\mid\boldsymbol{\xi}+\mathbf{D}_p \boldsymbol{n}) = \mathbf{B}(\boldsymbol{n})p(\boldsymbol{\theta}\mid \boldsymbol{\xi})
    \end{equation*}
\end{theorem}
Even though the multinomial observation $\boldsymbol{n}$ is non-negative integers, $\boldsymbol{n}$ can be naturally extended to any real numbers in Bayesian operator that makes the posterior meaningful.
The Bayesian operator satisfied the following properties:
\begin{itemize}[leftmargin=*]
    \item \textbf{Additivity and Commutativity}
    \begin{equation*}
            \mathbf{B}(\boldsymbol{n}_1 + \boldsymbol{n}_2) = \mathbf{B}(\boldsymbol{n}_1) \mathbf{B}(\boldsymbol{n}_2) = \mathbf{B}(\boldsymbol{n}_2) \mathbf{B}(\boldsymbol{n}_1)
    \end{equation*} 
    \textbf{Proof}
    \begin{equation*}
        \mathrm{RHS}
        = \left( \frac{g\left(\boldsymbol{\eta}\left( \boldsymbol{\xi}+\mathbf{D}_p\left( \boldsymbol{n}_1 + \boldsymbol{n}_2 \right) \right)\right)}{g\left( \boldsymbol{\eta}(\boldsymbol{\xi}) \right)} \right)^{-1} \exp \left( (\boldsymbol{n}_2 + \boldsymbol{n}_1)^{\top} \log \boldsymbol{\theta} \right)
        = \mathrm{LHS} 
    \end{equation*}

    \item \textbf{Inverse operator}
    \begin{equation*}
        \mathbf{B}^{-1}(\boldsymbol{n}) = \mathbf{B}(-\boldsymbol{n})
    \end{equation*}

    \item \textbf{Base posteriors}
    
    For $\forall p(\boldsymbol{\theta}\mid\boldsymbol{\xi}) \in \mathcal{T}$, there exists a group of base posteriors $\{p(\boldsymbol{\theta}|\boldsymbol{\xi}+\mathbf{D}_p \boldsymbol{1}_{(k)})\}_K$:
    \begin{equation*}
    p(\boldsymbol{\theta}|\boldsymbol{\xi}+\mathbf{D}_p \boldsymbol{1}_{(k)}) = \mathbf{B}(\boldsymbol{1}_{(k)}) p(\boldsymbol{\theta}|\boldsymbol{\xi}),\quad k=1,\dots,K
    \end{equation*}
    where $K=|\boldsymbol{\Omega}|$ and $\boldsymbol{1}_{(k)}$ is the $K$-dimensional binary vector where $k^{th}$ element equals 1 and 0 otherwise.
    We also use $p(\boldsymbol{\theta}|\boldsymbol{\xi}+\mathbf{D}_p\boldsymbol{1}_{(\boldsymbol{k})})$ to represent the group of derived distributions $\{p(\boldsymbol{\theta}|\boldsymbol{\xi}+\mathbf{D}_p \boldsymbol{1}_{(k)})\}_K$.
\end{itemize}
\begin{corollary}[Expectation operator]
    \begin{equation*}
    \begin{aligned}
        \mathbb{E}_{p(\boldsymbol{\theta}|\boldsymbol{\xi})} \left[ \prod_{k=1}^{K} {\theta_k}^{n_k} \right]
        &= \int_{\Delta^{K-1}} \exp\left(\boldsymbol{n}^{\top}\log\boldsymbol{\theta}\right) \; p(\boldsymbol{\theta}|\boldsymbol{\xi}) \,\mathrm{d} \boldsymbol{\theta}  \\[1ex]
        &= \int_{\Delta^{K-1}} \left[ \frac{g(\boldsymbol{\eta}(\boldsymbol{\xi} + \mathbf{D}_p \boldsymbol{n}))}{g(\boldsymbol{\eta}(\boldsymbol{\xi}))} \mathbf{B}(\boldsymbol{n}) \right] p(\boldsymbol{\theta}|\boldsymbol{\xi}) \,\mathrm{d} \boldsymbol{\theta} \\[1ex]
        &= \frac{g(\boldsymbol{\eta}(\boldsymbol{\xi} + \mathbf{D}_p \boldsymbol{n}))}{g(\boldsymbol{\eta}(\boldsymbol{\xi}))} \int_{\Delta^{K-1}} p(\boldsymbol{\theta}|\boldsymbol{\xi} + \mathbf{D}_p \boldsymbol{n}) \,\mathrm{d}\boldsymbol{\theta} \\[1ex]
        &= \frac{g(\boldsymbol{\eta}(\boldsymbol{\xi} + \mathbf{D}_p \boldsymbol{n}))}{g(\boldsymbol{\eta}(\boldsymbol{\xi}))}  \\[1ex]
        &= \prod_{s\in\boldsymbol{\Lambda}}\frac{B\left(\boldsymbol{\xi}_s+\mathbf{D}_{p}\boldsymbol{n}\right)}{B\left(\boldsymbol{\xi}_s\right)} \\[1ex]
        &= \prod_{s\in\boldsymbol{\Lambda}} \left(\frac{\prod_{t\mid s}\Gamma\left({\xi_{t\mid s}}^{\prime}\right)}{\Gamma\left(\sum_{t\mid s}{\xi_{t\mid s}}^{\prime}\right)}\right) \left(\frac{\Gamma\left(\sum_{t\mid s}\xi_{t\mid s}\right)}{\prod_{t\mid s}\Gamma\left(\xi_{t\mid s}\right)}\right)  \\[1ex]
        &= \prod_{s\in\boldsymbol{\Lambda}}\frac{\Gamma\left(\sum_{t\mid s}\xi_{t\mid s}\right)}{\Gamma\left(\sum_{t\mid s}\xi_{t\mid s} + \sum_{t\mid s}\sum_{\omega\in\boldsymbol{\Omega}} n_{\omega}\delta_{t\mid s}(\omega)\right)} \prod_{t\mid s} \frac{\Gamma\left(\xi_{t\mid s}+\sum_{\omega\in\boldsymbol{\Omega}}n_{\omega}\delta_{t\mid s}(\omega)\right)}{\Gamma\left(\xi_{t\mid s}\right)} \\[1ex]
    \end{aligned}
    \end{equation*}
    In particular,
    \begin{equation*}
    \begin{aligned}
        \mathbb{E}_{p(\boldsymbol{\theta}|\boldsymbol{\xi})} \left[\theta_k\right] 
        &= \frac{g(\boldsymbol{\eta}(\boldsymbol{\xi} + \mathbf{D}_p \boldsymbol{1}_{(k)}))}{g(\boldsymbol{\eta}(\boldsymbol{\xi}))}  \\[1ex]
        &= \prod_{s\in\boldsymbol{\Lambda}}\frac{\Gamma\left(\sum_{t\mid s}\xi_{t\mid s}\right)}{\Gamma\left(\sum_{t\mid s}\xi_{t\mid s} + \sum_{t\mid s}\sum_{\omega\in\boldsymbol{\Omega}} \boldsymbol{1}_{(k)}\delta_{t\mid s}(\omega)\right)} \prod_{t\mid s} \frac{\Gamma\left(\xi_{t\mid s}+\sum_{\omega\in\boldsymbol{\Omega}}\boldsymbol{1}_{(k)}\delta_{t\mid s}(\omega)\right)}{\Gamma\left(\xi_{t\mid s}\right)} \\[1ex]
        &= \prod_{s\in\boldsymbol{\Lambda}} \frac{\Gamma\left(\sum_{t\mid s}\xi_{t\mid s}\right)}{\Gamma\left(\sum_{t\mid s}\xi_{t\mid s} + \sum_{t\mid s}\delta_{t\mid s}(k)\right)} \prod_{t\mid s} \left(\frac{\Gamma\left(\xi_{t\mid s}+\delta_{t\mid s}(k)\right)}{\Gamma\left(\xi_{t\mid s}\right)}\right)  \\[1ex]
        &= \prod_{s\in\boldsymbol{\Lambda}} \left(\frac{1}{\sum_{t\mid s}\xi_{t\mid s}}\right)^{\sum_{t\mid s}\delta_{t\mid s}(k)} \prod_{t\mid s} {\xi_{t\mid s}}^{\delta_{t\mid s}(k)},
        \quad k=1,\dots,K;\;s \in \boldsymbol{\Lambda}
    \end{aligned}
    \end{equation*}
\end{corollary}
The derivation of the corollary uses the Gamma function's property:
\begin{equation*}
    \Gamma(x+1) = x\Gamma(x) 
\end{equation*}
Intuitively, the expectation of each component $\theta_k$ of Dirichlet-Tree is equal to the product of parameter proportions corresponding to all branches that lead to the component.

\subsection{Derived Distributions of Dirichlet-Tree}

It is useful to define the group of derived distributions to assist the Bayesian inference in the following sections.
\begin{definition}[Derived distributions of Dirichlet-Tree]
    Let $p(\boldsymbol{\theta}|\boldsymbol{\xi}) \in \mathcal{T}$ be a Dirichlet-Tree.
    The corresponding group of derived distributions $\{ p_{t\mid s}(\boldsymbol{\theta}) \}_{|\boldsymbol{\Lambda}|}$ is defined as:
    \begin{equation*}
        p_{t\mid s}(\boldsymbol{\theta})
        = { \mathbb{E}_{p(\boldsymbol{\theta}|\boldsymbol{\xi})} \left[ u_{t\mid s}(\boldsymbol{\theta}) \right] }^{-1} u_{t\mid s}(\boldsymbol{\theta}) p(\boldsymbol{\theta}|\boldsymbol{\xi}),\quad s \in \boldsymbol{\Lambda}
    \end{equation*}
\end{definition}
For notational simplicity, and when no ambiguity arises, we replace the label $t\mid s$ for each branch with $d$, and $|\boldsymbol{\Lambda}|=D$.
We also denote $|\boldsymbol{\Omega}| = K$.
\begin{corollary}[Expectation matrix of derived distributions]
    By stacking the expectation of each derived distribution, the ($K \times D$)-dimensional expectation matrix is derived, and each column gives:
    \begin{equation}
    \begin{aligned}
        \mathbb{E}_{p_d(\boldsymbol{\theta})} \left[ \boldsymbol{\theta} \right]
        &= \left\{\int_{\Delta^{K-1}} \left( \frac{g(\boldsymbol{\eta}(\boldsymbol{\xi} + \mathbf{D}_p \boldsymbol{1}_{(k)}))}{g(\boldsymbol{\eta}(\boldsymbol{\xi}))} \mathbf{B}(\boldsymbol{1}_{(k)}) \right) {\mathbb{E}_{p(\boldsymbol{\theta}|\boldsymbol{\xi})} \left[ u_d(\boldsymbol{\theta}) \right]}^{-1} u_d(\boldsymbol{\theta}) p(\boldsymbol{\theta}|\boldsymbol{\xi}) \,\mathrm{d} \boldsymbol{\theta}\right\}_{k=1}^{K} \\[1ex]
        &= \left\{\frac{g(\boldsymbol{\eta}(\boldsymbol{\xi} + \mathbf{D}_p \boldsymbol{1}_{(k)}))}{g(\boldsymbol{\eta}(\boldsymbol{\xi}))} {\mathbb{E}_{p(\boldsymbol{\theta}|\boldsymbol{\xi})} \left[ u_d(\boldsymbol{\theta}) \right]}^{-1} \int_{\Delta^{K-1}} u_d(\boldsymbol{\theta}) p(\boldsymbol{\theta}|\boldsymbol{\xi} + \mathbf{D}_p \boldsymbol{1}_{(k)}) \,\mathrm{d} \boldsymbol{\theta}\right\}_{k=1}^{K} \\[1ex]
        &= \frac{\mathbb{E}_{p(\boldsymbol{\theta}|\boldsymbol{\xi})} \left[ \boldsymbol{\theta} \right] \odot_k \mathbb{E}_{p(\boldsymbol{\theta}|\boldsymbol{\xi} + \mathbf{D}_p \boldsymbol{1}_{(\boldsymbol{k})})} \left[ u_d(\boldsymbol{\theta}) \right]}{\mathbb{E}_{p(\boldsymbol{\theta}|\boldsymbol{\xi})} \left[ u_d(\boldsymbol{\theta}) \right]}, \quad d = 1,\dots,D
    \end{aligned}
    \label{eq:expectation_tilted_distributions}
    \end{equation}
    where $D=|\boldsymbol{\Lambda}|$ and $\odot_k$ is the Hadamard (element-wise) product with respect to the subscript $k$.
\end{corollary}

\section{Latent Dirichlet-Tree Allocation}

In this section, we introduce Latent Dirichlet-Tree Allocation (LDTA), which generalizes the Latent Dirichlet Allocation model of \citet{blei:03} by replacing the Dirichlet prior with an arbitrary Dirichlet-Tree.
Similar to LDA, LDTA is a generative probabilistic graphical model.
For clarity, we adopt column-major order in the mathematical formulation of the model, whereas the implementation uses row-major order to adhere to standard programming conventions.

\subsection{The Generative Process}

Two sets with known cardinalities $V$ and $K$ are postulated before constructing our model: (i) the collection of $V$ words (the vocabulary, or visual words in the case of images):
\begin{equation*}
    \mathbb{W} = \{ w_1, \dots, w_v, \dots, w_V \}
\end{equation*}
 and (ii) the collection of $K$ topics:
\begin{equation*}
    \mathbb{Z} = \{ z_1, \dots, z_k, \dots, z_K \}
\end{equation*}
The observed data set is a sequence of $M$ documents (the corpus, or images): 
\begin{equation*}
    \mathcal{D} = ( {\boldsymbol{w}}_1, \dots, {\boldsymbol{w}}_m, \dots , {\boldsymbol{w}}_M )
\end{equation*}
where each document is a sequence of $N_m$ words:
\begin{equation*}
    {\boldsymbol{w}}_m = (w_{1m}, \dots , w_{nm}, \dots , w_{N_m m}), \quad w_{nm}\in\mathbb{W}
\end{equation*}
In LDTA, a latent topic is assumed with respect to each observed word:
\begin{equation*}
    {\boldsymbol{z}}_m = (z_{1m}, \dots , z_{nm}, \dots , z_{N_m m}), \quad z_{nm}\in\mathbb{Z}
\end{equation*}
Sometimes, we ignore the subscript $m$ to simply keep our notation uncluttered when we refer to any document in $\mathcal{D}$.
According to the ex-changeability assumption \citep{blei:03}, we only care about the occurring numbers of the words that appeared in a document, instead of the order of the word sequence.
Therefore, we indicate each word as a binary vector $\boldsymbol{1}_{nm}^{(v)}\in\{0, 1\}^V$ defined over  $\mathbb{W}$, where the $v^{th}$ element equals 1 if the $v^{th}$ word is chosen, and 0 otherwise.
Consequently, we obtain the word counting vector of a document by summing up all the word indicator vector:
\begin{equation*}
    \boldsymbol{n}_{m} = (n_{1m}, \dots, n_{vm}, \dots n_{Vm})^{\top} = \sum_{i=1}^{N_{m}}\boldsymbol{1}_{im}^{(v)}
\end{equation*}
Taking the latent topics into account, the binary indicator vector for each word is lifted to a binary matrix
$\boldsymbol{1}_{nm}^{(vk)} \in \{0,1\}^{V \times K}$,
defined over $\mathbb{W} \times \mathbb{Z}$, to represent word--topic co-occurrence.
Similarly, we obtain the word--topic counting matrix for a document by summing up all the word--topic indicator matrices:
\begin{equation*}
    \boldsymbol{n}_{m}
    =
    \{n_{vk}\}_{(V \times K)}
    =
          \left(
          \begin{array}{ccccc}
          n_{11} & \cdots & n_{1k} & \cdots & n_{1K}\\
          \vdots &        & \vdots &        & \vdots\\
          n_{v1} & \cdots & n_{vk} & \cdots & n_{vK}\\
          \vdots &        & \vdots &        & \vdots\\
          n_{V1} & \cdots & n_{Vk} & \cdots & n_{VK}
          \end{array}
          \right)
    = \sum_{i=1}^{N_{m}}\boldsymbol{1}_{im}^{(vk)}
\end{equation*}
Moreover, the document--word--topic counting tensor is constructed by stacking counting matrices of all documents in the data set (Figure~\ref{fig:mvk_tensor}).
\begin{figure}[t]
    \centering
    \includegraphics[width=0.35\linewidth]{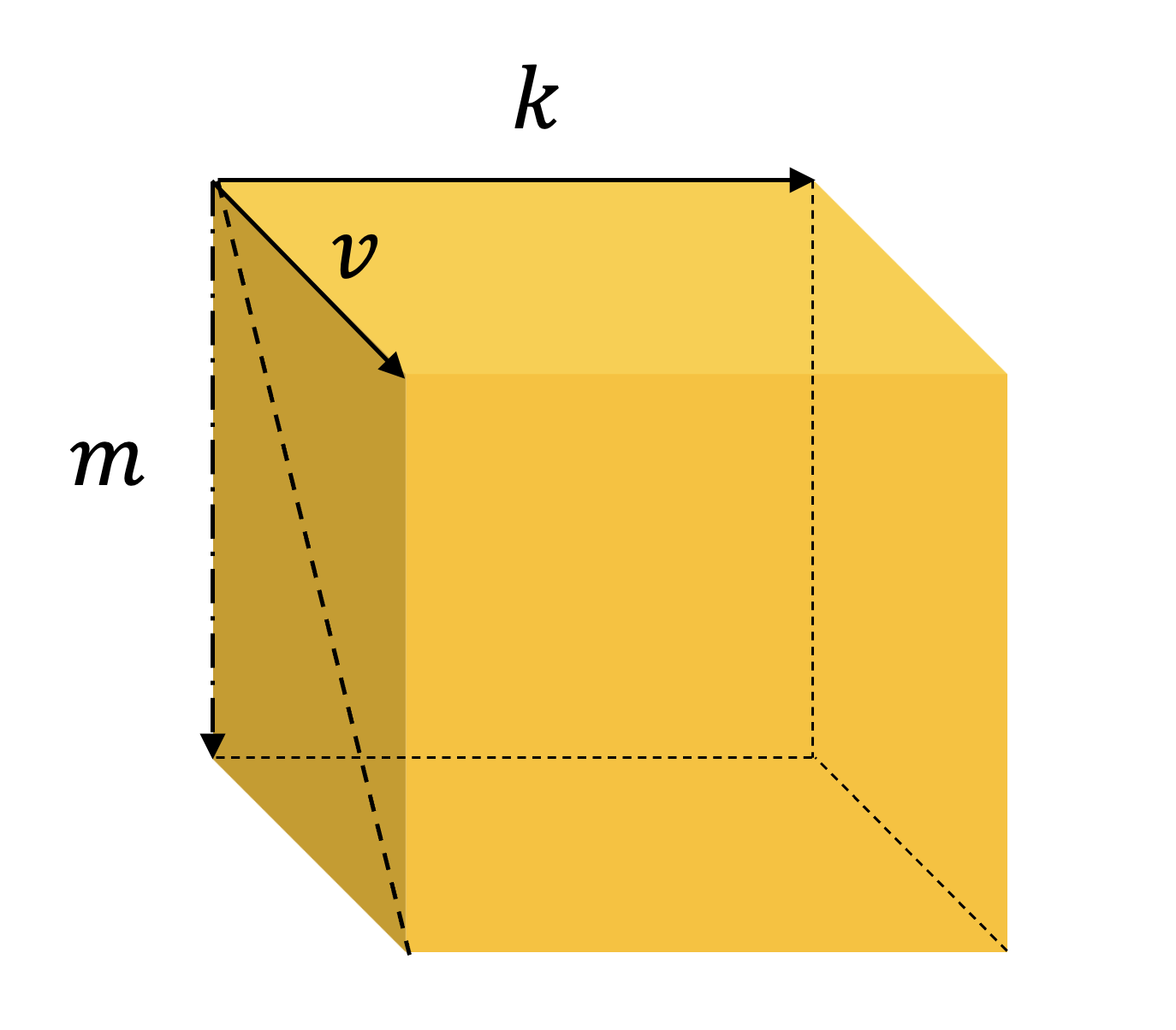}
    \caption{The document--topic--word counting tensor. }
    \label{fig:mvk_tensor}
\end{figure}
Similar to LDA, LDTA is probabilistic graphic model.
The generative process for each document is described as follows:
\begin{enumerate}
    \item Draw $N \sim \operatorname{Poisson}(\mu)$.
    \item Draw $\boldsymbol{\theta} \sim \mathrm{DT}(\boldsymbol{\xi})$.
    \item For each word $w_n$ for $n=1,\dots,N$:
    \begin{enumerate}
        \item Draw $z_n \sim \operatorname{Mult}(\boldsymbol{\theta})$.
        \item Draw $w_n \sim \operatorname{Mult}(\boldsymbol{\varphi}_{z_n})$.
    \end{enumerate}
\end{enumerate}
And the whole generation process produced our data set by repeating above procedures for $M$ times (Figure~\ref{fig:prob_graph}).
Among the parameters, $\boldsymbol{\xi}$ is the hyper-parameter for the Dirichlet-Tree prior;
$\boldsymbol{\theta}\in\Delta^{K-1}$ is the topic proportion with respect to $\mathbb{Z}$, and we denote the square diagonal form of each document's proportion as $M(\boldsymbol{\theta}_m)$;
$\boldsymbol{\varphi}$ is a word--topic conditional probability matrix defined as the conditional probability of generating $v^{th}$ word given the $k^{th}$ topic:
\begin{equation*}
    \boldsymbol{\varphi} = 
    \{\varphi_{vk}\}_{V \times K} = \left(
                                \begin{array}{ccccc}
                                \varphi_{11} & \cdots & \varphi_{1k} & \cdots & \varphi_{1K}    \\
                                \vdots       &        & \vdots       &        & \vdots          \\
                                \varphi_{v1} & \cdots & \varphi_{vk} & \cdots & \varphi_{vK}    \\
                                \vdots       &        & \vdots       &        & \vdots          \\
                                \varphi_{V1} & \cdots & \varphi_{Vk} & \cdots & \varphi_{VK}
                                \end{array}
                          \right)
\end{equation*}
and we assume $\boldsymbol{\varphi}$ to be non-generated.
By multiplying the word--topic conditional probability matrix and the square diagonal form of each document's topic proportion, we obtain the joint probability matrix of the co-occurrence of each word--topic couple in each document:
\begin{equation*}
    \Phi_m =    \left(
                \begin{array}{ccccc}
                \varphi_{11}\theta_1 & \cdots & \varphi_{1k}\theta_k & \cdots & \varphi_{1K}\theta_K    \\
                \vdots               &        & \vdots               &        & \vdots                  \\
                \varphi_{v1}\theta_1 & \cdots & \varphi_{vk}\theta_k & \cdots & \varphi_{vK}\theta_K    \\
                \vdots               &        & \vdots               &        & \vdots                  \\
                \varphi_{V1}\theta_1 & \cdots & \varphi_{Vk}\theta_k & \cdots & \varphi_{VK}\theta_K
                \end{array}
                \right)
            =   \boldsymbol{\varphi} M(\boldsymbol{\theta})
\end{equation*}
The relationship of the variables in LDTA expressed in tensors is shown in Figure~\ref{fig:tensor_rela_var}.

\begin{figure}[t]
    \centering
    \includegraphics[width=0.4\linewidth]{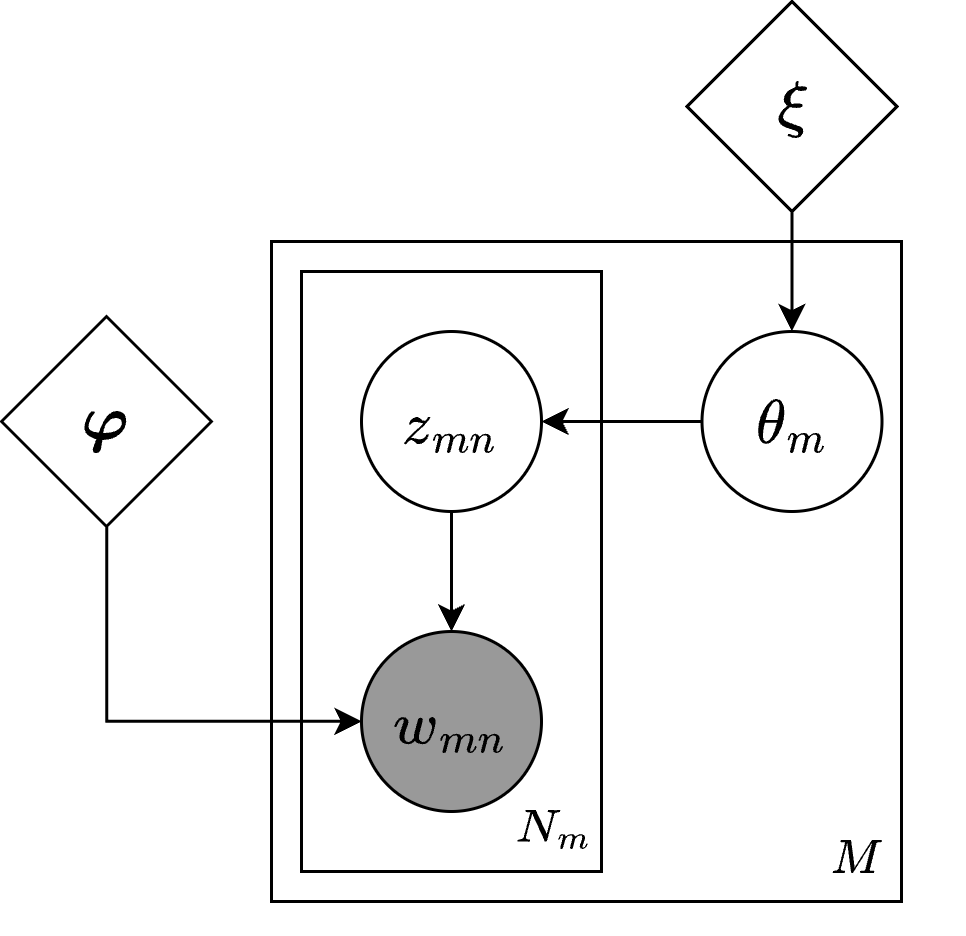}
    \caption{Graphical representation (plate notation) of LDTA model. The shaded circle represents the observed variable; the blank circles represent the latent variables and diamonds indicate hyper-parameters. The arrows imply dependent generative relationship and the rectangles with number at its right-lower corner represent repetitions.}
    \label{fig:prob_graph}
\end{figure}

\begin{figure}[t]
    \centering
    \includegraphics[width=0.8\linewidth]{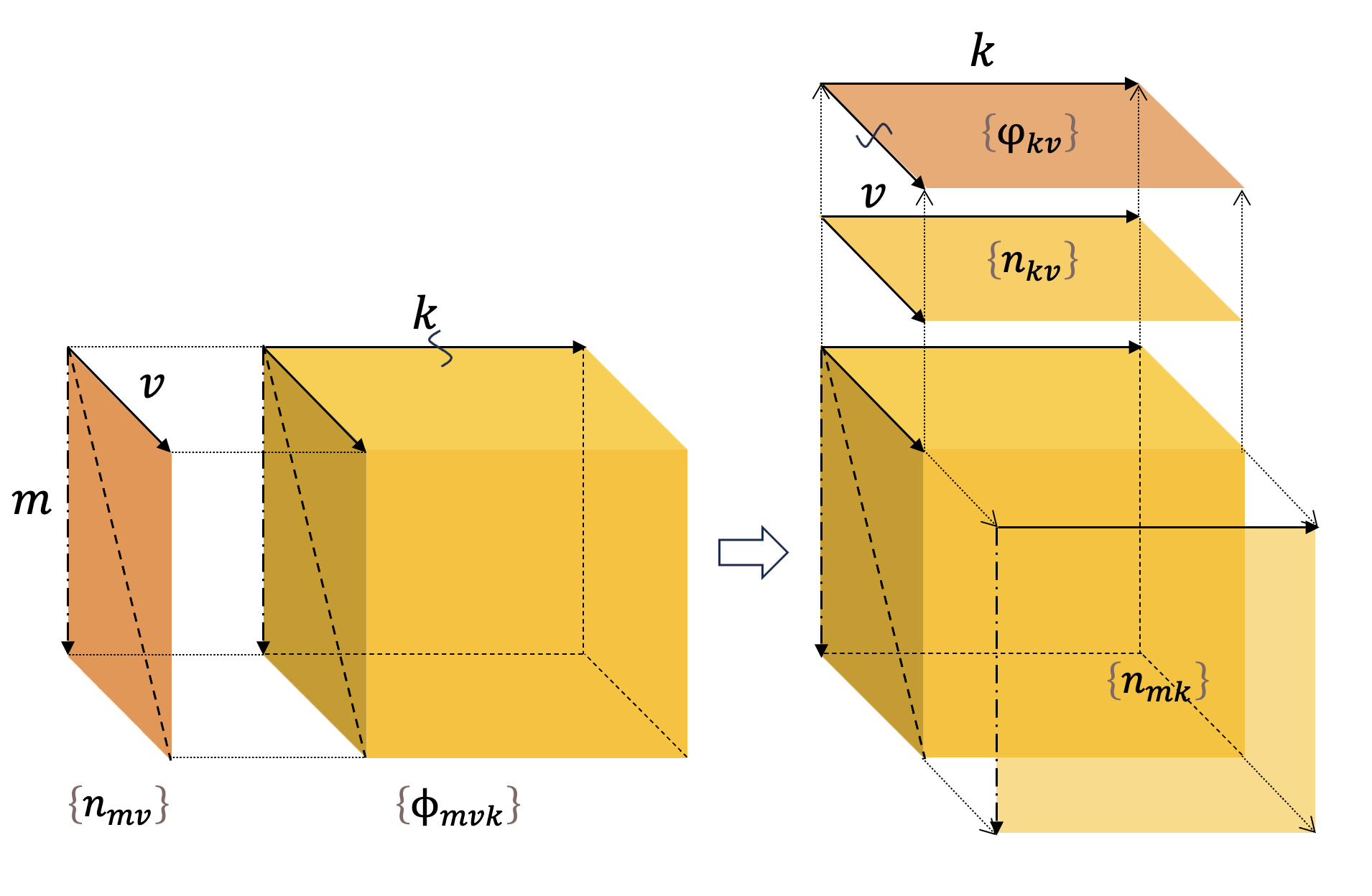}
    \caption{Tensor expression of relationship of the variables}
    \label{fig:tensor_rela_var}
\end{figure}

\subsection{The Joint and the Marginalized Distributions}

Given the hyper-parameters $\boldsymbol{\xi}$ and $\boldsymbol{\varphi}$, as well as the conditional independence assumption of the variables, the joint distribution of the latent parameters and the observed counting vector for a single document is given by:
\begin{equation*}
\begin{aligned}
    p(\boldsymbol{w}, \boldsymbol{z}, \boldsymbol{\theta}|\boldsymbol{\varphi}, \boldsymbol{\xi})
    & = p(\boldsymbol{w}|\boldsymbol{z}, \boldsymbol{\varphi}) p(\boldsymbol{z}|\boldsymbol{\theta}) p(\boldsymbol{\theta}|\boldsymbol{\xi})   \\[1ex]
    & = p(\boldsymbol{\theta}|\boldsymbol{\xi}) \prod_{n=1}^{N_m} p(z_n|\boldsymbol{\theta}) p(w_n|z_n, \boldsymbol{\varphi})  
\end{aligned}
\end{equation*}
where $\boldsymbol{w}$ and $\boldsymbol{z}$ are sequences of words and topics respectively.
Based on above discussion, the order of the word sequence is trivial and can simply be replaced by counting matrix $\{n_{vk}\}_{(V \times K)}$.
Hence, each term is respectively represented as:
\begin{align*}
    & p(\boldsymbol{\theta}|\boldsymbol{\xi}) = \mathrm{DT}(\boldsymbol{\theta}|\boldsymbol{\xi})  \\[1ex]
    & p(\boldsymbol{z}|\boldsymbol{\theta}) = \prod_{k=1}^{K} {\theta_{k}}^{n_k} = \prod_{k=1}^{K} \prod_{v=1}^{V} {\theta_{k}}^{n_{vk}} \\[1ex]
    & p(\boldsymbol{w}|\boldsymbol{z}, \boldsymbol{\varphi}) = \prod_{k=1}^{K} \prod_{v=1}^{V} {\varphi_{vk}}^{n_{vk}}
\end{align*}
Therefore, we arrive at the product form of prior multiplying likelihood:
\begin{equation*}
    p(\boldsymbol{w}, \boldsymbol{z}, \boldsymbol{\theta}|\boldsymbol{\varphi}, \boldsymbol{\xi}) 
    = \mathrm{DT}(\boldsymbol{\theta}|\boldsymbol{\xi}) \prod_{k=1}^{K} \prod_{v=1}^{V} \left( \varphi_{vk} \theta_{k} \right)^{n_{vk}}
\end{equation*}
By collapsing the topic sequence, we derive the joint distribution of observed $\boldsymbol{w}$ and latent $\boldsymbol{\theta}$:
\begin{equation*}
\begin{aligned}
        p(\boldsymbol{w}, \boldsymbol{\theta}|\boldsymbol{\varphi}, \boldsymbol{\xi}) 
        & = \sum_{\boldsymbol{z}} p(\boldsymbol{w}, \boldsymbol{z}, \boldsymbol{\theta}|\boldsymbol{\varphi}, \boldsymbol{\xi})   \\[1ex]
        & = \mathrm{DT}(\boldsymbol{\theta}|\boldsymbol{\xi}) \prod_{v=1}^{V} \sum_{\{n_{vk}\}_K} 
                \left(
                \begin{array}{c}
                    n_v \\
                    n_{v1}, \cdots, n_{vk}, \cdots, n_{vK} 
                \end{array}
                \right) 
            \prod_{k=1}^K \left( \varphi_{vk} \theta_k \right)^{n_{vk}}     \\[1ex]
        & = \mathrm{DT}(\boldsymbol{\theta}|\boldsymbol{\xi}) \prod_{v=1}^{V} \left(\sum_{k=1}^{K}\varphi_{vk}\theta_k\right)^{n_v}    \\
\end{aligned}
\end{equation*}
It is noteworthy to mention that a document's probability vector over word collection $\mathbb{W}$ denoted as $\{t_v(\boldsymbol{\theta}) = \sum_{k=1}^K \varphi_{vk} \theta_k \}_V$ is determined given the document-topic proportion vector $\boldsymbol{\theta}$ and topic-word conditional probability matrix $\boldsymbol{\varphi}$:
\begin{equation*}
\begin{aligned}
        \{ t_{v}(\boldsymbol{\theta}) \}_{v=1}^{V}
                        & =
                        \left(
                            \begin{array}{ccccc}
                                \varphi_{11} & \cdots & \varphi_{1k} & \cdots & \varphi_{1K}\\
                                \vdots       &        & \vdots       &        & \vdots      \\
                                \varphi_{v1} & \cdots & \varphi_{vk} & \cdots & \varphi_{vK}\\
                                \vdots       &        & \vdots       &        & \vdots      \\
                                \varphi_{V1} & \cdots & \varphi_{Vk} & \cdots & \varphi_{VK}
                            \end{array}
                         \right) 
                         \left( 
                            \begin{array}{c}
                                 \theta_1\\
                                 \vdots\\
                                 \theta_k\\
                                 \vdots\\
                                 \theta_K
                            \end{array}
                         \right)    \\[1ex]
                      & = \left(\sum_{k=1}^{K}\varphi_{1k}\theta_k, \ldots, \sum_{k=1}^{K}\varphi_{vk}\theta_k, \ldots, \sum_{k=1}^{K}\varphi_{Vk}\theta_k\right)^{\top}
\end{aligned}                     
\end{equation*}
By integrating out the latent variable $\boldsymbol{\theta}$, the evidence, or the marginal distribution of observed word counting vector is given:
\begin{equation*}
\begin{aligned}
        p(\boldsymbol{w}|\boldsymbol{\varphi}, \boldsymbol{\xi}) 
        & = \int_{\Delta^{K-1}} p(\boldsymbol{w}, \boldsymbol{\theta}|\boldsymbol{\varphi}, \boldsymbol{\xi}) \mathrm{d}\boldsymbol{\theta}  \\[1ex]
        & = \int_{\Delta^{K-1}} \mathrm{DT}(\boldsymbol{\theta}|\boldsymbol{\xi}) \prod_{v=1}^{V} \left( t_v\left(\boldsymbol{\theta}\right) \right)^{n_v} \mathrm{d}\boldsymbol{\theta}
\end{aligned}
\end{equation*}
Finally, we compute the total probability of corpus $\mathcal{D}$ containing $M$ documents:
\begin{eqnarray}    
    p(\mathcal{D}|\boldsymbol{\varphi}, \boldsymbol{\xi}) 
    &=& \prod_{m=1}^{M} p_{m}({\boldsymbol{w}}_{m}|\boldsymbol{\varphi}, \boldsymbol{\xi})\nonumber    \\
    &=& \prod_{m=1}^{M} \int_{\boldsymbol{\theta}_m} \mathrm{DT}(\boldsymbol{\theta}_m|\boldsymbol{\xi}) \prod_{v=1}^{V} {\left(t_{vm}(\boldsymbol{\theta}_m)\right)}^{n_{vm}}  \mathrm{d}\boldsymbol{\theta}_m \nonumber
\end{eqnarray}

\section{Mean-Field Variational Inference}

The mean-field variational inference employs a global mean-field distribution to perform approximation for each latent node in the probabilistic graph.
The approximated posterior is achieved by minimizing the KL-divergence:
\begin{equation*}
\begin{aligned}
    q(\mathbf{z}|\boldsymbol{\eta}^{*})
    &= \arg\;\min\limits_{q \in \mathcal{Q}}\;\mathrm{KL}(q(\mathbf{z}|\boldsymbol{\eta})\;||\;p(\mathbf{z}|\boldsymbol{w}))  \\[1ex]
    &= \arg\;\min\limits_{q \in \mathcal{Q}}\;\int_{\mathbf{z}} q(\mathbf{z}|\boldsymbol{\eta}) \log \frac{q(\mathbf{z}|\boldsymbol{\eta})}{p(\mathbf{z}|\boldsymbol{w})} \,\mathrm{d}\mathbf{z}
\end{aligned}
\end{equation*}
which is equivalent to maximizing Evidence Lower Bound (ELBO):
\begin{equation*}
        \mathrm{ELBO} = \int_{\mathbf{z}} q(\mathbf{z}|\boldsymbol{\eta})\log \frac{p(\boldsymbol{w}, \mathbf{z})}{q(\mathbf{z}|\boldsymbol{\eta})} \,\mathrm{d}\mathbf{z}
\end{equation*}
Since we will deal with both continuous variables (topic proportions) and discrete variables (topics) in our model, we introduce a new notation:
\begin{equation*}
\Bigl\langle f(x), g(x) \Bigr\rangle_x =
\begin{cases}
\sum_{i} f(x_i) g(x_i), & \text{if } x \text{ is discrete}, \\[6pt]
\int_{x} f(x)g(x) \,\mathrm{d}x,        & \text{if } x \text{ is continuous}
\end{cases}
\end{equation*}
The feasibility of variational mean-field as an approximating distribution lies on the independency assumption between its random variables, which leads to an important property:
\begin{equation*}
    \Bigl\langle y(x_n), \;\; q(x_1)q(x_2)\cdots q(x_n)\cdots q(x_N) \Bigr\rangle_{\boldsymbol{x}} = \Bigl\langle y(x_n), \;\; q(x_n) \Bigr\rangle_{x_n}
\end{equation*}
This enables us to expand the joint distribution over every individual latent or observed variable node in the generative probabilistic graph, and deal with each ELBO term with respect to the corresponding latent nodes separately from an external viewpoint.
By taking the probabilistic graph in Figure~\ref{fig:prob_graph} as demonstrating example, we illustrate the derivation of mean-field variational inference and provide a general algorithm regardless of the specific Dirichlet-Tree prior.
More inference algorithms for the generalized probabilistic graphs can be easily derived in a similar manner.

\subsection{The Evidence Lower Bound}

Suppose $p(\boldsymbol{\theta}|\boldsymbol{\xi})$ is a Dirichlet-Tree distribution with selection matrix $\mathbf{D}_p$ and $p(\boldsymbol{\theta}|\boldsymbol{\xi})$ is employed as the prior distribution in the generative plate notation in Figure~\ref{fig:prob_graph}.
We already prove that $p(\boldsymbol{\theta}|\boldsymbol{\xi})$ is both a member of exponential family and a conjugate prior to the multinomial:
\begin{equation*}
    p(\boldsymbol{\theta}|\boldsymbol{\xi})
    = \exp \{ \boldsymbol{\eta}(\boldsymbol{\xi})^{\mathrm{T}} \mathbf{u}(\boldsymbol{\theta}) - \log g(\boldsymbol{\eta}(\boldsymbol{\xi})) \}
\end{equation*}
Without the need to specify the concrete expression, the universal form of approximation algorithm and key steps in the optimization is demonstrated by only utilizing Dirichlet-Tree's properties of exponential family and conjugacy to Multinomial.
Our goal is to approximate the joint posterior distribution $p(\boldsymbol{z}, \boldsymbol{\theta}|\boldsymbol{w}, \boldsymbol{\varphi}, \boldsymbol{\xi})$ of all latent variables of interest given the observed sequences $\boldsymbol{w}$, and the global mean-field approximating distribution is defined as:
\begin{equation*}
        q(\boldsymbol{\theta}, \boldsymbol{z}|\boldsymbol{\zeta}, \boldsymbol{\phi})
        = \prod_{m=1}^{M} q(\boldsymbol{\theta}|\boldsymbol{\zeta}_m) \prod_{n=1}^{N_m} q(z_{nm}|\boldsymbol{\phi}_{nm})
\end{equation*}
It is obvious that the global approximation wipes out the original dependency between variables such as the dependency between $\boldsymbol{\theta}$ and $\boldsymbol{z}$.
Given the approximating distribution, the optimized parameters are obtained by minimizing $\mathrm{KL}(q(\boldsymbol{\theta}, \boldsymbol{z}|\boldsymbol{\zeta}, \boldsymbol{\phi})||p(\boldsymbol{\theta}, \boldsymbol{z}|\boldsymbol{w}, \boldsymbol{\varphi}, \boldsymbol{\xi}))$, which is equivalent to maximizing the Evidence Lower Bound (ELBO):

\begin{eqnarray*}
        L(\boldsymbol{\zeta}, \boldsymbol{\phi}; \boldsymbol{\xi}, \boldsymbol{\varphi})
        &=& \left\langle \log \frac{p(\boldsymbol{w}, \boldsymbol{z}, \boldsymbol{\theta}|\boldsymbol{\xi}, \boldsymbol{\varphi})}{q(\boldsymbol{z}, \boldsymbol{\theta}|\boldsymbol{\zeta}, \boldsymbol{\phi})}, \;\;q(\boldsymbol{z}, \boldsymbol{\theta}|\boldsymbol{\zeta}, \boldsymbol{\phi}) \right\rangle_{(\boldsymbol{z}, \boldsymbol{\theta})} \nonumber \\[1ex]
        &=& \sum_{m=1}^{M} \biggl\{ \Bigl\langle \log p(\boldsymbol{\theta}|\boldsymbol{\xi}) \Bigr\rangle_{q(\boldsymbol{\theta}|\boldsymbol{\zeta}_m)} + \Bigl\langle \log p(\boldsymbol{z}|\boldsymbol{\theta}) \Bigr\rangle_{q(\boldsymbol{\theta}|\boldsymbol{\zeta}_m) q(\boldsymbol{z}|\boldsymbol{\phi}_m)} \\[1ex]
        &\;& \;\;\;\;\;\;\;\;\;\;+ \Bigl\langle \log p(\boldsymbol{w}_m|\boldsymbol{z}, \boldsymbol{\varphi}) \Bigr\rangle_{q(\boldsymbol{z}|\boldsymbol{\phi}_m)} - \Bigl\langle \log q(\boldsymbol{\theta}|\boldsymbol{\zeta}_m) \Bigr\rangle_{q(\boldsymbol{\theta|\boldsymbol{\zeta}_m})}\nonumber \\[1ex]
        &\;& \;\;\;\;\;\;\;\;\;\;  - \Bigl\langle \log q(\boldsymbol{z}|\boldsymbol{\phi}_m) \Bigr\rangle_{q(\boldsymbol{z}|\boldsymbol{\phi}_m)} \biggr\} \nonumber \\[1ex]
        &=& \sum_{m=1}^{M} \biggl\{ \Bigl\langle \log p(\boldsymbol{\theta}_m|\boldsymbol{\xi}) \Bigr\rangle_{q(\boldsymbol{\theta}_m|\boldsymbol{\zeta}_m)} + \sum_{n=1}^{N_m} \; \Bigl\langle \log p(z_{nm}|\boldsymbol{\theta}_m) \Bigr\rangle_{q(\boldsymbol{\theta}_m|\boldsymbol{\zeta}_m) q(z_{nm}|\boldsymbol{\phi}_{nm})} \nonumber \\[1ex] 
        &\;& \;\;\;\;\;\;\;\;\;\; + \sum_{n=1}^{N_m} \; \Bigl\langle \log p(w_{nm}|z_{nm}, \boldsymbol{\varphi}) \Bigr\rangle_{q(z_{nm}|\boldsymbol{\phi}_{nm})} - \Bigl\langle \log q(\boldsymbol{\theta}_m|\boldsymbol{\zeta}_m) \Bigr\rangle_{q(\boldsymbol{\theta}_m|\boldsymbol{\zeta}_m)} \nonumber \\[1ex]
        &\;& \;\;\;\;\;\;\;\;\;\; - \sum_{n=1}^{N_m} \; \Bigl\langle \log q(z_{nm}|\boldsymbol{\phi}_{nm}) \Bigr\rangle_{q(z_{nm}|\boldsymbol{\phi}_{nm})} \biggr\} \nonumber
\end{eqnarray*}
Without loss of generality, we consider ELBO for any one of the documents and drop the subscript $m$. We have:
\begin{eqnarray*}
        L_m(\boldsymbol{\zeta}, \boldsymbol{\phi}; \boldsymbol{\xi}, \boldsymbol{\varphi})
        &=& \boldsymbol{\eta}(\boldsymbol{\xi})^{\mathrm{T}} \mathbb{E}_{q(\boldsymbol{\theta}|\boldsymbol{\zeta})} \left[ \mathbf{u}(\boldsymbol{\theta}) \right] - \log g(\boldsymbol{\eta}(\boldsymbol{\xi})) \nonumber \\
        &\;& + \sum_{n=1}^{N} {\boldsymbol{\phi}_n}^{\mathrm{T}} \mathbb{E}_{q(\boldsymbol{\theta}|\boldsymbol{\zeta})} \left[ \log \boldsymbol{\theta} \right] \nonumber \\
        &\;& + \sum_{n=1}^{N} {\boldsymbol{\phi}_n}^{\mathrm{T}} \log \boldsymbol{\varphi}_{v(n)} \nonumber \\
        &\;& - \boldsymbol{\eta}(\boldsymbol{\zeta})^{\mathrm{T}} \mathbb{E}_{q(\boldsymbol{\theta}|\boldsymbol{\zeta})} \left[ \mathbf{u}(\boldsymbol{\theta}) \right] + \log g(\boldsymbol{\eta}(\boldsymbol{\zeta}))\nonumber \\
        &\;& - \sum_{n=1}^{N} {\boldsymbol{\phi}_n}^{\mathrm{T}} \log \boldsymbol{\phi}_n \nonumber \\
        &=& \left[ \boldsymbol{\eta}(\boldsymbol{\xi}) - \boldsymbol{\eta}(\boldsymbol{\zeta}) + \mathbf{D}_q \left( \sum_{n=1}^{N} \boldsymbol{\phi}_n \right) \right]^{\mathrm{T}} \mathbb{E}_{q(\boldsymbol{\theta}|\boldsymbol{\zeta})} \left[ \mathbf{u}(\boldsymbol{\theta}) \right] \nonumber \\
        &\;& + \log g(\boldsymbol{\eta}(\boldsymbol{\zeta})) - \log g(\boldsymbol{\eta}(\boldsymbol{\xi})) \nonumber \\
        &\;& + \sum_{n=1}^{N} {\boldsymbol{\phi}_n}^{\mathrm{T}} \log \boldsymbol{\varphi}_{v(n)} - \sum_{n=1}^{N} {\boldsymbol{\phi}_n}^{\mathrm{T}} \log \boldsymbol{\phi}_n \nonumber \\
\end{eqnarray*}
where the subscript $v$ indicates the $v^{th}$ choice from the dictionary at the $n^{th}$ observed word, and $\mathbf{D}_q$ represents Dirichlet selection matrix which depends only on the structure of approximator $q(\boldsymbol{\theta}|\boldsymbol{\zeta})$ regardless of the value of $\boldsymbol{\zeta}$.
The tensor expression of ELBO is shown in Figure~\ref{fig:tensor_elbo}.
It is noteworthy to mention that the form of this specific expression of ELBO depends only on the specific generative probabilistic graph with the corresponding nodes replaced by the choice of the specific Dirichlet-Tree distribution.

\begin{figure}[t]
    \centering
    \includegraphics[width=0.9\linewidth]{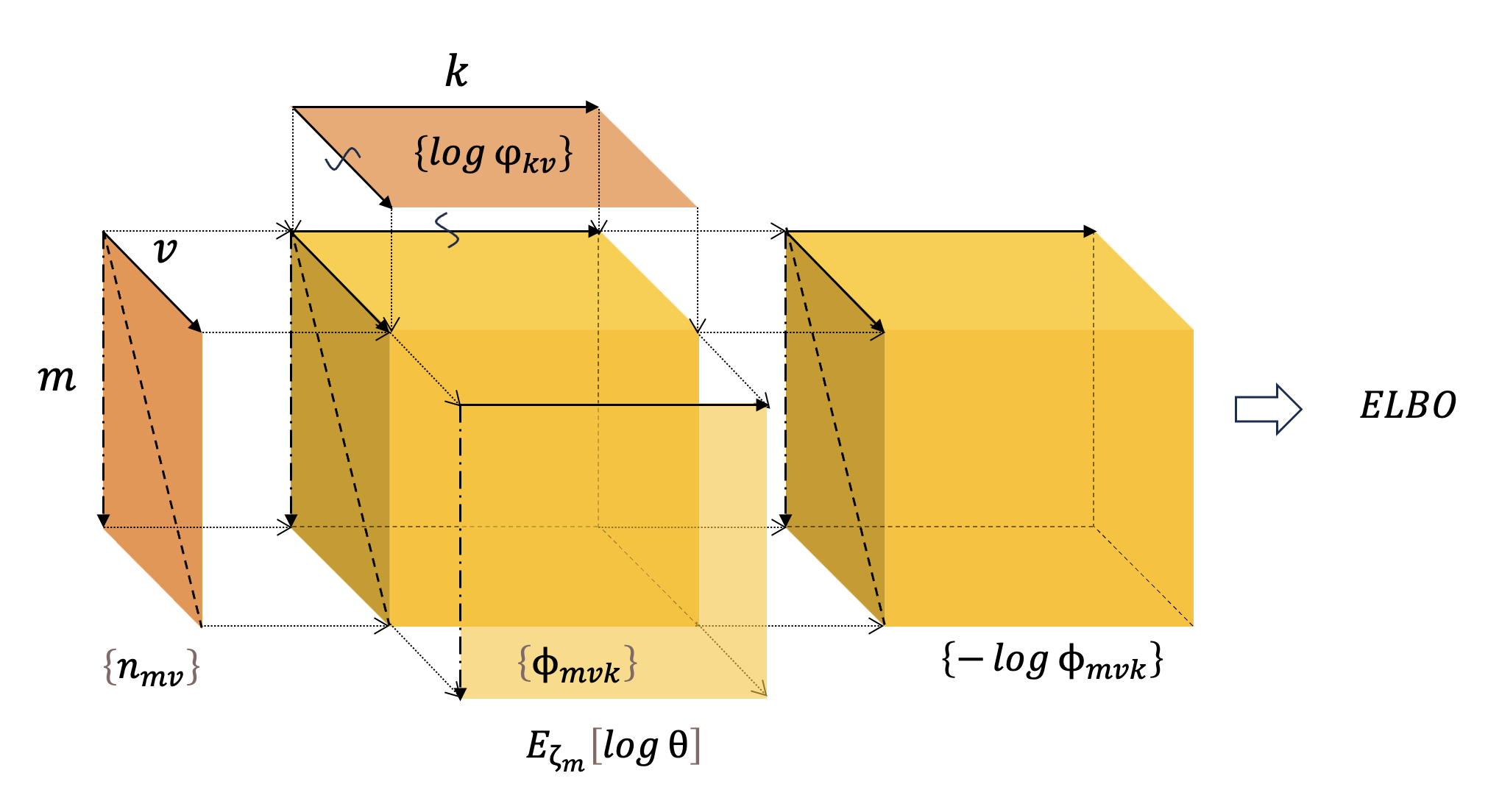}
    \caption{Tensor expression of ELBO}
    \label{fig:tensor_elbo}
\end{figure}

\begin{figure}[t]
    \centering
    \includegraphics[width=0.8\linewidth]{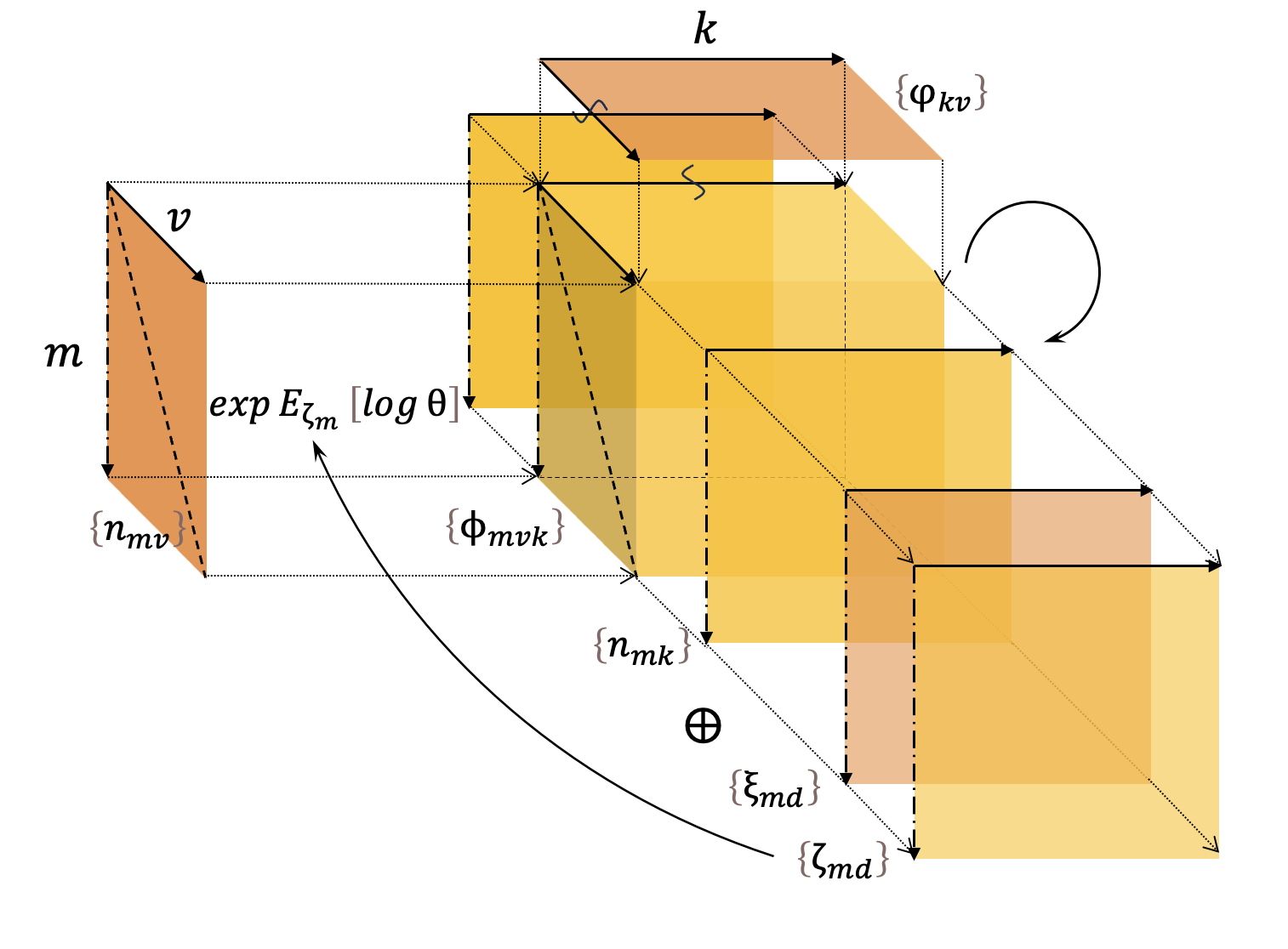}
    \caption{Expectation-Maximization iteration process of mean-field variational inference}
    \label{fig:vi_em}
\end{figure}

\subsection{E-Step}

\vspace{1em}
\noindent
\textbullet\; \textbf{Optimization in terms of $\boldsymbol{\zeta}$}
\vspace{0.5em}

\noindent
Parameter $\boldsymbol{\zeta}$ belongs to $q(\boldsymbol{\theta}|\boldsymbol{\zeta})$ over a document's topic proportion $\boldsymbol{\theta}$, one term of the global approximating mean-field. We differentiate the ELBO with respect to $\boldsymbol{\zeta}$:

\begin{equation*}
\begin{aligned}
        \nabla_{\boldsymbol{\zeta}} L(\boldsymbol{\zeta})
        &= \nabla_{\boldsymbol{\zeta}} \left\{ \left[ \boldsymbol{\eta}(\boldsymbol{\xi}) - \boldsymbol{\eta}(\boldsymbol{\zeta}) + \mathbf{D}_q \left( \sum_{n=1}^{N} \boldsymbol{\phi}_n \right) \right]^{\mathrm{T}} \mathbb{E}_{q(\boldsymbol{\theta}|\boldsymbol{\zeta})} \left[ \mathbf{u}(\boldsymbol{\theta}) \right] + \log g(\boldsymbol{\eta}(\boldsymbol{\zeta})) + C \right\} \\
        &= \left[ \boldsymbol{\xi} - \boldsymbol{\zeta} + \mathbf{D}_q\left( \sum_{n=1}^{N} \boldsymbol{\phi}_n \right) \right]^{\mathrm{T}} \nabla_{\boldsymbol{\zeta}} \mathbb{E}_{q(\boldsymbol{\theta}|\boldsymbol{\zeta})}\left[ \mathbf{u}(\boldsymbol{\theta}) \right] \\
        &= \left\{ \left[ \boldsymbol{\xi} - \boldsymbol{\zeta} + \mathbf{D}_q \left( \sum_{n=1}^{N} \boldsymbol{\phi}_n \right) \right]^{\mathrm{T}} \left( \frac{\partial}{\partial \zeta_d} \mathbb{E}_{q(\boldsymbol{\theta}|\boldsymbol{\zeta})} \left[ \mathbf{u}(\boldsymbol{\theta}) \right] \right) \right\}_{d=1}^{D}
\end{aligned}
\end{equation*}
By setting derivatives to $\boldsymbol{0}$, we have:
\begin{equation*}
        \boldsymbol{\zeta}^{*} = \boldsymbol{\xi} + \mathbf{D}_q\left( \sum_{n=1}^{N} \boldsymbol{\phi}_{n} \right) = \boldsymbol{\xi} \oplus_q \left( \sum_{n=1}^{N} \boldsymbol{\phi}_n \right)
\end{equation*}
Notice that $\boldsymbol{\phi}_n$ only depends on the observed words $w_v$ in this single document, the optimization step can be expressed in terms of matrix multiplication:
\begin{equation*}
        \boldsymbol{\zeta}^{*} = \boldsymbol{\xi} \oplus_q \left\{ \phi_{kv}n_{v} \right\}_{K}
\end{equation*}

\vspace{1em}
\noindent
\textbullet\; \textbf{Optimization in terms of $\boldsymbol{\phi}$}
\vspace{0.5em}

\noindent
Parameter $\boldsymbol{\phi}$ is matrix $\left\{ \phi_{kn} \right\}_{(K \times N)}$, the multinomial parameters of $\prod_{n=1}^{N} q(z_n|\boldsymbol{\phi}_n)$ given the observed word $w_n$ at each location $n$. The optimized $\boldsymbol{\phi}^{*}$ is obtained by maximizing the ELBO, which is a constrained optimization problem with respect to $\boldsymbol{\phi}$ where there are $N$ constraints:
\begin{equation*}
    \sum_{l=1}^{K} \phi_{ln} = 1, \;\;\;\; n = 1, 2, \cdots, N
\end{equation*}
Without loss of generality, we consider the Lagrangian function with respect to $\phi_{kn}$ of $\boldsymbol{\phi}_n$:
\begin{equation*}
        \mathcal{L} (\boldsymbol{\phi}_n)
        = {\boldsymbol{\phi}_n}^{\mathrm{T}} \mathbb{E}_{q(\boldsymbol{\theta}|\boldsymbol{\zeta})} \left[ \log \boldsymbol{\theta} \right] + {\boldsymbol{\phi}_n}^{\mathrm{T}} \log \boldsymbol{\varphi}_{v(n)} - {\boldsymbol{\phi}_n}^{\mathrm{T}} \log \boldsymbol{\phi}_n + \sum_{n=1}^{N} \lambda_n \left( \sum_{l=1}^{K} \phi_{ln} -1  \right) + C
\end{equation*}
\begin{equation*}
    \frac{\partial \mathcal{L}}{\partial \phi_{kn}} = \mathbb{E}_{q(\boldsymbol{\theta}|\boldsymbol{\zeta})} \left[ \log \theta_k \right] +\log \varphi_{vk} - \log \phi_{kn} - 1 + \lambda_n
\end{equation*}
By setting partial derivatives with respect to $\phi_{nk}$ zero, we have:
\begin{equation*}
    {\phi_{kn}}^{*} = \varphi_{vk} \exp \left\{ \mathbb{E}_{q(\boldsymbol{\theta}|\boldsymbol{\boldsymbol{\zeta}})} \left[\log \theta_k \right] + \lambda_n -1 \right\}
\end{equation*}
where $\exp \{ \lambda_n -1 \}$ is the normalizer.

Since $\left\{ \boldsymbol{\phi}_n \right\}_N$ depends only on which word $\left\{w_{vn}\right\}_N$ out of vocabulary $\left\{w_v\right\}_V$ is chosen, the matrix $\left\{ \phi_{kn} \right\}_{(K \times N)}$ is transformed into $\left\{ \phi_{kv} \right\}_{(K \times V)}$ for $\prod_{v=1}^{V} {q(z|\boldsymbol{\phi}_v)}^{n_v}$, where $n_v$ is the count of the $v^{th}$ word in the document.
Therefore, the optimization step can be expressed in terms of matrices:
\begin{equation*}
\begin{aligned}
        \{\phi_{kv}\}_{(K \times V)}^{*} 
        &= \left\{ {z_v}^{-1} \varphi_{vk} \exp \left\{ \mathbb{E}_{q(\boldsymbol{\theta}|\boldsymbol{\zeta})} \left[\log \theta_k\right] \right\} \right\}_{(K \times V)} \\
        & = \left\{ {z_v}^{-1} \varphi_{vk} \exp \left\{ \left(\mathbf{D}_q^{\mathrm{T}}\right) \mathbb{E}_{q(\boldsymbol{\theta}|\boldsymbol{\zeta})} \left[ \mathbf{u}(\boldsymbol{\theta}) \right] \right\}_k \right\}_{(K \times V)} \\
\end{aligned}
\end{equation*}
where $z_v$ is the normalizer.
The E-step of mean-field variational inference is concluded in Algorithm~\ref{alg:meanfieldvi}.

\begin{algorithm}
    \caption{Variational mean-field inference for a single document $\boldsymbol{w}_m$}
    \label{alg:meanfieldvi}
    \KwIn{Prior $\boldsymbol{\xi}$, $\{\varphi_{vk}\}_{V \times K}$; Observation count $\{n_{v}\}_{V}$}
    \KwOut{Optimized $\boldsymbol{\zeta}$, $\{\phi_{kv}\}_{K \times V}$}
    Initialize each element in $\{\phi_{kv}\}_{K \times V}$ as $1/K$\;
    Initialize $\boldsymbol{\zeta} = \boldsymbol{\xi} \oplus \{\phi_{kv} n_v\}_K$\;
    \While{not convergence}
    {
        \For{$w_v$ in $\mathbb{W}$ in parallel}
        {
            Compute $\{\phi_{kv}\}_{(K \times V)}^{new} = \left\{ \varphi_{vk} \exp \{ \mathbb{E}_{q(\boldsymbol{\theta}|\boldsymbol{\zeta})} \left[\log \theta_k\right] \} \right\}_{(K \times V)}$\;
            Normalize $\{ \phi_{kv} \}_{(K \times V)}^{new}$ such that $\sum_{k=1}^{K} \phi_{kv} = 1$\;
        }
        Update $\boldsymbol{\zeta}^{new} = \boldsymbol{\xi} \oplus \left\{ \phi_{kv}^{new} n_{v} \right\}_{K}$\;
    }
\end{algorithm}

\subsection{M-Step}

Given optimized $\{q(\boldsymbol{\theta}|\boldsymbol{\zeta}_{m}^{*})\}_{M}$ and $\{\phi_{kvm}^{*}\}_{(K \times V \times M)}$ after an inference step, we proceed to optimize the model's hyper-parameters $\boldsymbol{\xi}$ and $\{\varphi_{vk}\}_{V \times K}$ by maximizing ELBO on all documents $\mathcal{D}$:
\begin{equation*}
    L(\boldsymbol{\xi}, \boldsymbol{\varphi}) = \sum_{m=1}^{M}\left\{\boldsymbol{\eta}(\boldsymbol{\xi})^{\mathrm{T}} \mathbb{E}_{q(\boldsymbol{\theta}|\boldsymbol{\zeta}_m^{*})} \left[ \mathbf{u}(\boldsymbol{\theta}) \right] - \log g(\boldsymbol{\eta}(\boldsymbol{\xi})) + \sum_{n=1}^{N_m} {\boldsymbol{\phi}^{*}_{nm}}^{\mathrm{T}} \log \boldsymbol{\varphi}_{v(nm)} \right\} + C
\end{equation*}

\vspace{1em}
\noindent
\textbullet\; \textbf{Optimization of $\boldsymbol{\varphi}$}
\vspace{0.5em}

\noindent
ELBO with respect to $\boldsymbol{\varphi}$ is given as:
\begin{equation*}
    L(\boldsymbol{\varphi}) = \sum_{m=1}^{M} \sum_{n=1}^{N_m} {\boldsymbol{\phi}^{*}_{nm}}^{\mathrm{T}} \log \boldsymbol{\varphi}_{v(nm)} + C
\end{equation*}
Notice that $\boldsymbol{\varphi}$ has $K$ constraints:
\begin{equation*}
    \sum_{l=1}^{V} \varphi_{lk} = 1, \;\;\;\; k = 1, 2, \cdots, K
\end{equation*}
and the Lagrangian is given as:
\begin{equation*}
\begin{aligned}
        \mathcal{L}(\boldsymbol{\varphi})
        &= \sum_{m=1}^{M} \sum_{n=1}^{N} {\boldsymbol{\phi}^{*}_{nm}}^{\mathrm{T}} \log \boldsymbol{\varphi}_{v(nm)} + \sum_{k=1}^{K} \lambda_k\left(\sum_{v=1}^{V} \varphi_{vk} - 1\right) \\
        &= \sum_{v=1}^{V} \sum_{k=1}^{K} \left( \sum_{m=1}^{M} \phi^{*}_{kvm} n_{vm} \right) \log \varphi_{vk} + \sum_{k=1}^{K} \lambda_k \left( \sum_{v=1}^{V} \varphi_{vk} - 1 \right)
\end{aligned}
\end{equation*}
Without loss of generality, the partial derivatives of Lagrangian with respect to $\varphi_{vk}$ is written as:
\begin{equation*}
        \frac{\partial \mathcal{L}}{\partial \varphi_{vk}} = \left( \sum_{m=1}^{M} \phi^{*}_{kvm} n_{vm} \right) \frac{1}{\varphi_{vk}} + \lambda_k
\end{equation*}
By setting the partial derivatives to zero, we have:
\begin{equation*}
    {\varphi_{vk}}^{*} = - \frac{1}{\lambda_k} \left( \sum_{m=1}^{M} \phi^{*}_{kvm} n_{vm} \right)
\end{equation*}
Therefore, the optimized $\{\varphi_{vk}\}_{(V \times K)}^{*}$ is derived as:
\begin{equation*}
    \{\varphi_{vk}\}_{(V \times K)}^{*} = \left\{ {z_k}^{-1} \sum_{m=1}^{M} \phi_{kvm}^{*} n_{vm} \right\}_{(V \times K)}
\end{equation*}
where $\left\{z_k\right\}_K$ is the normalizer.

\vspace{1em}
\noindent
\textbullet\; \textbf{Optimization of $\boldsymbol{\xi}$}
\vspace{0.5em}

\noindent
ELBO with respect to $\boldsymbol{\xi}$ is given as:
\begin{equation*}
    L(\boldsymbol{\xi}) = {\boldsymbol{\eta}(\boldsymbol{\xi})}^{\mathrm{T}} \left( \sum_{m=1}^{M} \mathbb{E}_{q(\boldsymbol{\theta}|\boldsymbol{\zeta}_m^{*})} \left[ \mathbf{u}(\boldsymbol{\theta}) \right] \right) - M \log g(\boldsymbol{\eta}(\boldsymbol{\xi})) + C
\end{equation*}
The first-order derivative is written as:
\begin{equation*}
    \nabla_{\boldsymbol{\xi}} L(\boldsymbol{\xi}) = \left\{ \sum_{m=1}^{M} \mathbb{E}_{q(\boldsymbol{\theta}|\boldsymbol{\zeta}_m)} \left[ u_d(\boldsymbol{\theta}) \right] - M \mathbb{E}_{p(\boldsymbol{\theta}|\boldsymbol{\xi})} \left[ u_d(\boldsymbol{\theta}) \right] \right\}_{d=1}^{D}
\end{equation*}
By setting the derivative to zero, we can see that the optimized $\boldsymbol{\xi}$ is obtained by matching the moment of hyper-Dirichlet with parameter $\boldsymbol{\xi}$ to the average of those of the inferred posteriors, which is a similar process to the message passing in Expectation Propagation.
Thus, we proceed with Newton's method, and the Hessian is derived as:
\begin{equation*}
    \mathbf{H}_{\boldsymbol{\xi}} L(\boldsymbol{\xi}) = -M \left\{ \frac{\partial}{\partial \xi_i} \mathbb{E}_{p(\boldsymbol{\theta}|\boldsymbol{\xi})} \left[ u_j (\boldsymbol{\theta}) \right] \right\}_{(D \times D)}
\end{equation*}
The optimized $\boldsymbol{\xi}^{*}$ can be approximated iteratively by Newton's method:
\begin{equation*}
    \boldsymbol{\xi}^{(t+1)} := \boldsymbol{\xi}^{(t)} - {\mathbf{H}(\boldsymbol{\xi}^{(t)})}^{-1} \nabla(\boldsymbol{\xi}^{(t)})
\end{equation*}
The tensor expression of the whole Expectation-Maximization is shown in Figure~\ref{fig:vi_em}.

\section{Expectation Propagation}

In Bayesian inference, we compute the posterior distribution on latent variables given the observation.
In mean-field variational inference, the algorithm employs a global mean-field distribution to perform approximation for each latent node in the probabilistic graph.
In contrast, Expectation Propagation (EP) separates target distributions into terms, approximate each one and combine together to obtain a global approximation.

\subsection{Formal Alignment of Approximate Distribution to Target Posterior}

In EP for LDTA, we focus on the marginal posterior on $\boldsymbol{\theta}$ given the observed sequence of words incorporating the model's hyper-parameters.
For a single document we have the target marginal posterior:
\begin{equation*}
\begin{aligned}
        p(\boldsymbol{\theta}|\boldsymbol{w}, \boldsymbol{\varphi}, \boldsymbol{\xi})
        &= \frac{p(\boldsymbol{w}, \boldsymbol{\theta}|\boldsymbol{\varphi}, \boldsymbol{\xi})}{p(\boldsymbol{w}|\boldsymbol{\varphi}, \boldsymbol{\xi})} \\
        &\propto p(\boldsymbol{\theta}|\boldsymbol{\xi}) \prod_{v=1}^{V} \left( \sum_{k=1}^{K} \varphi_{vk} \theta_k \right)^{n_{v}} \\
        &= p(\boldsymbol{\theta}|\boldsymbol{\xi}) \exp \Bigl \langle n_v, \log t_v(\boldsymbol{\theta}; \boldsymbol{\varphi}) \Bigr \rangle_{v} \\
\end{aligned}
\end{equation*}
where
\begin{equation*}
    t_{v}(\boldsymbol{\theta};\boldsymbol{\varphi})
    = \sum_{k=1}^{K} \varphi_{vk} \theta_{k}
\end{equation*}
We can see that the hierarchical probabilistic graph leads to the transform of likelihood breaking the conjugacy, where the compound of $\boldsymbol{\theta}$ and $\boldsymbol{\varphi}$ makes it impossible to revise the parameters directly from prior, and eventually causes the whole posterior to be intractable.
Therefore, we need to employ approximating methods such as Variational Inference or Expectation Propagation with tractable distributions to estimate the intractable target posterior.
The mean-field variational inference is introduced in the last section, and we demonstrate the Expectation Propagation routine in the following.

\vspace{1em}
\noindent
\textbullet\; \textbf{Dual conjugate expressions}
\vspace{0.5em}

\noindent
In EP, we try to approximate the target distribution by an approximating distribution with the similar form:
\begin{equation*}
    q(\boldsymbol{\theta})
    \propto p(\boldsymbol{\theta}|\boldsymbol{\xi}) \exp \Bigl\langle n_v, \log \Tilde{t}_v (\boldsymbol{\theta})\Bigr\rangle_{v}
\end{equation*}
In contrast to variational mean-field where we simply optimize the approximating distribution's parameter by maximizing the ELBO, EP performs a more in-depth fine-grained calibration based on the structure of target distribution by constraining the functional space of the approximating distributions.
In addition to the requirements of tractablility such as pre-normalization and independence that variational mean-field complies, EP further expects the production form of the approximating distribution, in accordance with the prior-likelihood-points product form of the target distribution.
Moreover, due to the nature of likelihood, the observation $\boldsymbol{n_{(v)}}$ is incorporated into the approximator, which in effect imposes a constraint on the functional space of the approximating distribution.
Given the bijective correspondence of the terms between the approximator and the target, distributed term-to-term approximation is achieved by moment matching in a serial or parallel manner, while the simplicity of global normalization is maintained.
Based on the further requirements on approximator in EP, Dirichlet-Tree with the same selection matrix $\mathbf{D}$ as the prior's is chosen as the approximator of target posterior, thanks to its dual conjugate expressions:
\begin{equation*}
\begin{aligned}
        q(\boldsymbol{\theta}|\boldsymbol{\zeta}_{\boldsymbol{n}})
        &= q(\boldsymbol{\theta}|\boldsymbol{\xi}+ \mathbf{D}_{q} \boldsymbol{n_{(k)}}) \\[1ex]
        &= \mathbf{B}_{q} (\boldsymbol{n_{(k)}}) p(\boldsymbol{\theta}|\boldsymbol{\xi}) \\[1ex]
        &= \left\{ \left[ \frac{g(\boldsymbol{\eta}(\boldsymbol{\zeta}_{\boldsymbol{n}}))}{g(\boldsymbol{\eta}(\boldsymbol{\xi}))} \right]^{-1} \exp \Bigl\langle n_k, \log \theta_k \Bigr\rangle_{k} \right\} p(\boldsymbol{\theta}|\boldsymbol{\xi}) \\[1ex]
        &\propto p(\boldsymbol{\theta}|\boldsymbol{\xi}) \exp \Bigl\langle n_k, \log \theta_k \Bigr\rangle_{k}
\end{aligned}
\end{equation*}
where $\mathbf{D}_{q}$ is Dirichlet selection matrix and $\mathbf{B}_{q}$ the Bayesian operator.
We call the first expression whose parameters are directly revised the explicit form of approximator, and the second expression conforming the target the implicit form of approximator.
It is noteworthy to mention that the global normalizer is determined if latent $\{ n_k \}_K$ on the set of topic  is given.

\vspace{1em}
\noindent
\textbullet\; \textbf{Explicitization of pseudo-observation}
\vspace{0.5em}

\noindent
Even though it closely resembles the target distribution, the inner product expression for the likelihood terms is with respect to the topic dimension $k$ which relies on the latent variable $\boldsymbol{z}$, instead of the word dimension $v$ in those of the target distribution which is drawn directly from the observed $\boldsymbol{w}$.
To continue constructing our approximate distribution, we introduce the transition matrix $\{\phi_{kv}\}_{(K \times V)}$, which leads to the counting matrix of pseudo-observation $\{\phi_{kv} n_v\}_{(K \times V)}$, so that the vector $\{\phi_{kv} n_v\}_K$ approximates to $\{n_k\}_K$.
Observation transition matrix $\boldsymbol{\phi}$ can roughly be viewed as the pseudo-proportions of $k^{th}$ topic in the authentic count of $v^{th}$ word in a document, therefore the normalization is required.
And we call the reverse transform procedure $\{ \phi_{kv} n_v \}_{(K \times V)}$, which formulates an explicit counting matrix representing both the observed and the latent variables, the explicitization of pseudo-observation on latent variable---in our case the latent topics as direct results of the stochastic procedure subject to the multinomial whose parameters are dominated by the Dirichlet-Tree prior. 
We say ``pseudo" because the reverse transform does not strictly reflect the genuine integer counts as described in our model, but gives the decimal estimations instead. 
Therefore, we have:
\begin{equation*}
\begin{aligned}
    q(\boldsymbol{\theta}|\boldsymbol{\zeta}_{\boldsymbol{n}})
    &\propto p(\boldsymbol{\theta}|\boldsymbol{\xi}) \exp \biggl\langle \sum_{v} \phi_{kv} n_v, \log \theta_k \biggr\rangle_k \\[1ex]
    &= p(\boldsymbol{\theta}|\boldsymbol{\xi}) \exp \sum_v n_v \Bigl \langle \phi_{kv},\log \theta_k \Bigr \rangle_k \\[1ex]
    &= p(\boldsymbol{\theta}|\boldsymbol{\xi}) \exp \Bigl \langle n_v, \bigl \langle \phi_{kv}, \log \theta_k \bigr  \rangle_k \Bigr \rangle_v \\[1ex]
    &= p(\boldsymbol{\theta}|\boldsymbol{\xi}) \exp \biggl\langle n_v, \log \prod_k {\theta_k}^{\phi_{kv}} \biggr\rangle_v
\end{aligned}
\end{equation*}

It is easy to tell that the explicitization of pseudo-observation is the process of reversely transforming from the authentic observed variable to some approximated latent variables, which are usually the unseen results of random process controlled by the parameters of interest, so that the information from authentic observation can be passed to an object which helps build a tractable explicit approximate posterior on those parameters of interest.
As desired by the inference, the resulted approximate posterior remains a similar implicit product form in compliance with the target prior-likelihood-points posterior. 
We call the object the pseudo-observation.

\vspace{1em}
\noindent
\textbullet\; \textbf{Term-wise synchronization}
\vspace{0.5em}

\noindent
Now, we have a quite similar form to the target distribution with $t(\boldsymbol{\theta};\boldsymbol{\varphi})$ replaced by $\prod_k {\theta_k}^{\phi_{kv}}$. 
According to Jensen's inequality:
\begin{equation*}
        \prod_{k=1}^{K} {\theta_{k}}^{\phi_{kv}}
        = \exp \left\{ \sum_{k=1}^{K} \phi_{kv} \log \theta_k \right\}  
        \leq \sum_{k=1}^{K} \phi_{kv} \theta_k  
\end{equation*}
where $\sum_{k=1}^{K} \phi_{kv} = 1$.
However, we have:
\begin{equation*}
        0 \leq \sum_{k=1}^{K} \varphi_{vk} \leq K
\end{equation*}
It is likely that the corresponding approximate term will be smaller than the target term---the bigger the topic cardinality $K$ is, the more likely $\prod_{k} {\theta_{k}}^{\phi_{kv}}$ will be smaller than $t(\boldsymbol{\theta};\boldsymbol{\varphi})$.
To solve this inherent limitation, we add a coefficient for each approximate term to synchronize with the target term, which will be removed by global normalizer later.
As a result, we obtain our final form of approximating likelihood terms:
\begin{equation*}
    \Tilde{t}_{v}(\boldsymbol{\theta};\boldsymbol{\phi})
    = s_{v} \prod_{k=1}^{K} {\theta_{k}}^{\phi_{kv}}
\end{equation*}
where $s_{v}$ is the synchronization coefficient such that each $\Tilde{t}_v(\boldsymbol{\theta})$ approximates $t_v(\boldsymbol{\theta})$.
With each approximating term obtained, we derive our final approximating distribution expressions:
\begin{equation*}
\begin{aligned}
        q(\boldsymbol{\theta}|\boldsymbol{\zeta_n})
        &= q(\boldsymbol{\theta}|\boldsymbol{\xi} + \mathbf{D}_{q} \boldsymbol{\phi} \boldsymbol{n_{(v)}}) \\[1ex]
        &= \mathbf{B}_{q} (\boldsymbol{\phi} \boldsymbol{n_{(v)}}) p(\boldsymbol{\theta}|\boldsymbol{\xi}) \\[1ex]
        &= \left[ \frac{g(\boldsymbol{\eta}(\boldsymbol{\zeta_n}))}{g(\boldsymbol{\eta}(\boldsymbol{\xi}))} \exp \Bigl \langle n_v, \log s_v \Bigr \rangle_v \right]^{-1} p(\boldsymbol{\theta}|\boldsymbol{\xi}) \exp \biggl \langle n_v, \log s_v \prod_{k=1}^{K} {\theta_k}^{\phi_{kv}} \biggr \rangle_{v}  \\[1ex]
        &\propto p(\boldsymbol{\theta}|\boldsymbol{\xi}) \exp \Bigl \langle n_v, \log \Tilde{t}_v(\boldsymbol{\theta};\boldsymbol{\phi}) \Bigr \rangle_v
\end{aligned}
\end{equation*}
Now, it is clear that the transition matrix $\{\phi_{kv}\}_{K \times V}$ works as a transformer which converts the observation over vocabulary to pseudo-observation over latent topics to achieve a conjugate posterior on the latent variable of interest, and to build a series of cell likelihoods in compliance with the likelihood of observations conforming the target posterior as a whole without the loss of global normalization. 
Given the transition matrix $\{\phi_{kv}\}_{K \times V}$, observation $\boldsymbol{n_{(v)}}$, and hyper-parameter $\boldsymbol{\xi}$, the parameter $\boldsymbol{\zeta_{n}}$ for the explicit approximating distribution is given as follows:
\begin{equation*}
    \boldsymbol{\zeta_{n}} = \boldsymbol{\xi} \oplus \boldsymbol{\phi}\boldsymbol{n_{(v)}}
\end{equation*}
It is obvious to tell that the implicit form of approximate distribution is now in alignment with the target posterior which holds a prior-likelihood-points product form. 
The process of achieving this bijectively corresponding alignment from approximate distribution to the target posterior is called the alignment of approximate distribution. 
Figure \ref{fig:ep} illustrates the whole picture over the relationships among the target posterior, the explicit approximate distribution, and implicit approximate distribution.

\begin{figure}[t]
    \centering
    \includegraphics[width=0.8\linewidth]{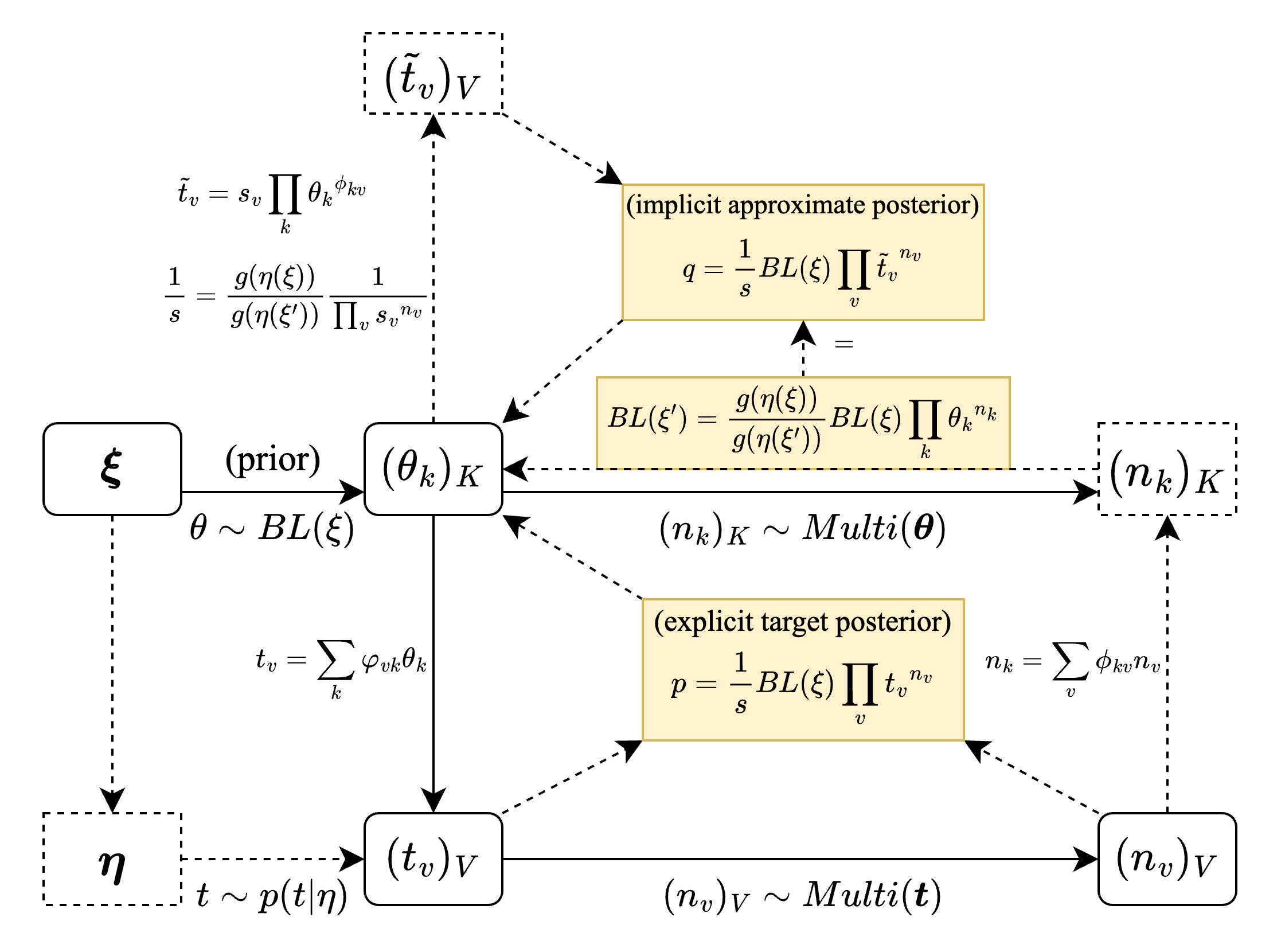}
    \caption{Construction of approximate posterior distribution and its relationship with other parameters. The $(n_v)_V$ and $(n_k)_K$ represents observed count vector of words, and latent count vector of topics, respectively. In explicit posterior distribution, the normalizer, or the evidence $s(\xi, \boldsymbol{\varphi})$ is intractable, which is the reason why we have to use approximate method. Notice $\prod_{k} {\theta_{k}}^{\sum_{v} \phi_{kv} n_{v}} = \prod_{v} \left({\prod_{k} {\theta_{k}}^{\phi_{kv}}}\right)^{n_{v}}$. $\xi^\prime$ is the revision of prior parameter $\xi$ which incorporates the latent count of topics $(n_k)_K$ based on the property of conjugacy.}
    \label{fig:ep}
\end{figure}

\subsection{Cavity and Tilted Distributions}

Given the explicit and implicit expression of the approximating distribution, the approximator's statistical manifold is fixed and the task now is to determine an optimized group of the paramter $(\boldsymbol{\phi}^{*}, \boldsymbol{s_{(v)}}^{*})$ as the final approximation to target posterior. 
With the term-to-term bijective correspondence between the approximating distribution and the target posterior, the problem of global approximation is transformed into the approximation of each term $\Tilde{t}_v(\boldsymbol{\theta};\boldsymbol{\phi})$ to $t_v(\boldsymbol{\theta}; \boldsymbol{\varphi})$ individually. 
A series of cavity and tilted distributions, $\{q(\boldsymbol{\theta}|\boldsymbol{\zeta}^{\backslash v})\}_V$ and $\{p_{v}^{*}(\boldsymbol{\theta})\}_V$ with respect to each term $t_v(\boldsymbol{\theta};\boldsymbol{\varphi})$ or $\Tilde{t}_v(\boldsymbol{\theta}; \boldsymbol{\phi})$ are built to perform the proceeding approximations in a distributed manner.

Theoretically, for an ideal approximator $q(\boldsymbol{\theta}|\boldsymbol{\zeta}^{*})$ where $\{\Tilde{t}_v (\boldsymbol{\theta};\boldsymbol{\phi}^{*}) = t_v(\boldsymbol{\theta};\boldsymbol{\varphi})\}_V$, it would cause no effect if any $\Tilde{t}_v(\boldsymbol{\theta};\boldsymbol{\phi}^{*})$ is devided out and replaced by $t_v (\boldsymbol{\theta};\boldsymbol{\phi})$. 
However, for any initialized value $\boldsymbol{\phi}^{(0)}$, the statement will not hold in general, meaning that the resulted new distributions are neither normalized nor proportional to the original approximator. 
The key of EP is that, due to the nature of the approximating distribution we choose, a group of normalized distributions $\{p_{v}^{*}(\boldsymbol{\theta})\}_V$ with respect to every individual term, which we call the titled distributions, can be formulated by unloading the approximate term and loading the target term, so that one local individual term is separated and focused on each time and further optimizations can be performed with respect to each term iteratively. 
Suppose we have the approximator $q(\boldsymbol{\theta}|\boldsymbol{\zeta}^{(t)})$ in an iteration. 
Normalized as they are, the corresponding $\{p_v^{*\;(t)} (\boldsymbol{\theta})\}_{V}$ are not proportional to $q(\boldsymbol{\theta}|\boldsymbol{\zeta}^{(t)})$ because there is still a distance between $q(\boldsymbol{\theta}|\boldsymbol{\zeta}^{(t)})$ and target $p(\boldsymbol{\theta}|\boldsymbol{n_{(v)}}, \boldsymbol{\varphi}, \boldsymbol{\xi})$. 
Besides, the tilted distributions containing information from both $t_v(\boldsymbol{\theta};\boldsymbol{\varphi})$ and $q(\boldsymbol{\theta}|\boldsymbol{\zeta}^{(t)})$ locate between the target and current approximator. 
Starting from the ``imperfect" initial approximator $q(\boldsymbol{\theta}|\boldsymbol{\zeta}^{(0)})$, the approximator is drawn closer to its titled distributions in each iteration, and therefore is drawn closer to the target posterior, until the iteration converges when $q(\boldsymbol{\theta}|\boldsymbol{\zeta}^{*})$, $\{p_v^{*}(\boldsymbol{\theta})\}_V$, and $p(\boldsymbol{\theta}|\boldsymbol{n_{(v)}}, \boldsymbol{\varphi}, \boldsymbol{\xi})$ almost overlap (Figure \ref{fig:ep_manifold}).

\begin{figure}[t]
    \centering
    \includegraphics[width=0.9\linewidth]{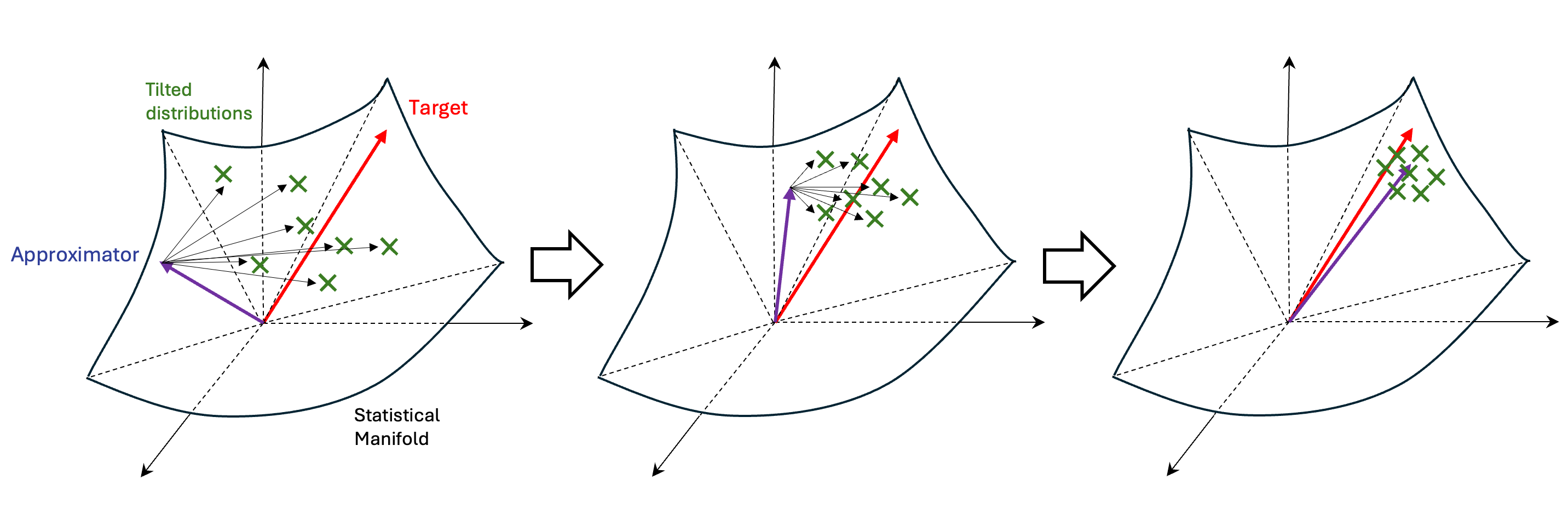}
    \caption{EP inference in each iteration. The statistical manifold is conceptually illustrated with three dimensions. The target, the approximator and the tilted distributions are all located on the statistical Riemann manifold thanks to the tractability of the approximating distribution we choose. The approximator is drawn to the target by the group of tilted distributions in each iteration.}
    \label{fig:ep_manifold}
\end{figure}

The update of the approximator is further elaborated as follows. 
In one iteration, suppose we have the transition matrix $\boldsymbol{\phi}^{(t)}$, and the cavity distributions $\{q(\boldsymbol{\theta}|\boldsymbol{\zeta}^{\backslash v})\}_{V}$ playing the role as a group of background distributions which are formulated as:
\begin{equation*}
    q(\boldsymbol{\theta}|\boldsymbol{\zeta}^{\backslash v})
    = \mathbf{B}^{-1}(\boldsymbol{\phi}_{v}) \mathbf{B}(\boldsymbol{\phi}\boldsymbol{n_{(v)}}) p(\boldsymbol{\theta}|\boldsymbol{\xi})
\end{equation*}
Notice that the cavity distributions are naturally normalized due to tractability of approximating distribution.
Subsequently, the normalized tilted distributions are derived as:
\begin{equation*}
    p_{v}^{*}(\boldsymbol{\theta}) = {z_v}^{-1} \left( \sum_{k=1}^{K} \varphi_{vk} \theta_k \right) q(\boldsymbol{\theta}|\boldsymbol{\zeta}^{\backslash v})
\end{equation*}
where the normalizer $\{z_v\}_V$ is trivial to compute:
\begin{equation*}
    z_v 
    = \int_{\boldsymbol{\theta}} \left( \sum_{k=1}^{K} \varphi_{vk} \theta_k \right) q(\boldsymbol{\theta}|\boldsymbol{\zeta}^{\backslash v}) \mathrm{d}\boldsymbol{\theta} 
    = {\boldsymbol{\varphi}_v}^{\mathrm{T}} \mathbb{E}_{q(\boldsymbol{\theta}|\boldsymbol{\zeta}^{\backslash v})} \left[ \boldsymbol{\theta} \right]
\end{equation*}
The messages from the target term contained in the titled distributions are subsequently passed to a group of new separate approximators, which are of same selection matrix $\mathbf{D}_q$ as the current approximator, by minimizing the KL-divergence:
\begin{equation}
        \boldsymbol{\zeta}_{v}^{sep} := \arg \min_{\boldsymbol{\zeta}_{v}^{sep}} \mathrm{KL}(p_v^{*}(\boldsymbol{\theta})||q(\boldsymbol{\theta}|\boldsymbol{\zeta}_{v}^{sep}))
\label{eq:message_passing}
\end{equation}
This process is equivalent to match the moment of two distributions and will be introduced in the next section. 
After obtaining the optimized $\{\boldsymbol{\zeta}_{v}^{sep}\}_V$, the parameter $\boldsymbol{\phi}$ is updated such that:
\begin{equation*}
        q(\boldsymbol{\theta}|\boldsymbol{\zeta}_{v}^{sep})
        = \mathbf{B}(\boldsymbol{\phi}_{v}^{(t+1)}) q(\boldsymbol{\theta}|\boldsymbol{\zeta}^{\backslash v}) 
        = q(\boldsymbol{\theta}|\boldsymbol{\zeta}^{\backslash v} + \mathbf{D}_q \boldsymbol{\phi}_{v}^{(t+1)}) 
\end{equation*}
and
\begin{equation*}
    \boldsymbol{\phi}_{v}^{(t+1)} := \mathbf{D}_q^{-1}\left(\boldsymbol{\zeta}_{v}^{sep} - \boldsymbol{\zeta}^{\backslash v}\right)
\end{equation*}
where $\mathbf{D}_q^{-1}$ is the inverse selection operator.
In practice, we only focus on the leaf nodes of $\mathbf{D}_q$ to obtain the updated $\boldsymbol{\phi}^{(t+1)}$.

The optimization of synchronization coefficient $\{s_v\}_V$ is the next step. 
Suppose we have $\boldsymbol{\phi}^{(t)}$ in an iteration, and so we have the approximate distribution (regardless of explicit or implicit form): $q(\boldsymbol{\theta}|\boldsymbol{\zeta}^{(t)}) = \mathbf{B}(\boldsymbol{\phi}^{(t)} \boldsymbol{n_{(v)}}) p(\boldsymbol{\theta}|\boldsymbol{\xi})$. 
Ideally, if $q(\boldsymbol{\theta}|\boldsymbol{\zeta}^{(t)})$ approximates target posterior enough, the following equation of the corresponding tilted distributions should hold:

\begin{equation*}
\begin{aligned}
    q_{v}^{*(t)}(\boldsymbol{\theta})
    &\equiv q(\boldsymbol{\theta}|\boldsymbol{\zeta}^{(t)}) \left( s_v \prod_{k=1}^{K} \theta_k^{\phi_{kv}^{(t)}} \right)^{-1} \left( \sum_{k=1}^{K} \varphi_{vk} \theta_k \right) \\[1ex]
    &= \mathbf{B}^{-1}(\boldsymbol{\phi}_v) \mathbf{B}(\boldsymbol{\phi} \boldsymbol{n_{(v)}}) p(\boldsymbol{\theta}|\boldsymbol{\xi}) \left( \sum_{k=1}^{K} \varphi_{vk} \theta_k \right) \frac{g(\boldsymbol{\eta}(\boldsymbol{\zeta}^{\backslash v}))}{g(\boldsymbol{\eta}(\boldsymbol{\zeta}))} {s_v}^{-1} \\[1ex]
    &\approx \mathbf{B}^{-1}(\boldsymbol{\phi}_v) \mathbf{B}(\boldsymbol{\phi} \boldsymbol{n_{(v)}}) p(\boldsymbol{\theta}|\boldsymbol{\xi}) \left( \sum_{k=1}^{K} \varphi_{vk} \theta_k \right) {z_v}^{-1} 
\end{aligned}
\end{equation*}
where $v=1,2,\cdots, V$. 
Consequently, we obtain the update of $\{s_v\}_V$:
\begin{equation*}
    {s_v}^{(t)} := \frac{g(\boldsymbol{\eta}(\boldsymbol{\zeta}^{\backslash v \; (t)}))}{g(\boldsymbol{\eta}(\boldsymbol{\zeta}^{(t)}))} {z_v}^{(t)}
\end{equation*}

According to the above discussion, the parameter update as preparation for message passing in each iteration is concluded as Algorithm \ref{alg:essential_parameters}. 
The whole inference procedure for a single document is concluded as Algorithm \ref{alg:ep_inference}. 
The relationship between approximator, target, cavity and tilted distributions is illustrated in Figure \ref{fig:rela_cavity_tilted}.

\begin{figure}[t]
    \centering
    \includegraphics[width=0.8\linewidth]{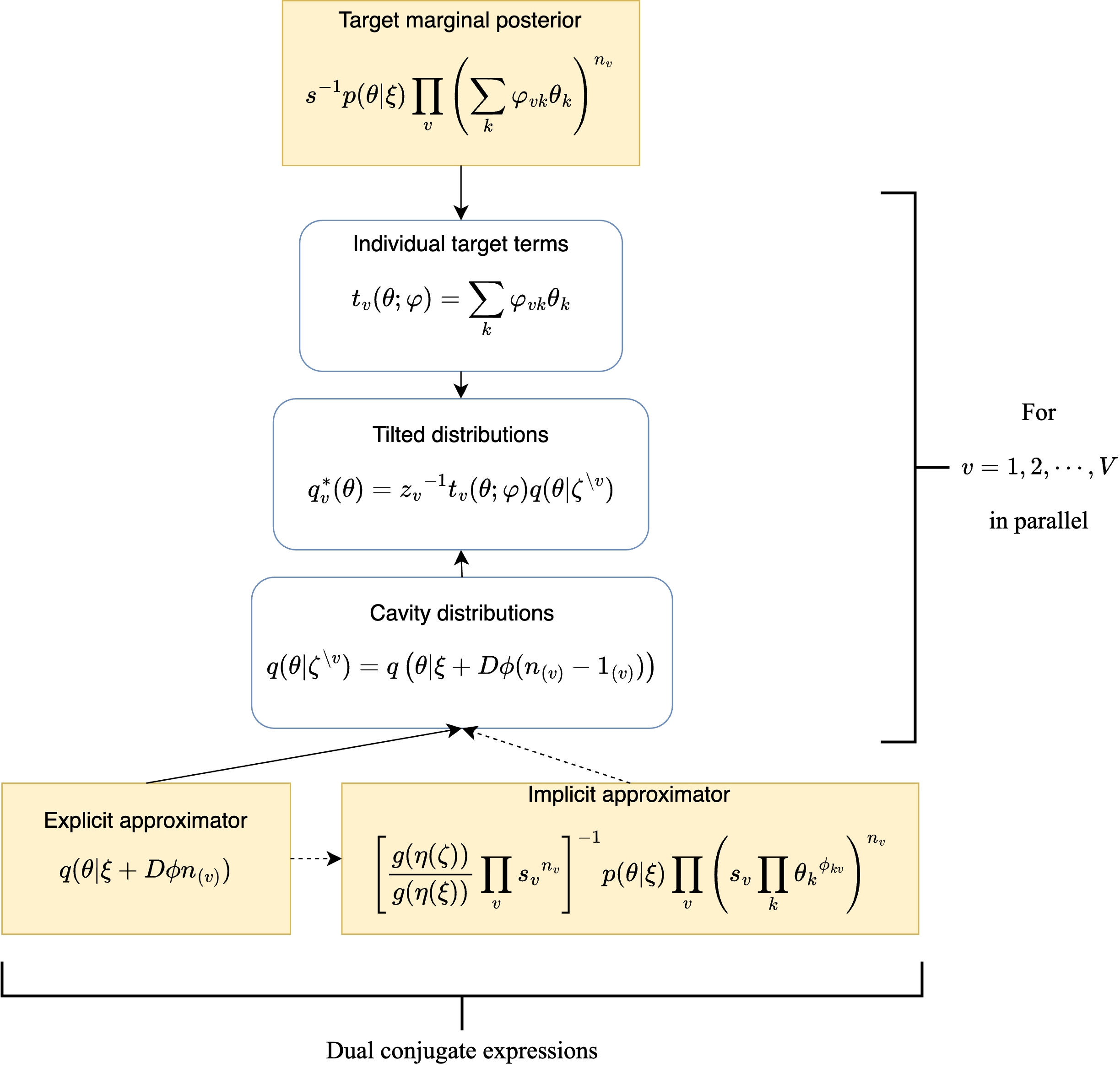}
    \caption{Relationships of Cavity and Tilted distributions}
    \label{fig:rela_cavity_tilted}
\end{figure}

\subsection{Parallel Message Passing}

This section elaborates in detail the process of message passing, which is implemented by moment matching as indicated in \eqref{eq:message_passing}.
The purpose of message passing is to pass the information of target terms $\{t_v(\boldsymbol{\theta};\boldsymbol{\varphi})\}_V$ from the tilted distributions $\{p_{v}^{*}(\boldsymbol{\theta})\}_V$ to a group of separate approximators $\{q(\boldsymbol{\theta}|\boldsymbol{\zeta}_{v}^{sep})\}_V$ respectively, so that the updated $\boldsymbol{\phi}^{(t+1)}$ can subsequently be obtained by $\{q(\boldsymbol{\theta}|\boldsymbol{\zeta}_{v}^{sep})\}_V$ peeling off the backgroud $\{q(\boldsymbol{\theta}|\boldsymbol{\zeta}^{\backslash v})\}_V$ to support the newer approximation $\{\Tilde{t}_v(\boldsymbol{\theta};\boldsymbol{\phi}^{(t+1)})\}_V$ to $\{t_v(\boldsymbol{\theta};\boldsymbol{\varphi})\}_V$. 
The moment matching is achieved in parallel with respect to index $v$.

According to Theorem~\ref{thm:moment_matching} in Appendix~\ref{app:properties_expo}, the optimized separate approximators $\{q(\boldsymbol{\theta}|\boldsymbol{\zeta}_{v}^{sep})\}_V$ defined in \eqref{eq:message_passing} are the ones whose expectations on sufficient statistics equal to those expectations of the titled distributions. 
In other words, the optimization process of $\{q(\boldsymbol{\theta}|\boldsymbol{\zeta}_{v}^{sep})\}_V$ is equivalent to matching the moment of $\{q(\boldsymbol{\theta}|\boldsymbol{\zeta}_{v}^{sep})\}_V$ to $\{p_v^{*}(\boldsymbol{\theta})\}_V$. 
This parallel optimization process can be expressed as follows. 
On the one hand, thanks to the property of exponential family, the expectations on sufficient statistics of separate approximators are easily derived, which span a $(D \times V)$ matrix $\mathbf{Q}(\boldsymbol{\zeta}^{sep})$ with unknowns $\boldsymbol{\zeta}^{sep}$:
\begin{equation*}
        \mathbf{Q} =\left(
                        \begin{array}{ccccc}
                            \mathbb{E}_{q_1^{sep}} \left[ u_1 (\boldsymbol{\theta}) \right] & \cdots & \mathbb{E}_{q_v^{sep}} \left[ u_1 (\boldsymbol{\theta}) \right] & \cdots & \mathbb{E}_{q_V^{sep}} \left[ u_1 (\boldsymbol{\theta}) \right]\\
                            \vdots &        & \vdots &        & \vdots\\
                            \mathbb{E}_{q_1^{sep}} \left[ u_d (\boldsymbol{\theta}) \right] & \cdots & \mathbb{E}_{q_v^{sep}} \left[ u_d (\boldsymbol{\theta}) \right] & \cdots & \mathbb{E}_{q_V^{sep}} \left[ u_d (\boldsymbol{\theta}) \right]\\
                            \vdots &        & \vdots &        & \vdots\\
                            \mathbb{E}_{q_1^{sep}} \left[ u_{D} (\boldsymbol{\theta}) \right] & \cdots & \mathbb{E}_{q_v^{sep}} \left[ u_{D} (\boldsymbol{\theta}) \right] & \cdots & \mathbb{E}_{q_V^{sep}} \left[ u_{D} (\boldsymbol{\theta}) \right]
                        \end{array}
                    \right)
\end{equation*}
On the other hand, the $(D \times V)$-dimensional  expectation matrix $\mathbf{P}(\boldsymbol{\phi}^{(t)})$ on the same group of sufficient statistics with respect to the tilted distributions $\{p_{v}^{*}(\boldsymbol{\theta})\}_V$ is formulated as:
\begin{equation*}
        \mathbf{P} =\left(
                        \begin{array}{ccccc}
                            \mathbb{E}_{p_1^*} \left[ u_1 (\boldsymbol{\theta}) \right] & \cdots & \mathbb{E}_{p_v^*} \left[ u_1 (\boldsymbol{\theta}) \right] & \cdots & \mathbb{E}_{p_V^*} \left[ u_1 (\boldsymbol{\theta}) \right]\\
                            \vdots &        & \vdots &        & \vdots\\
                            \mathbb{E}_{p_1^*} \left[ u_d (\boldsymbol{\theta}) \right] & \cdots & \mathbb{E}_{p_v^*} \left[ u_d (\boldsymbol{\theta}) \right] & \cdots & \mathbb{E}_{p_V^*} \left[ u_d (\boldsymbol{\theta}) \right]\\
                            \vdots &        & \vdots &        & \vdots\\
                            \mathbb{E}_{p_1^*} \left[ u_{D} (\boldsymbol{\theta}) \right] & \cdots & \mathbb{E}_{p_v^*} \left[ u_{D} (\boldsymbol{\theta}) \right] & \cdots & \mathbb{E}_{p_V^*} \left[ u_{D} (\boldsymbol{\theta}) \right]
                        \end{array}
                    \right)
\end{equation*}
Therefore, the moment matching is more specifically expressed by solving a system of $(D \times V)$ non-linear equations:
\begin{equation}              \mathbf{Q}\left(\boldsymbol{\zeta}^{sep}\left(\boldsymbol{\phi}^{(t+1)}\right)\right) := \mathbf{P}\left(\boldsymbol{\phi}^{(t)}\right)
\label{eq:system_equations_message_passing}
\end{equation}
This system of equations can be more explicitly reduced to several sub-groups of equations with respect to the internal nodes of the tree structure in the following form:
\begin{equation*}
    \psi\left(\zeta_{t\mid s}\right) - \psi\left(\sum_{t\mid s} \zeta_{t\mid s}\right) = C_{t\mid s},\quad s\in \boldsymbol{\Lambda}
\end{equation*}
where $\{\zeta_{t\mid s}\}$ is the sub-group of child nodes of node $s$ and $C_{t\mid s}$ the constant from $\mathbf{P}$. 
Each individual sub-group of equations can be solved numerically by fix-point iteration or Newton's method. 
Notice that $\mathbf{P}$ is determined given $\boldsymbol{\phi}^{(t)}$ of current iteration:
\begin{equation}
\begin{aligned}
        \mathbb{E}_{p_v^*} \left[ u_d (\boldsymbol{\theta}) \right]
        &= \int_{\boldsymbol{\theta}} \frac{u_d(\boldsymbol{\theta})}{z_v} \left( \sum_{k=1}^{K} \varphi_{vk} \theta_k\right) q(\boldsymbol{\theta}|\boldsymbol{\zeta}^{\backslash v}) \mathrm{d}\boldsymbol{\theta} \\[1ex]
        & = \frac{\mathbb{E}_{q(\boldsymbol{\theta}|\boldsymbol{\zeta}^{\backslash v})}[u_d (\boldsymbol{\theta})]}{z_v} \int_{\boldsymbol{\theta}} \left( \sum_{k=1}^{K} \varphi_{vk} \theta_k\right) \frac{1}{\mathbb{E}_{q(\boldsymbol{\theta}|\boldsymbol{\zeta}^{\backslash v})}[u_d (\boldsymbol{\theta})]} u_{d}(\boldsymbol{\theta}) q(\boldsymbol{\theta}|\boldsymbol{\zeta}^{\backslash v}) \mathrm{d}\boldsymbol{\theta}  \\[1ex]
        & = \mathbb{E}_{q(\boldsymbol{\theta}|\boldsymbol{\zeta}^{\backslash v})} [u_d (\boldsymbol{\theta})] \frac{ {\boldsymbol{\varphi}_{v}}^{\mathrm{T}} \mathbb{E}_{q_d(\boldsymbol{\theta}|\boldsymbol{\zeta}^{\backslash v})} [\boldsymbol{\theta}]}{ {\boldsymbol{\varphi}_{v}}^{\mathrm{T}} \mathbb{E}_{q(\boldsymbol{\theta}|\boldsymbol{\zeta}^{\backslash v})} [\boldsymbol{\theta}]} \\[1ex]
        & = \frac{{\boldsymbol{\varphi}_v}^{\mathrm{T}} \left( \mathbb{E}_{q(\boldsymbol{\theta}|\boldsymbol{\zeta}^{\backslash v})} \left[ \boldsymbol{\theta} \right] \odot_k \mathbb{E}_{q(\boldsymbol{\theta}|\boldsymbol{\zeta}^{\backslash v} + \mathbf{D}_q \boldsymbol{1}_{(k)})} \left[ u_d(\boldsymbol{\theta}) \right] \right)}{{\boldsymbol{\varphi}_v}^{\mathrm{T}} \mathbb{E}_{q(\boldsymbol{\theta}|\boldsymbol{\zeta}^{\backslash v})} \left[ \boldsymbol{\theta} \right]} \\[1ex]
        &= \frac{\left( \boldsymbol{\varphi}_v \odot_k \mathbb{E}_{q(\boldsymbol{\theta}|\boldsymbol{\zeta}^{\backslash v})} \left[ \boldsymbol{\theta} \right] \right)^{\mathrm{T}} \mathbb{E}_{q(\boldsymbol{\theta}|\boldsymbol{\zeta}^{\backslash v} + \mathbf{D}_q \boldsymbol{1}_{(k)})} \left[ u_d(\boldsymbol{\theta}) \right]}{\left( \boldsymbol{\varphi}_v \odot_k \mathbb{E}_{q(\boldsymbol{\theta}|\boldsymbol{\zeta}^{\backslash v})} \left[ \boldsymbol{\theta} \right] \right)^{\mathrm{T}} \boldsymbol{1}}
\end{aligned}
\label{eq:tilted_distributions_expectation_sufficient_statistics}
\end{equation}
The \eqref{eq:tilted_distributions_expectation_sufficient_statistics} uses the result from \eqref{eq:expectation_tilted_distributions}, where $\{q(\boldsymbol{\theta}|\boldsymbol{\zeta}^{\backslash v}+\mathbf{D}_q \boldsymbol{1}_{(k)})\}_{(K \times V)}$ is the group of base posteriors of $\{q(\boldsymbol{\theta}|\boldsymbol{\zeta}^{\backslash v})\}_V$.
We can see that the right-hand side of the non-linear equations is reduced to quite a concise and elegant closed-form solution, meaning that the moment of a tilted distribution equals to the weighted average of the counterparts of its corresponding cavity distribution's base posteriors, where the weights are given by the Hadamard product of $\boldsymbol{\varphi}$ and $\{\mathbb{E}_{q(\boldsymbol{\theta}|\boldsymbol{\zeta}^{\backslash v})} \left[ \boldsymbol{\theta} \right]\}_{(V \times K)}$.

\begin{algorithm}
    \caption{Essential parameters update}
    \label{alg:essential_parameters}
    \KwIn{Prior $\boldsymbol{\xi}$, $\{\varphi_{vk}\}_{V \times K}$; Observation $\{n_{v}\}_{V}$; $\{\phi_{kv}\}_{K \times V}$ of current iteration}
    \KwOut{$\boldsymbol{\zeta}$; $\{\boldsymbol{\zeta}^{\backslash v}\}_V$, $\{z_v\}_V$ and $\{s_v\}_V$; evidence $p(\boldsymbol{w}_m)$ of current iteration}
        Compute $\boldsymbol{\zeta} = \boldsymbol{\xi} \oplus \boldsymbol{\phi}_m \boldsymbol{n}_m$\;
        \For{$w_v$ in $\mathbb{W}$ in parallel}
        {
            Compute $\boldsymbol{\zeta}^{\backslash v} = \boldsymbol{\zeta} \ominus \boldsymbol{\phi}_v$\;
            Compute $z_v = {\boldsymbol{\varphi}_v}^{\mathrm{T}} \mathbb{E}_{q(\boldsymbol{\theta}|\boldsymbol{\zeta}^{\backslash v})}[\boldsymbol{\theta}]$\;
            Compute $s_v = z_v\; g(\boldsymbol{\eta}(\boldsymbol{\zeta}^{\backslash v})) / g(\boldsymbol{\eta}(\boldsymbol{\zeta}))$\;
        }
        Compute $p(\boldsymbol{w}_m) = \exp \langle n_v, \log s_v\rangle_v \; g(\boldsymbol{\eta}(\boldsymbol{\zeta})) / g(\boldsymbol{\eta}(\boldsymbol{\xi}))$\;
\end{algorithm}

\begin{algorithm}
    \caption{EP inference for a single document $\boldsymbol{w}_m$}
    \label{alg:ep_inference}
    \KwIn{Prior $\boldsymbol{\xi}$, $\{\varphi_{vk}\}_{V \times K}$; Observation count $\{n_{v}\}_{V}$}
    \KwOut{Optimized $\{{\phi_{kv}}^{*}\}_{(K \times V)}$, $\{{s_{v}}^{*}\}_{V}$; Estimated evidence $p(\boldsymbol{w}_m)$}
    
    Initialize each element in $\{\phi_{kv}\}_{K \times V}$ as $1/K$\;
    Initialize essential parameters as Algorithm \ref{alg:essential_parameters}\;
    \While{not convergence}
    {
        \For{$w_v$ in $\mathbb{W}$ in parallel}
        {
            Compute $\boldsymbol{\zeta}_v^{sep}$ by moment matching as Algorithm \ref{alg:moment_matching}\;
            Compute $\boldsymbol{\phi}_v^{(t+1)} = \boldsymbol{\zeta}_v^{sep} \ominus \boldsymbol{\zeta}^{\backslash v}$\;
            Normalize $\boldsymbol{\phi}_v^{(t+1)}$ such that $\sum_k \phi_{kv} = 1$\;
        }
        Update essential parameters as Algorithm \ref{alg:essential_parameters}\;
    }
\end{algorithm}

\begin{algorithm}
    \caption{Moment matching}
    \label{alg:moment_matching}
    \KwIn{Essential parameters obtained in Algorithm \ref{alg:essential_parameters}}
    \KwOut{Optimized separate approximators $\{\boldsymbol{\zeta}_{v}^{sep}\}_{V}$}
        \For{$w_v$ in $\mathbb{W}$ in parallel}
        {
            Compute $\mathbb{E}_{q(\boldsymbol{\theta}|\boldsymbol{\zeta}^{\backslash v})} \left[ \boldsymbol{\theta} \right]$\;
            \For{$d=1$ to $D$ in parallel}
            {
                \For{$k=1$ to $K$ in parallel}
                {
                    Compute $\mathbb{E}_{q(\boldsymbol{\theta}|\boldsymbol{\zeta}^{\backslash v} + \mathbf{D}_q \boldsymbol{1}_{(k)})} \left[ u_d(\boldsymbol{\boldsymbol{\theta}}) \right]$\;
                }
                Compute $\mathbb{E}_{p_{v}^{*}} \left[ u_d (\boldsymbol{\theta}) \right]$ according to \eqref{eq:tilted_distributions_expectation_sufficient_statistics}\;
                Compute $\{\boldsymbol{\zeta}_{v}^{sep}\}_V$ by Newton's method according to \eqref{eq:system_equations_message_passing}\;
            }
        }
\end{algorithm}

\subsection{Parameter Estimation}

\begin{figure}[t]
    \centering
    \includegraphics[width=1.0\linewidth]{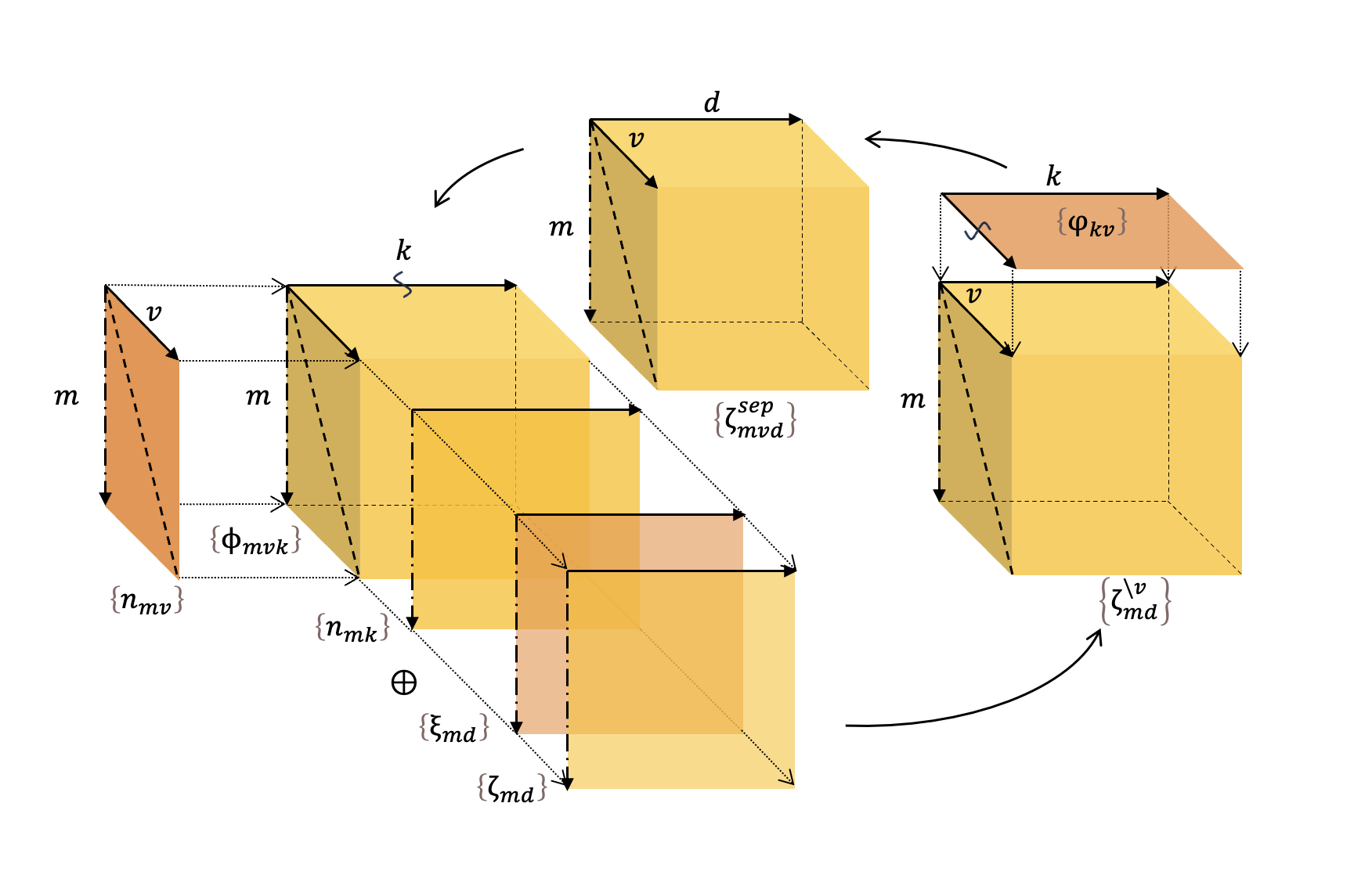}
    \caption{Tensor expression of Expectation Propagation}
    \label{fig:ep_tensor}
\end{figure}

We embedded the EP as the E-step into Expectation-Maximization framework.
The optimized model parameters $\boldsymbol{\xi}$ and $\{\varphi_{vk}\}_{V \times K}$ measured by Kullback–Leibler divergence are estimated by maximizing the Evidence Lower Bound (ELBO) with respect to $\boldsymbol{\xi}$ and $\{\varphi_{vk}\}_{V \times K}$:
\begin{equation}
\label{eq:ep_m}
\begin{aligned}
        L(\boldsymbol{\xi}, \boldsymbol{\varphi}) 
        &= \sum_{m=1}^{M} \mathbb{E}_{q(\boldsymbol{\theta}|\boldsymbol{\zeta}_m)} \left[\log p(\boldsymbol{\theta}|\boldsymbol{\xi})\right] + \sum_{m=1}^{M} \mathbb{E}_{q(\boldsymbol{\theta}|\boldsymbol{\zeta}_m)} \left[ \log p(\boldsymbol{w}_m|\boldsymbol{\varphi}, \boldsymbol{\theta}) \right] + C \\[1ex]
        &= \sum_{m=1}^{M} \left\{ \boldsymbol{\eta}(\boldsymbol{\xi})^{\mathrm{T}}\mathbb{E}_{q(\boldsymbol{\theta}|\boldsymbol{\zeta}_m)} \left[ \mathbf{u}(\boldsymbol{\theta}) \right] - \log g(\boldsymbol{\eta}(\boldsymbol{\xi})) + \mathbb{E}_{q(\boldsymbol{\theta}|\boldsymbol{\zeta}_m)} \left[ \sum_{v=1}^{V} n_{vm} \log \left( \sum_{k=1}^{K} \varphi_{vk}\theta_{k} \right) \right] \right\} + C
\end{aligned}
\end{equation}
The optimization is performed over the collection of all documents $\mathcal{D}$. 
The whole learning algorithm is depicted in Algorithm \ref{alg:total}, and the tensor expression is shown in Figure~\ref{fig:ep_tensor}.

\begin{algorithm}
    \caption{EP-embedded EM algorithm}\label{alg:total}
    \KwIn{Document collection $\mathcal{D}$}
    \KwOut{Approximated distribution parameters $\{\phi_{kvm}\}_{K \times V \times M}$, $\{s_{vm}\}_{V \times M}$, estimated evidence $\{\hat{p}(\boldsymbol{w}_m)\}_{M}$; Model parameters $\boldsymbol{\varphi}$, $\boldsymbol{\xi}$} 
    
    Initialize $\boldsymbol{\varphi}$, $\boldsymbol{\xi}$\;
    \While{not convergence}
    {
        \tcp{EP-Inference}
        \For{$\boldsymbol{w}_m$ in $\mathcal{D}$}
        {
            Given model parameters, approximate the posterior by EP as Algorithm \ref{alg:ep_inference}\;
        }
        \tcp{Maximization}
        Given the approximated posteriors, optimize model parameters according to \eqref{eq:ep_m}
    }
\end{algorithm}

\section{Experiments}

In this section, we test and evaluate LDTA with three different kinds of Dirichlet-Tree priors: the Dirichlet, the Beta-Liouville and the generalized Dirichlet.
The details of these three distributions are listed in Appendix~\ref{app:examples_DT}.
Both two inference methods are examined in three problem domains: document modeling, document and image classification, and a bioinformatics task. 
Our experiments are performed on several publicly available and well-known datasets: NIPS, Reuters-21578, 20 Newsgroups, 15 Scene Categories and Perpheral Blood Mononuclear Cells (PBMC).

\subsection{Document modeling}

\paragraph{Dataset}
In document modeling, we focus on NIPS dataset, which contains a corpus of 1740 research articles accepted by Neural Information Processing Systems conference from 1987-2016.

\paragraph{Preprocessing}
As a necessary preprocessing procedure for raw documents, we first tokenize the documents and remove numbers, punctuations, and stop words such as those with less than two letters.
Next, all words are reduced to their root form by a standard lemmatizer from NLTK library.
To capture more specific and identifiable semantic objects, bigrams are additionally recognized and added to the token list. 
Finally, we remove the words that are too rare or too common and generate the vocabulary.

\paragraph{Models and experiment description}
After preprocessing and tokenizing the raw documents, we train a series of topic models with different model settings: (i) three typical Dirichlet-Tree priors: original Dirichlet, Beta-Lioville and Generalized Dirichlet; (ii) two vectorized inference methods: Mean-Field Variational Inference and Expectation Propagation; and (iii) a range of topic quantities from 10, 20, to 80.

Therefore, in total, we train and examine $3\times2\times8 = 48$ different topic models. 
We evaluate the obtained models by metrics including (i) convergence of log-likelihood, (ii) predictive perplexity, (iii) topic coherence and (iv) topic diversity.

\subsubsection{Convergence}

The ELBO values are computed and recorded in each training iteration for all models. 
Equally for all model settings, the training process is judged as converged when the ELBO's change rate is less than 1e-4.
Fig. \ref{fig:elbo w.r.t. iterations} shows the ELBO values with respect to each iteration in different models.
We can see that the training process of all models converges successfully after several iterations.
The models with variationl inference generally achieves a little advantage in final ELBO values, while the models with expectation propagation converges in less iterations.

\subsubsection{Predictive perplexity}

Perplexity measures the quality of a topic model by evaluating its predictive capability on hold-out test set.
We split the dataset into 90\% training set and 10\% test set to check the models' predictive perplexity. 
We first train a model on the training set, and then perform inference on the test set.
In practice, the perplexity in our experiment is computed as follows:
\begin{equation*}
    \mathrm{Perplexity}(D_{test}) = \exp\left\{ - \frac{\sum_{m=1}^{M}\sum_{v=1}^{V} n_{mv}\log\sum_{k=1}^{K}\varphi_{vk} \mathbb{E}_{q_{m}(\boldsymbol{\zeta})}\left[ \theta_k \right]}{\sum_{m=1}^{M}\sum_{v=1}^{V} n_{mv}} \right\}
\end{equation*}
Fig. \ref{fig:perplexity} plots the perplexity of different model settings.

\subsubsection{Topic Coherence and Diversity}

Topic coherence and topic diversity are two complementary metrics to measure the quality of learned topics. 

Topic coherence targets at each individual topic and evaluate the semantic similarity within the top words assigned to a topic.
Popular topic coherence measures include UCI Coherence, UMass Coherence, and CV Coherence. 
In our experiments, we use a mixing topic coherence measure provided by Gensim library.
Top ten words are chosen for each topic, and one model's coherence value is computed by averaging the coherence values of each topic. 

On the contrary, topic diversity measures how ``orthogonal", or dissimilar the generated topics are against each other.
A good model should have diverse topics to convey abundant information dimensions.
In practice, our topic diversity is computed as follows:
\begin{equation*}
    \mathrm{Diversity}(\boldsymbol{\varphi}) = \frac{U(N)}{K N}
\end{equation*}
where $K$ is the number of topics, $N$ is number of top words chosen for each topic, and $U$ is number of unique words within these top words.
We choose $N=10$ for our experiments.
Fig. \ref{fig:coherence} and Fig. \ref{fig:diversity} shows the results of topic coherence and topic diversity, repectively.
We observe that Variational Inference generally has better topic coherence, while expectation propagation has better topic diversity.

\begin{figure}[htbp]
    \centering

    \begin{subfigure}[b]{0.48\textwidth}
        \centering
        \includegraphics[width=\linewidth]{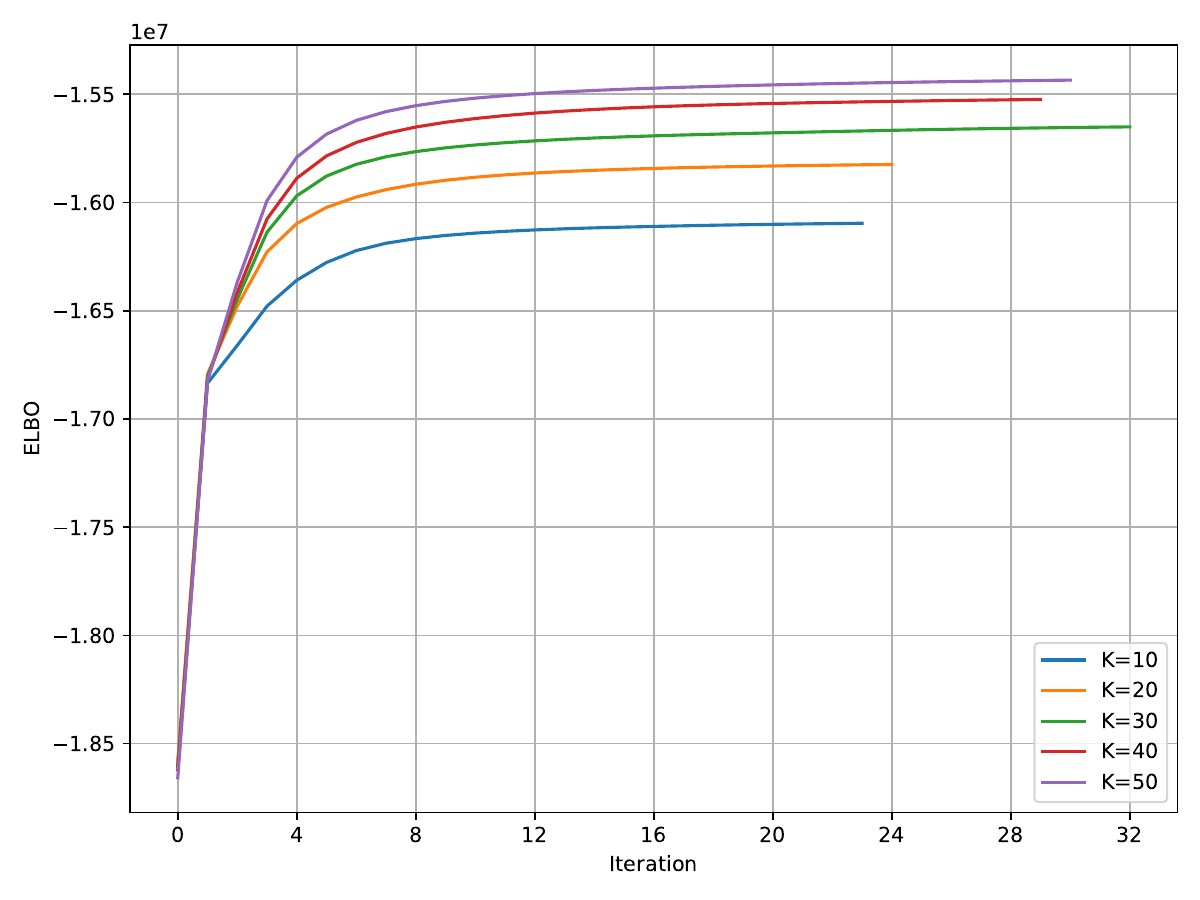}
        \caption{Dir-VI}
    \end{subfigure}
    \hspace{0.02\textwidth}
    \begin{subfigure}[b]{0.48\textwidth}
        \centering
        \includegraphics[width=\linewidth]{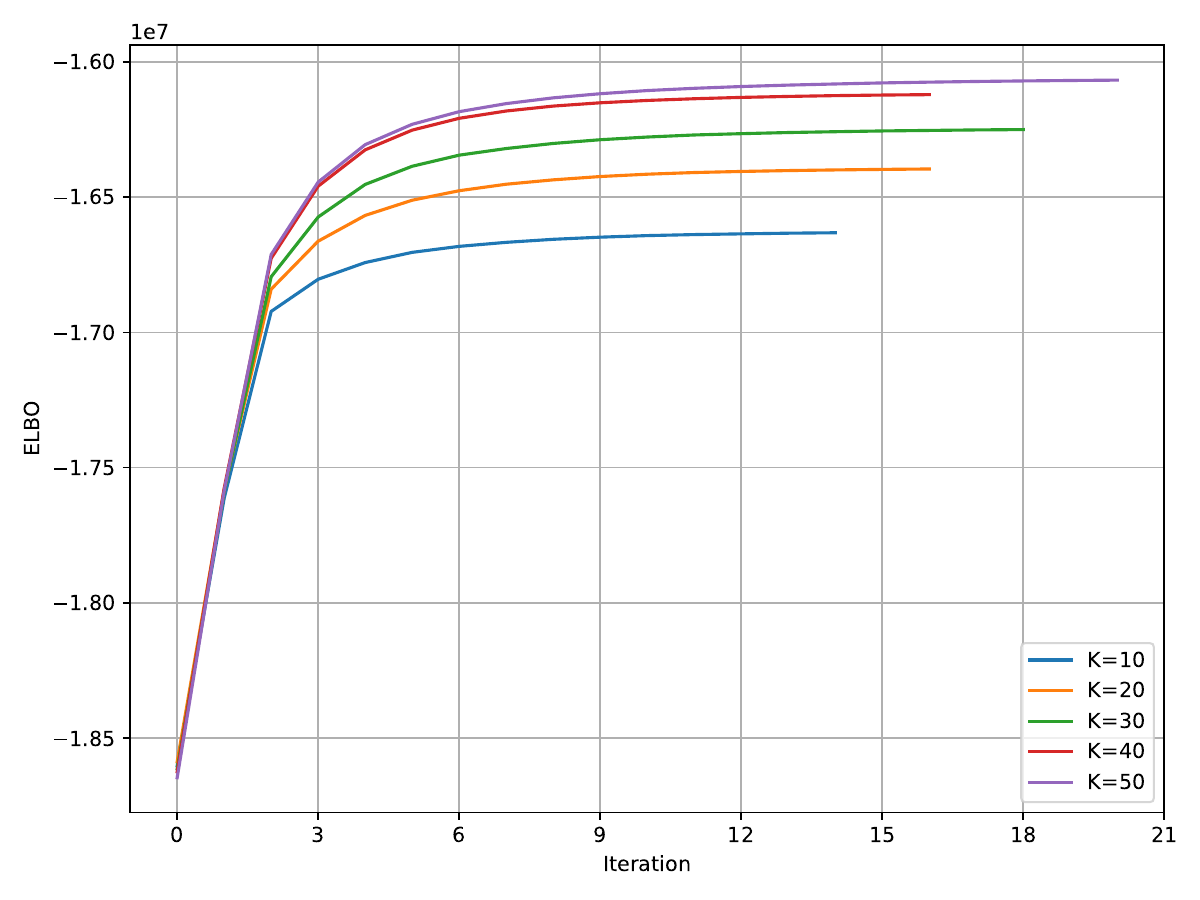}
        \caption{Dir-EP}
    \end{subfigure}

    \vspace{1em}

    \begin{subfigure}[b]{0.48\textwidth}
        \centering
        \includegraphics[width=\linewidth]{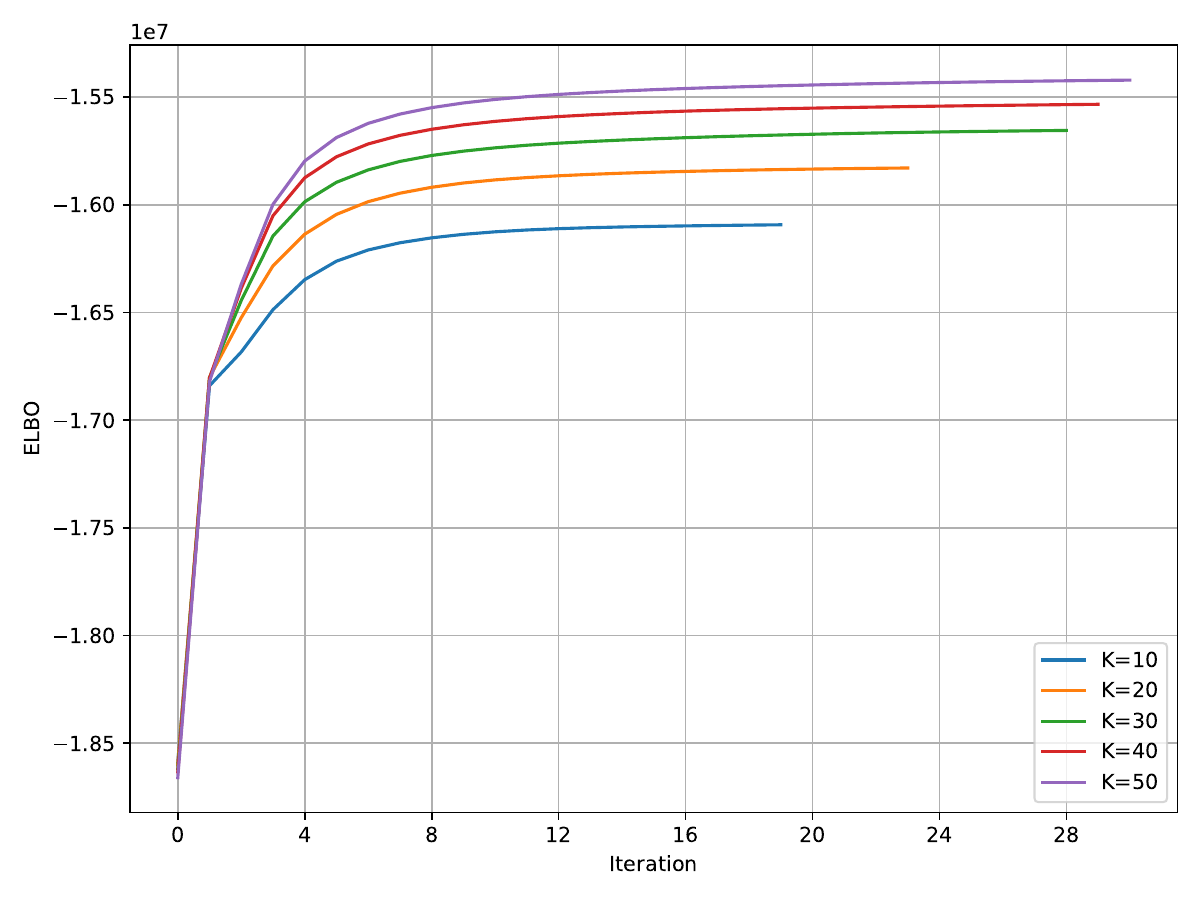}
        \caption{BL-VI}
    \end{subfigure}
    \hspace{0.02\textwidth}
    \begin{subfigure}[b]{0.48\textwidth}
        \centering
        \includegraphics[width=\linewidth]{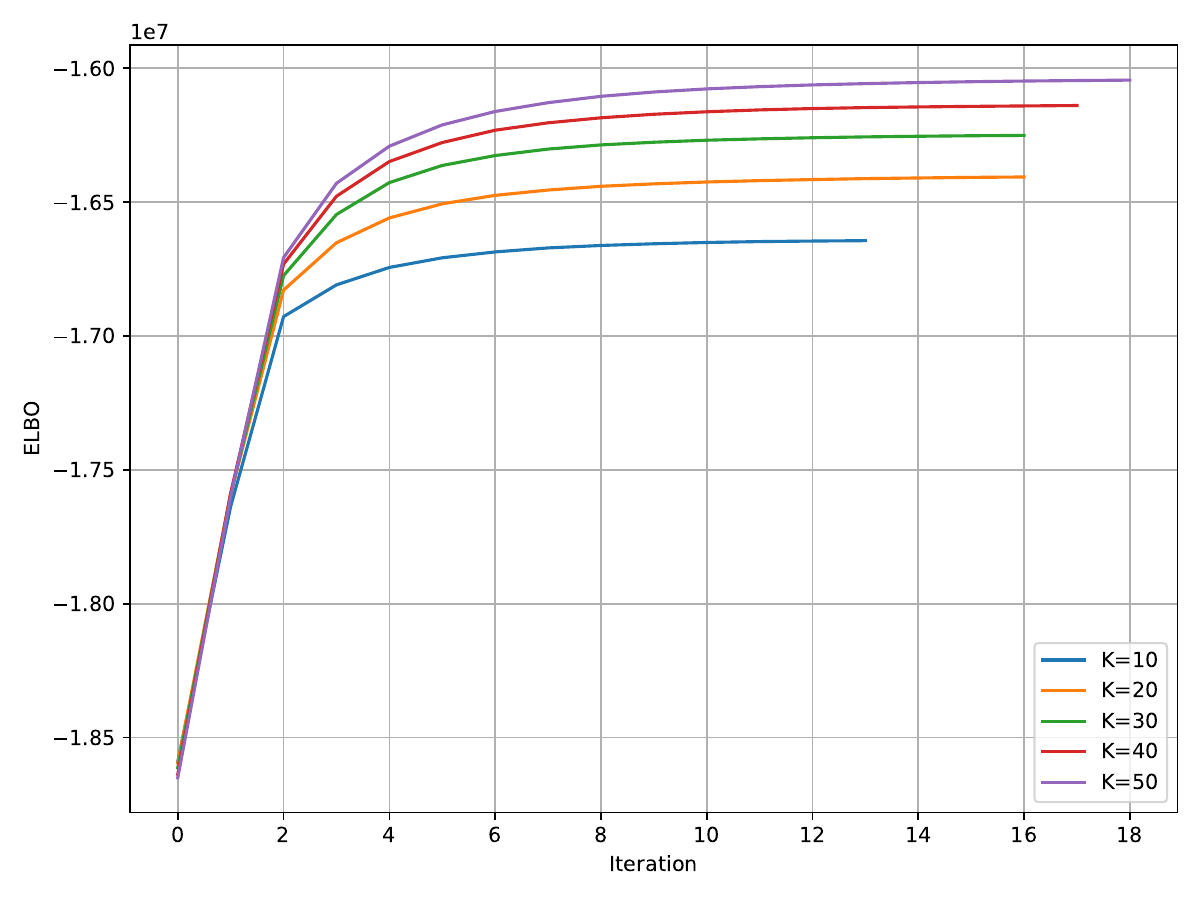}
        \caption{BL-EP}
    \end{subfigure}

    \vspace{1em}

    \begin{subfigure}[b]{0.48\textwidth}
        \centering
        \includegraphics[width=\linewidth]{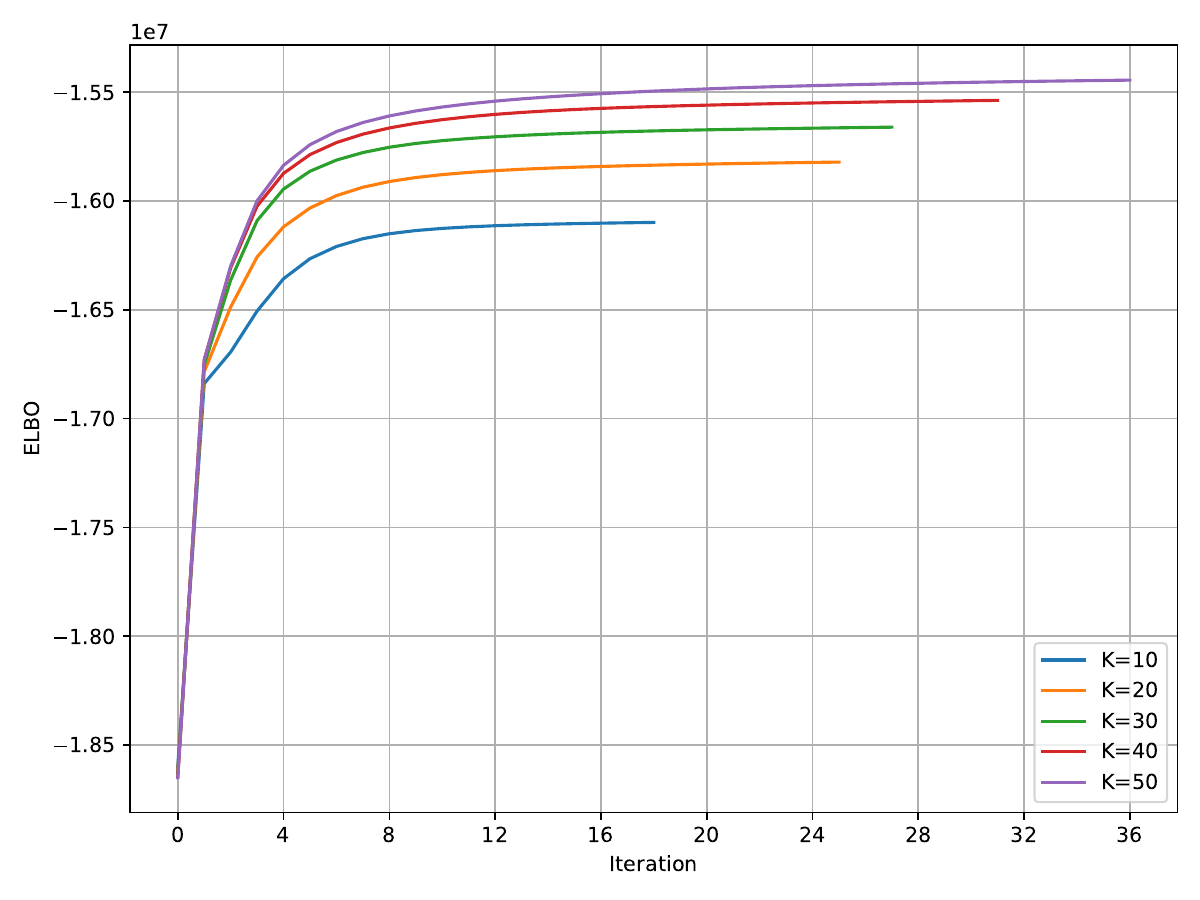}
        \caption{GDir-VI}
    \end{subfigure}
    \hspace{0.02\textwidth}
    \begin{subfigure}[b]{0.48\textwidth}
        \centering
        \includegraphics[width=\linewidth]{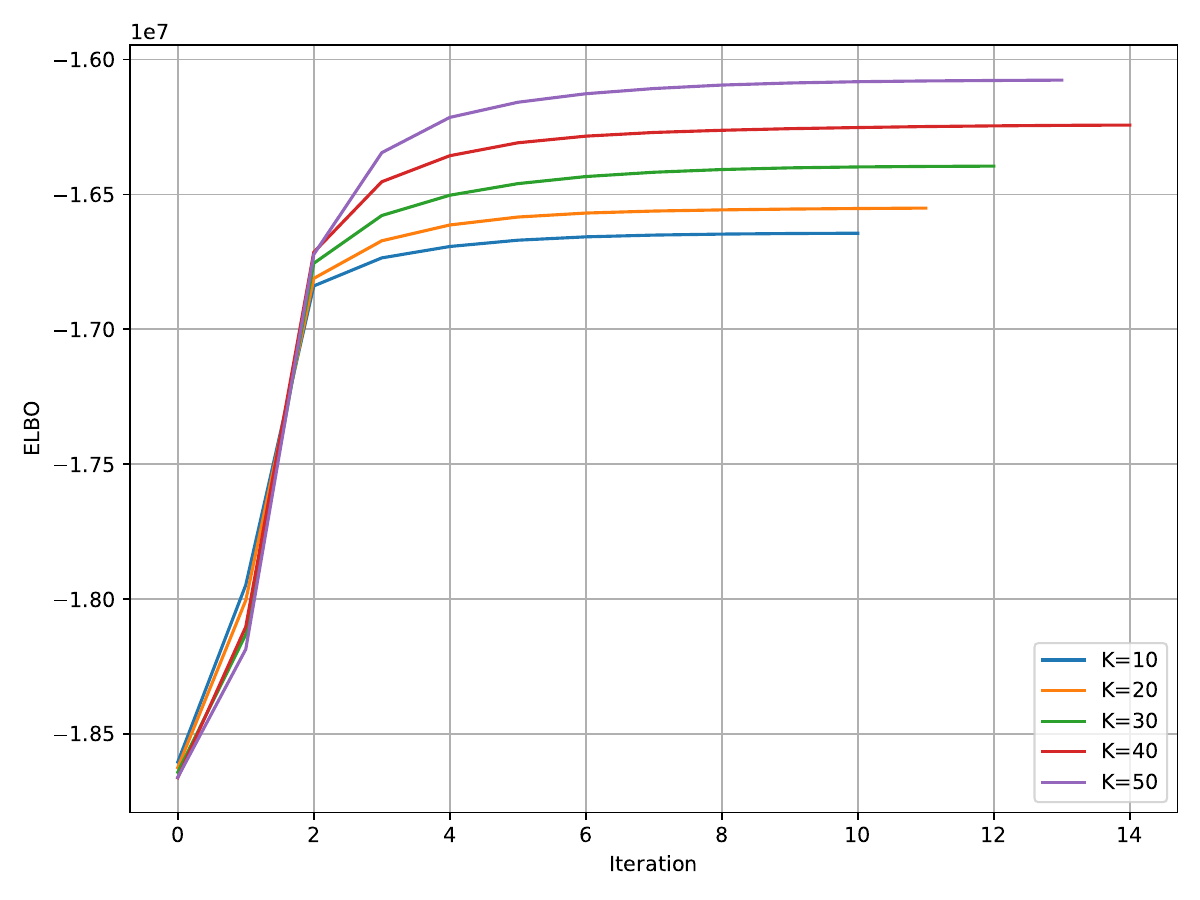}
        \caption{GDir-EP}
    \end{subfigure}

    \caption{ELBO convergence with respect to iterations; iterations stop when ELBO change less than 0.0001}
    \label{fig:elbo w.r.t. iterations}
\end{figure}

\begin{figure}[p] 
\centering

\begin{minipage}{0.48\textwidth}
    \centering
    \includegraphics[width=\linewidth]{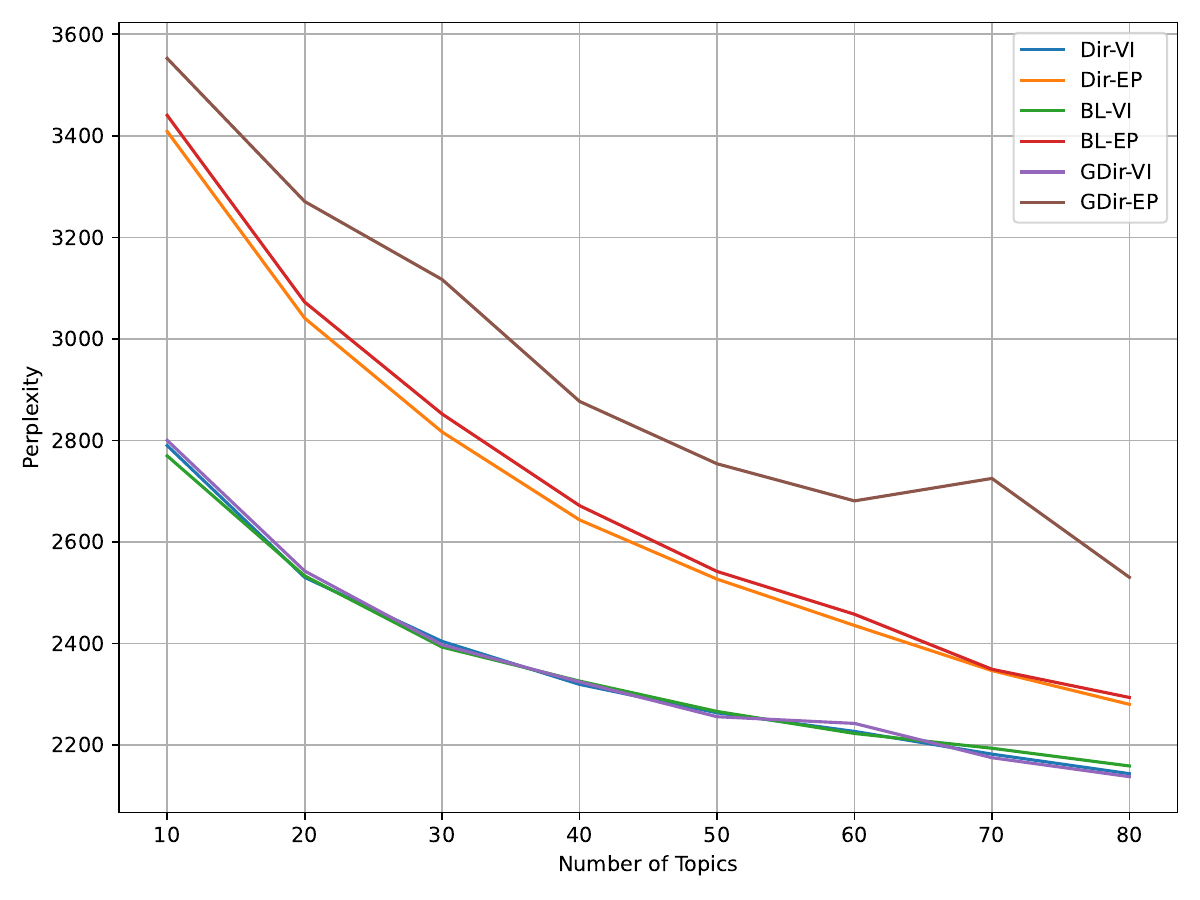}
    \caption{Perplexity}
    \label{fig:perplexity}
\end{minipage}
\hfill
\begin{minipage}{0.48\textwidth}
    \centering
    \includegraphics[width=\linewidth]{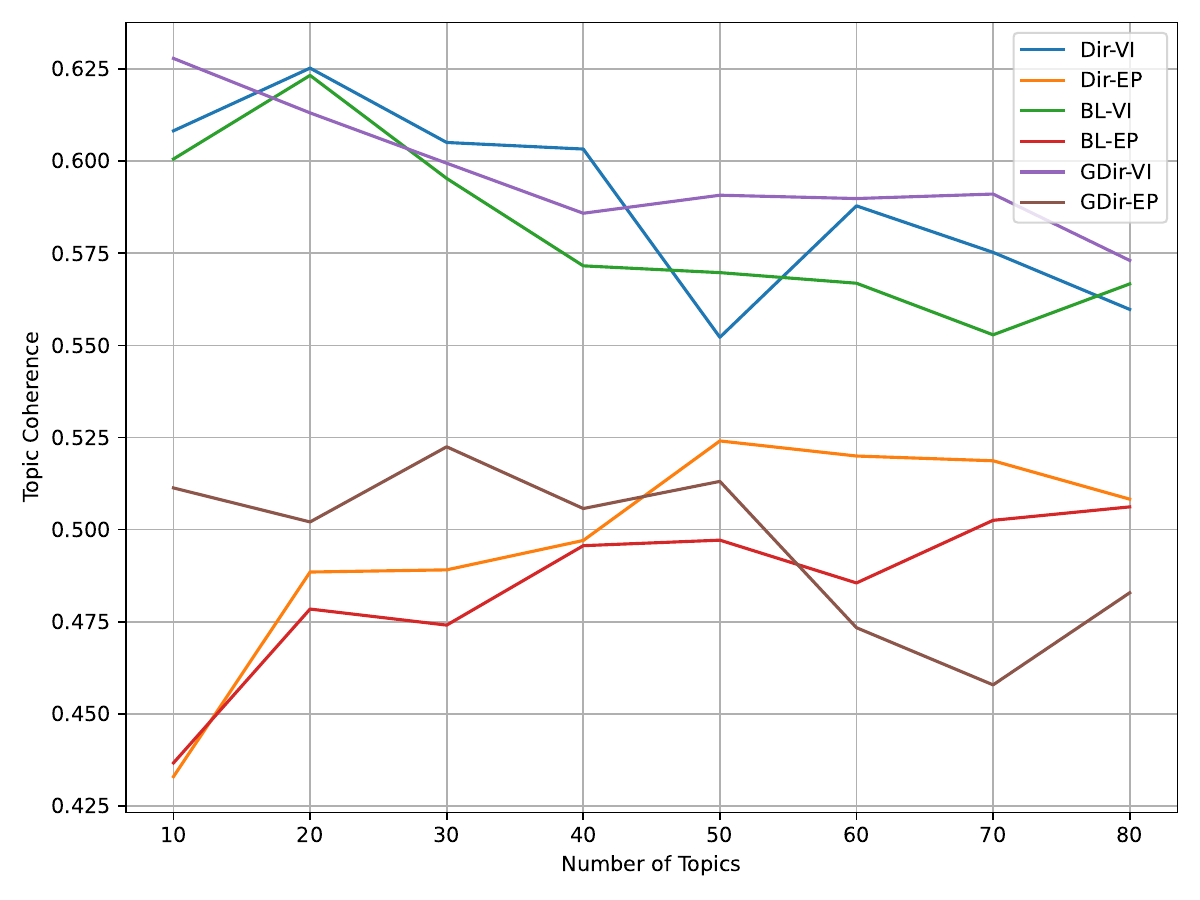}
    \caption{Topic Coherence}
    \label{fig:coherence}
\end{minipage}

\vspace{1.5em}

\begin{minipage}{0.48\textwidth}
    \centering
    \includegraphics[width=\linewidth]{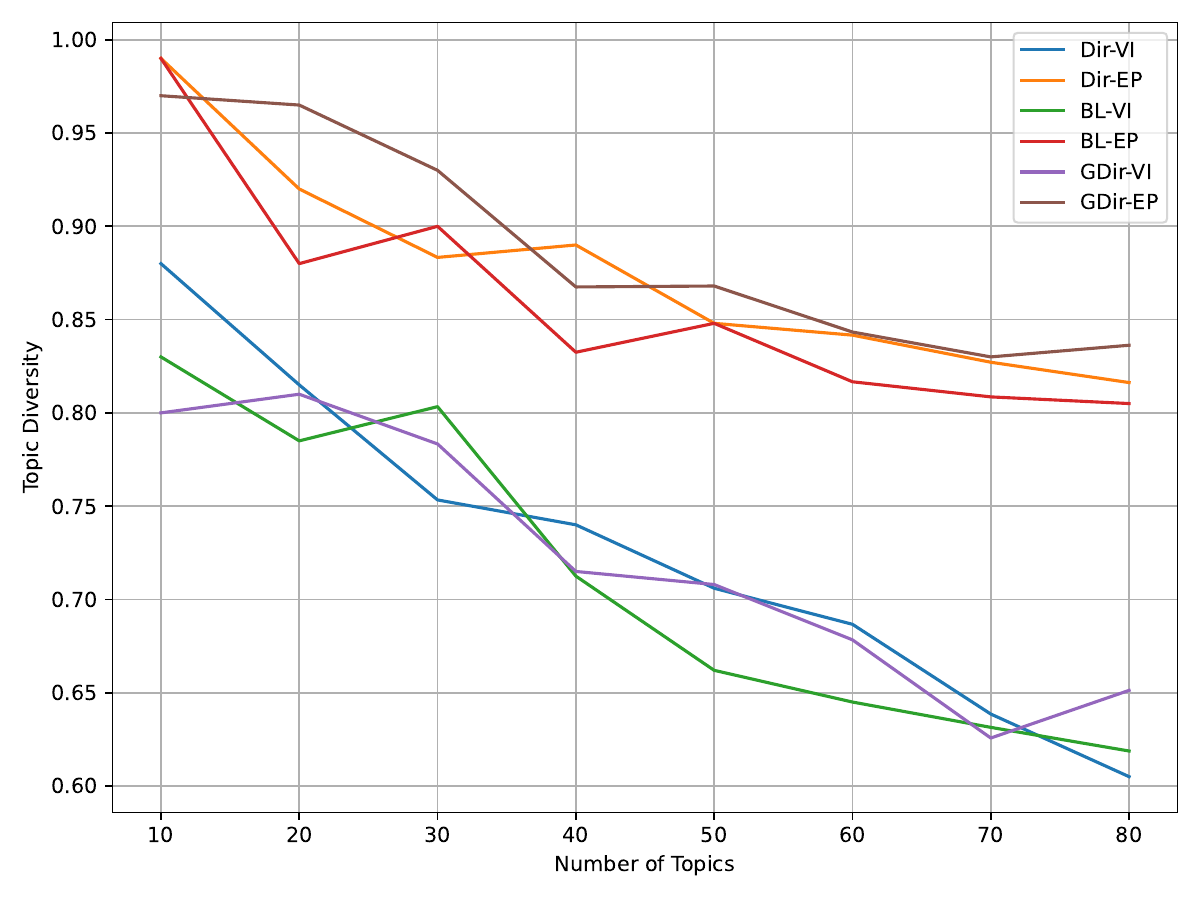}
    \caption{Topic Diversity}
    \label{fig:diversity}
\end{minipage}
\hfill
\begin{minipage}{0.48\textwidth}
    \centering
    \includegraphics[width=\linewidth]{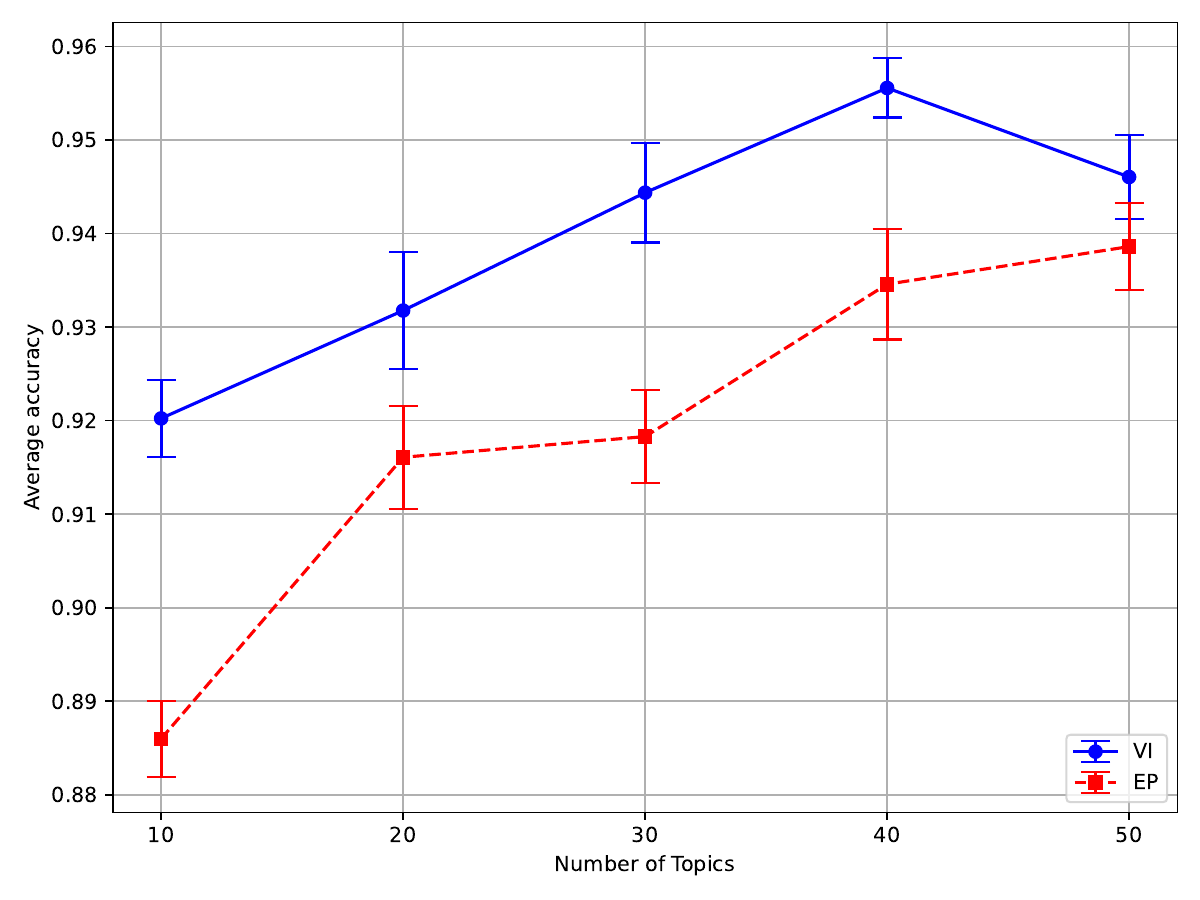}
    \caption{Average classification accuracy of Dirichlet prior}
    \label{fig:Dir_acc}
\end{minipage}

\vspace{1.5em}

\begin{minipage}{0.48\textwidth}
    \centering
    \includegraphics[width=\linewidth]{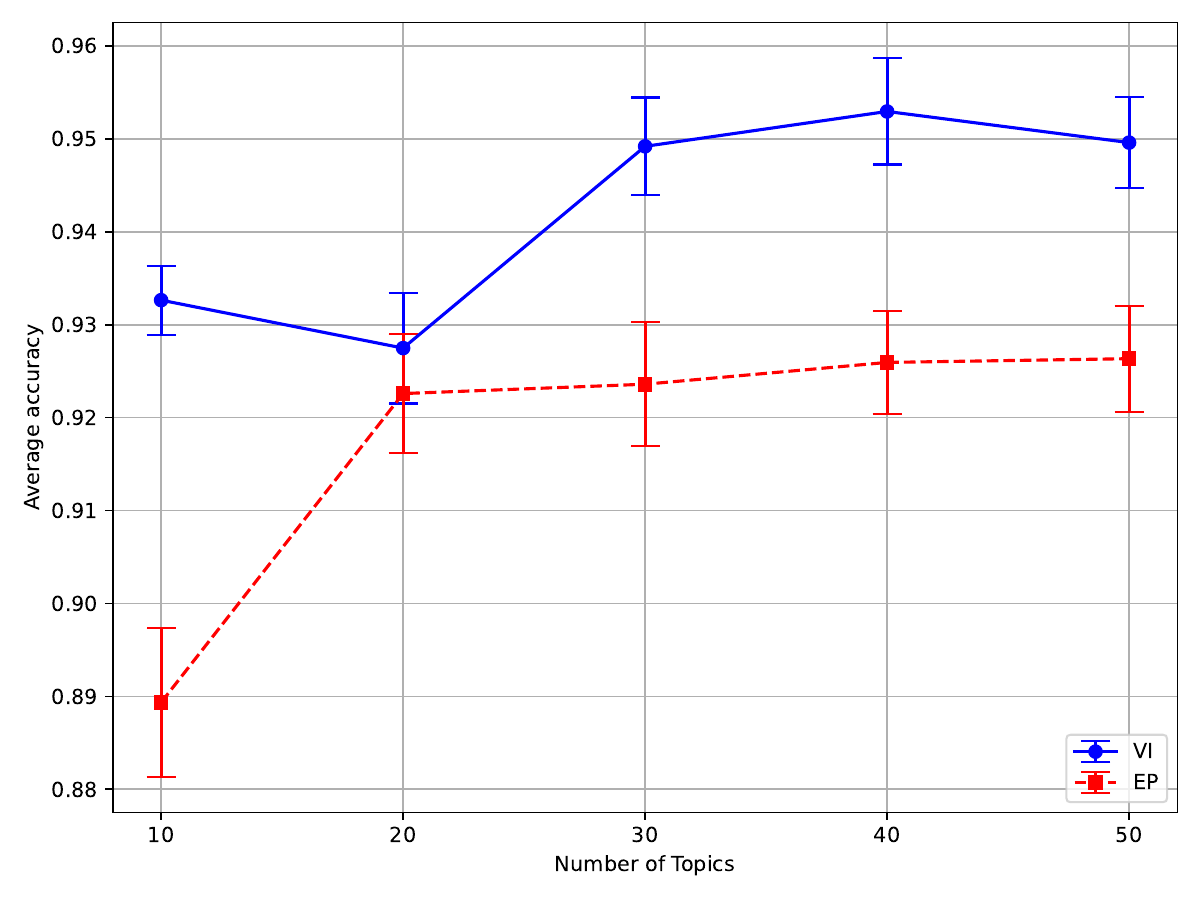}
    \caption{Average classification accuracy of Beta-Liouville prior}
    \label{fig:BL_acc}
\end{minipage}
\hfill
\begin{minipage}{0.48\textwidth}
    \centering
    \includegraphics[width=\linewidth]{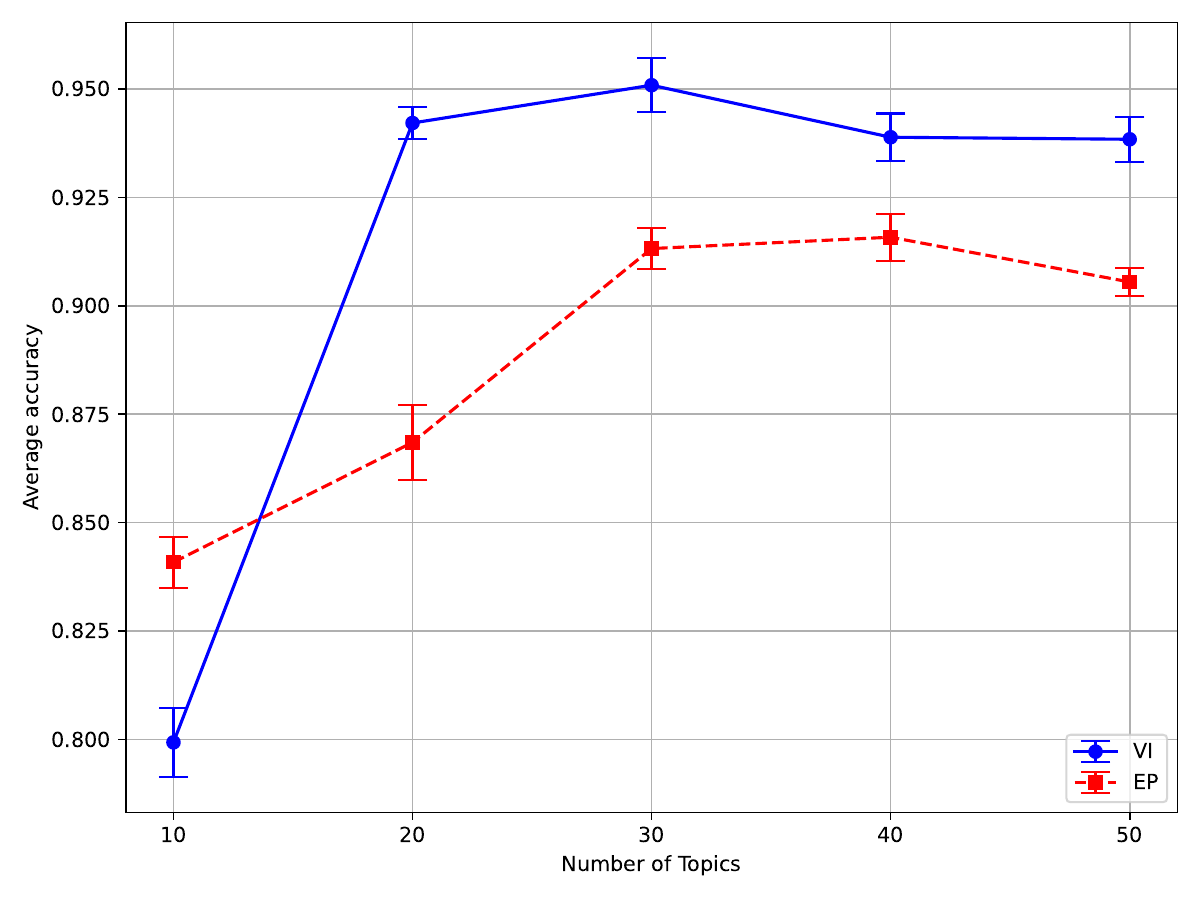}
    \caption{Average classification accuracy of Generalized Dirichlet prior}
    \label{fig:GDir_acc}
\end{minipage}

\end{figure}




\subsection{Document Classification}

In the tasks of document classification, a collection of documents are categorized into binary or multiple classes, where each document is represented by a fixed set of features (Blei 2003).

A natural thinking is to use the counts of all the words that appear in a document as the representative features, yet leading to an extremely large feature space.
Topic models play the role of dimension reduction in document classification, where the large vocabulary is reduced to K-dimensional simplex.

\begin{table}[ht]
\centering
\caption{Number of documents of chosen categories}
\label{table:chosen_documents_reuters}
\vspace{0.3em}
\begin{tabular}{lr}
\toprule
\textbf{Category} & \textbf{\# Documents} \\
\midrule
earn        & 3900 \\
acq         & 2289 \\
crude       &  370 \\
trade       &  324 \\
money-fx    &  306 \\
interest    &  271 \\
\bottomrule
\end{tabular}
\end{table}

In our experiment, we use Reuters-21578 corpus set which is a collection of labeled newswire articles.
After necessary preprocessing, we choose the six categories that contain the most documents.
Finally, we obtain a vocabulary of 4562 unique words, and 7460 documents in total.
The numbers of documents in each category are shown in Table \ref{table:chosen_documents_reuters}. 

We train multiple models with different settings for the whole corpus, and compute the expectation of each document's posterior distribution as its representative feature.
Subsequently, we train a logistic regression classifier based on the learned topic proportions and the original labels.
For each topic model setting, we randomly split the corpus into 80\% training set and 20\% test set for 10 times.
The average classification accuracy, as well as its standard deviation, is computed and employed as its final accuracy result.

Figure \ref{fig:Dir_acc}, Figure \ref{fig:BL_acc} and Figure \ref{fig:GDir_acc} are, respectively, the results of classification accuracy of Dirichlet prior, Beta Liouville prior, and Generalized Dirichlet prior.
Table \ref{table:classification_acc} shows the specific accuracy value of each model setting.
For each model setting, we choose the topic number that generates the highest average accuracy and demonstrate its confusion matrix in Figure \ref{fig:cm}.

\begin{table}[ht]
  \centering
  \begin{tabular}{lccccc}
    \toprule
    \multirow{2}{*}{\textbf{Models}} & \multicolumn{5}{c}{\textbf{Accuracy}} \\
    \cmidrule(lr){2-6}
    & K=10 & K=20 & K=30 & K=40 & K=50 \\
    \midrule
    Dir-VI        & 92.0\%   & 93.2\%   & 94.4\%   & \textbf{95.6\%}   & 94.6\% \\
    Dir-EP        & 88.6\%   & 91.6\%   & 91.8\%   & 93.5\%   & \textbf{93.9\%} \\
    BL-VI         & 93.3\%   & 92.7\%   & 94.9\%   & \textbf{95.3\%}   & 95.0\% \\
    BL-EP         & 88.9\%   & 92.3\%   & 92.4\%   & 92.5\%   & \textbf{92.6\%} \\
    GDir-VI       & 79.9\%   & 94.2\%   & \textbf{95.1\%}   & 93.9\%   & 93.8\% \\
    GDir-EP       & 84.1\%   & 86.8\%   & 91.3\%   & \textbf{91.6\%}   & 90.5\% \\
    \bottomrule
  \end{tabular}
  \caption{Average accuracy of classification}
  \label{table:classification_acc}
\end{table}




\begin{figure}[htbp]
    \centering

    \begin{subfigure}[b]{0.48\textwidth}
        \centering
        \includegraphics[width=\linewidth]{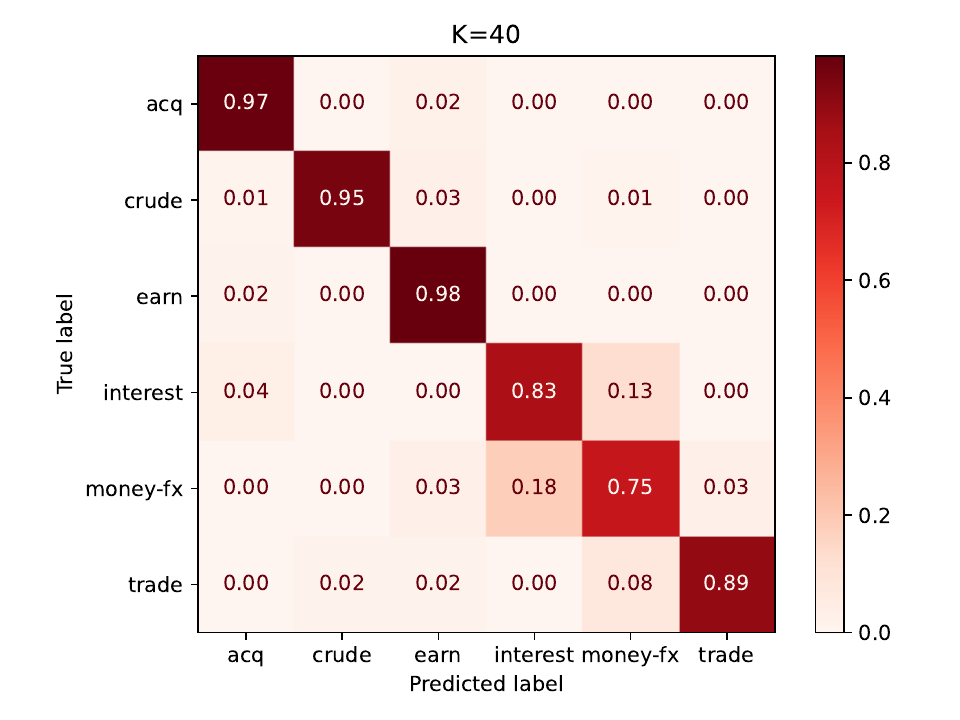}
        \caption{Dir-VI}
    \end{subfigure}
    \hspace{0.02\textwidth}
    \begin{subfigure}[b]{0.48\textwidth}
        \centering
        \includegraphics[width=\linewidth]{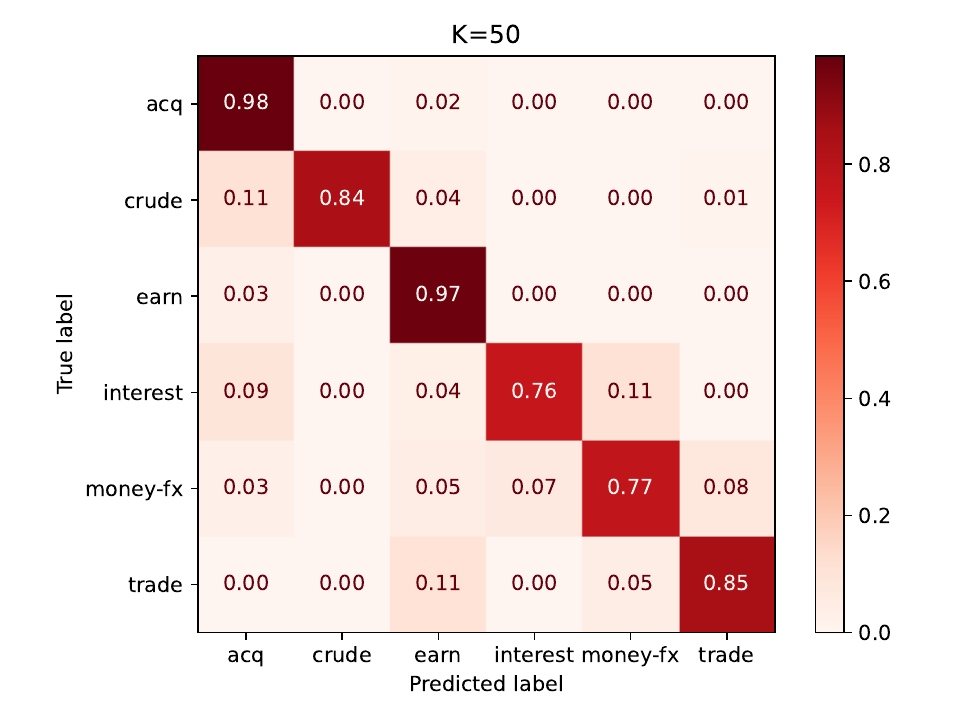}
        \caption{Dir-EP}
    \end{subfigure}

    \vspace{1em}

    \begin{subfigure}[b]{0.48\textwidth}
        \centering
        \includegraphics[width=\linewidth]{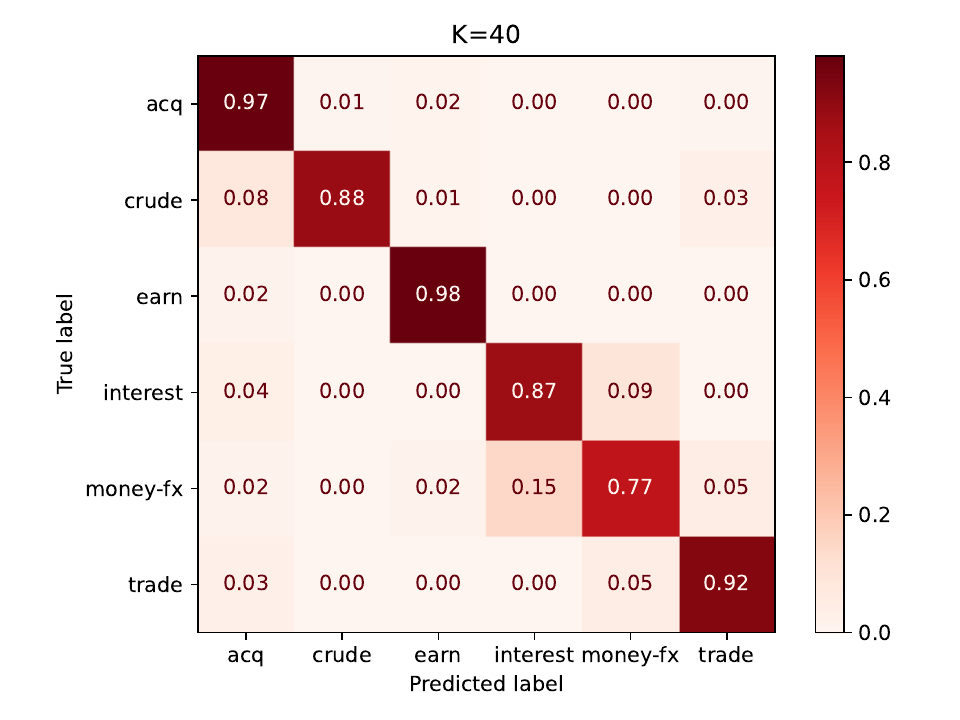}
        \caption{BL-VI}
    \end{subfigure}
    \hspace{0.02\textwidth}
    \begin{subfigure}[b]{0.48\textwidth}
        \centering
        \includegraphics[width=\linewidth]{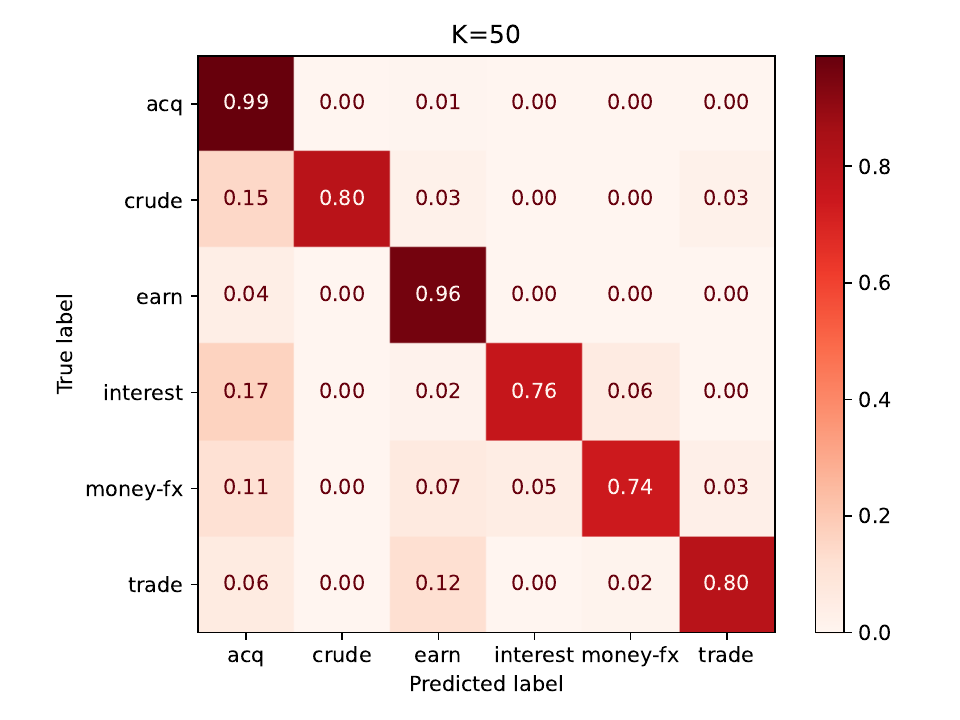}
        \caption{BL-EP}
    \end{subfigure}

    \vspace{1em}

    \begin{subfigure}[b]{0.48\textwidth}
        \centering
        \includegraphics[width=\linewidth]{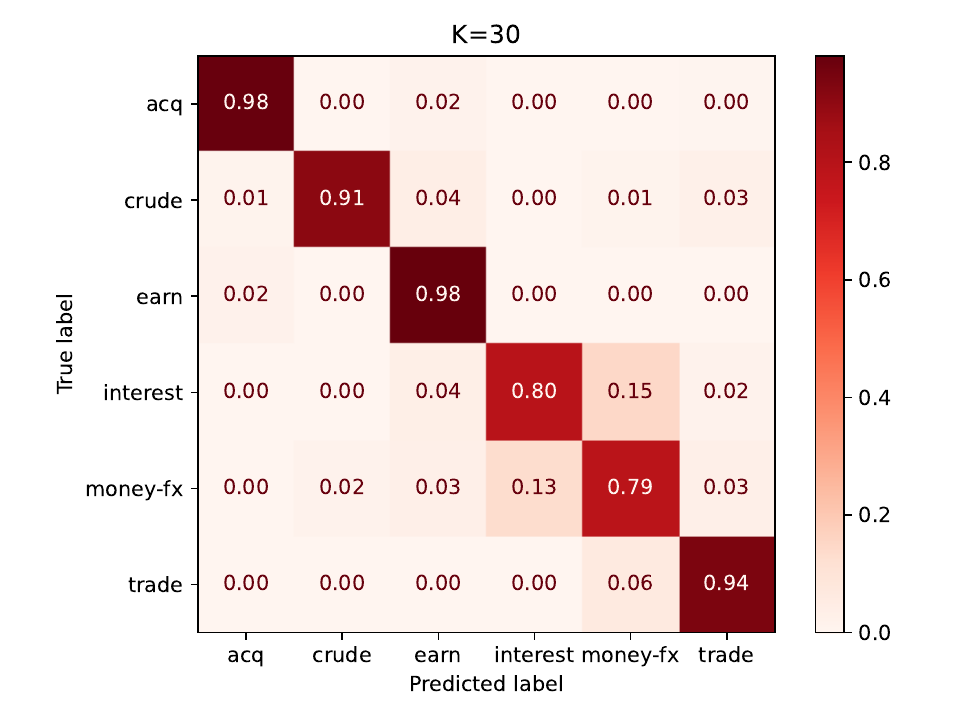}
        \caption{GDir-VI}
    \end{subfigure}
    \hspace{0.02\textwidth}
    \begin{subfigure}[b]{0.48\textwidth}
        \centering
        \includegraphics[width=\linewidth]{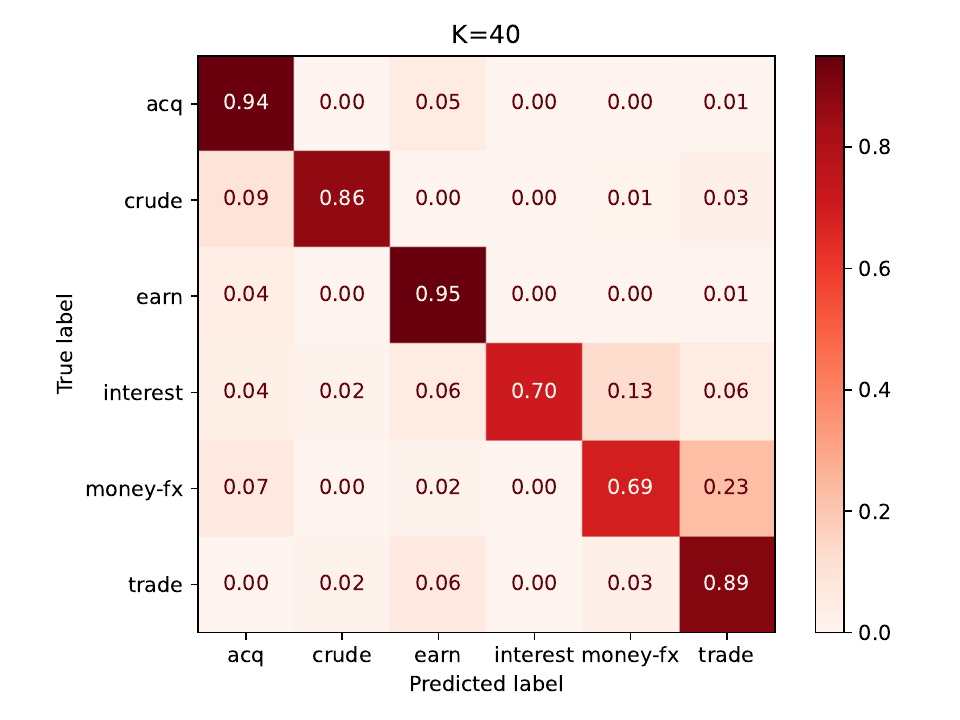}
        \caption{GDir-EP}
    \end{subfigure}

    \caption{Best confusion matrix for each model settings; the titles indicating the topic number}
    \label{fig:cm}
\end{figure}

\subsection{Image Classification}

\begin{figure}[htbp]
    \centering

    \begin{subfigure}[b]{0.41\textwidth}
        \centering
        \includegraphics[width=\linewidth]{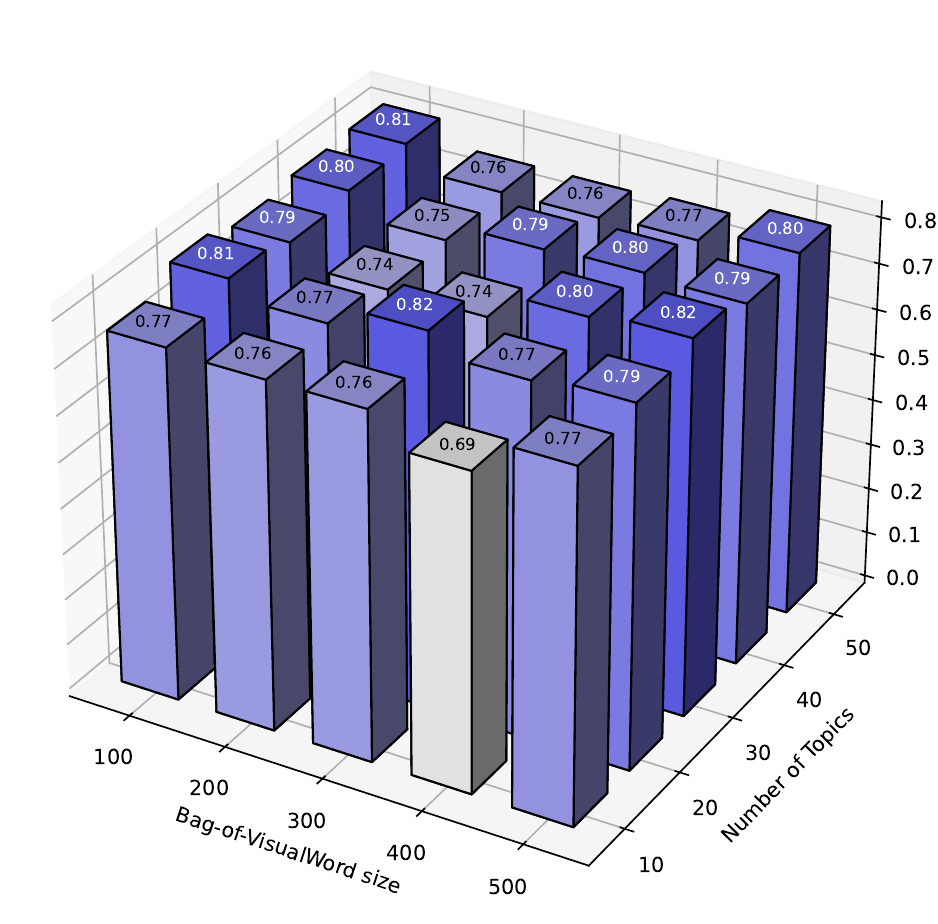}
        \caption{Dir-VI}
    \end{subfigure}
    \hspace{0.02\textwidth}
    \begin{subfigure}[b]{0.41\textwidth}
        \centering
        \includegraphics[width=\linewidth]{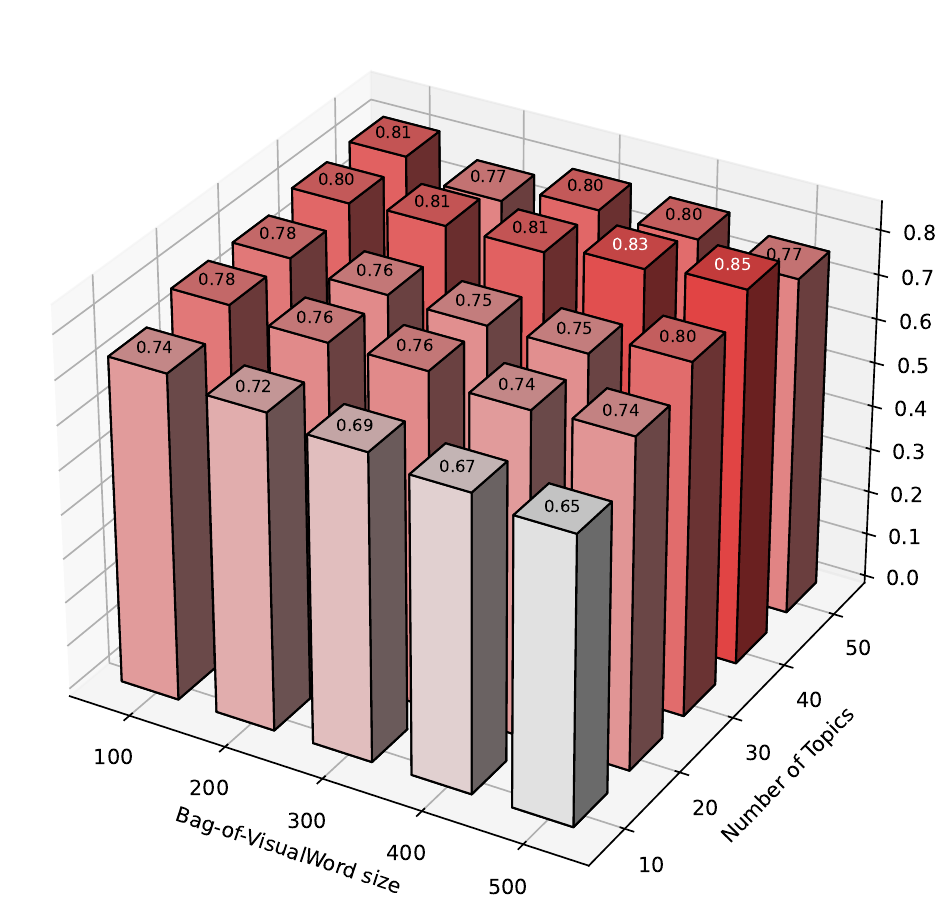}
        \caption{Dir-EP}
    \end{subfigure}


    \begin{subfigure}[b]{0.41\textwidth}
        \centering
        \includegraphics[width=\linewidth]{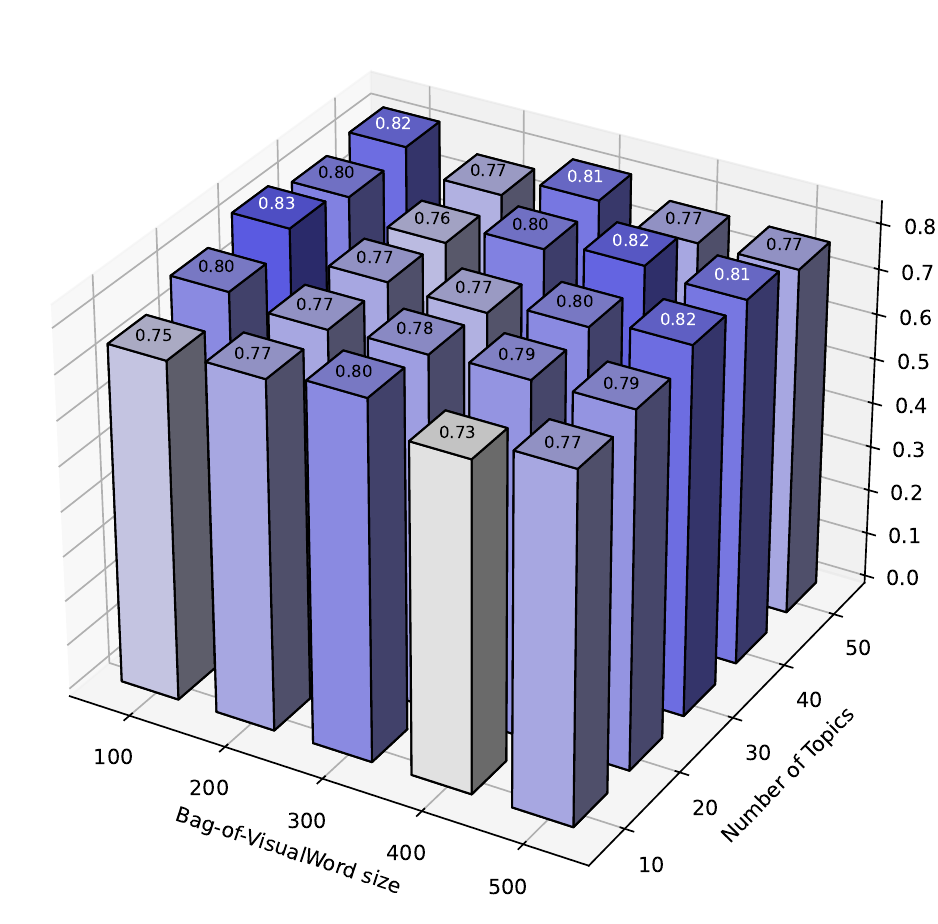}
        \caption{BL-VI}
    \end{subfigure}
    \hspace{0.02\textwidth}
    \begin{subfigure}[b]{0.41\textwidth}
        \centering
        \includegraphics[width=\linewidth]{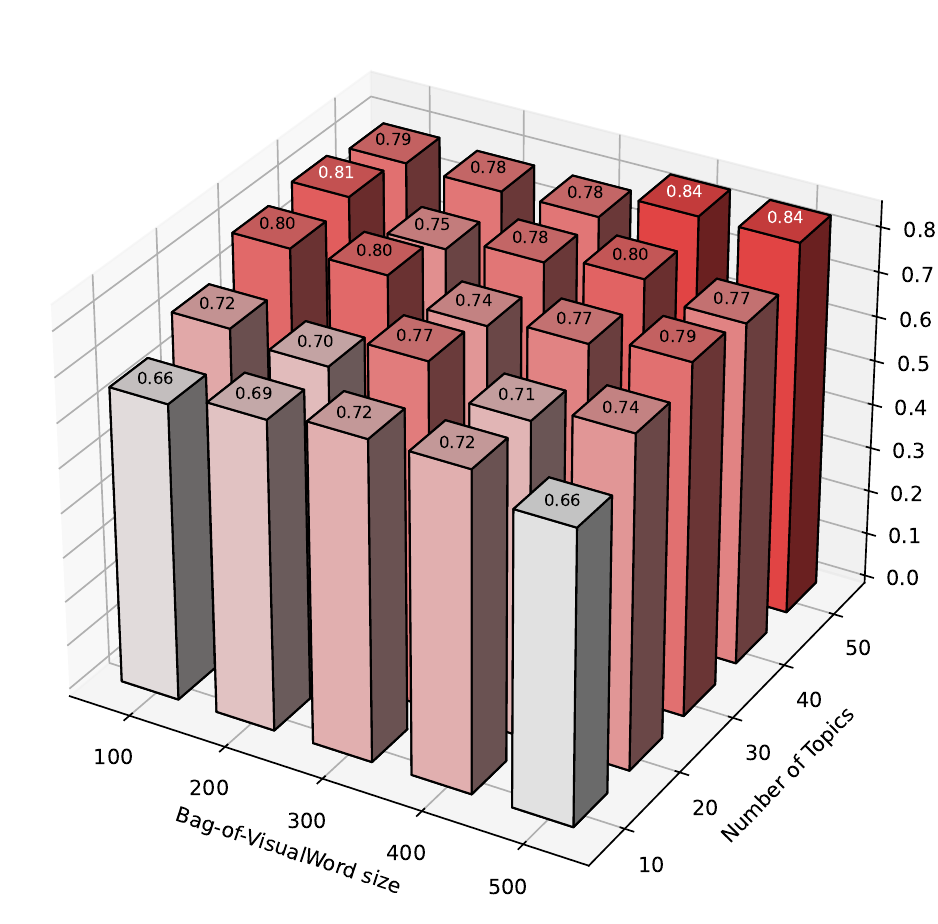}
        \caption{BL-EP}
    \end{subfigure}


    \begin{subfigure}[b]{0.41\textwidth}
        \centering
        \includegraphics[width=\linewidth]{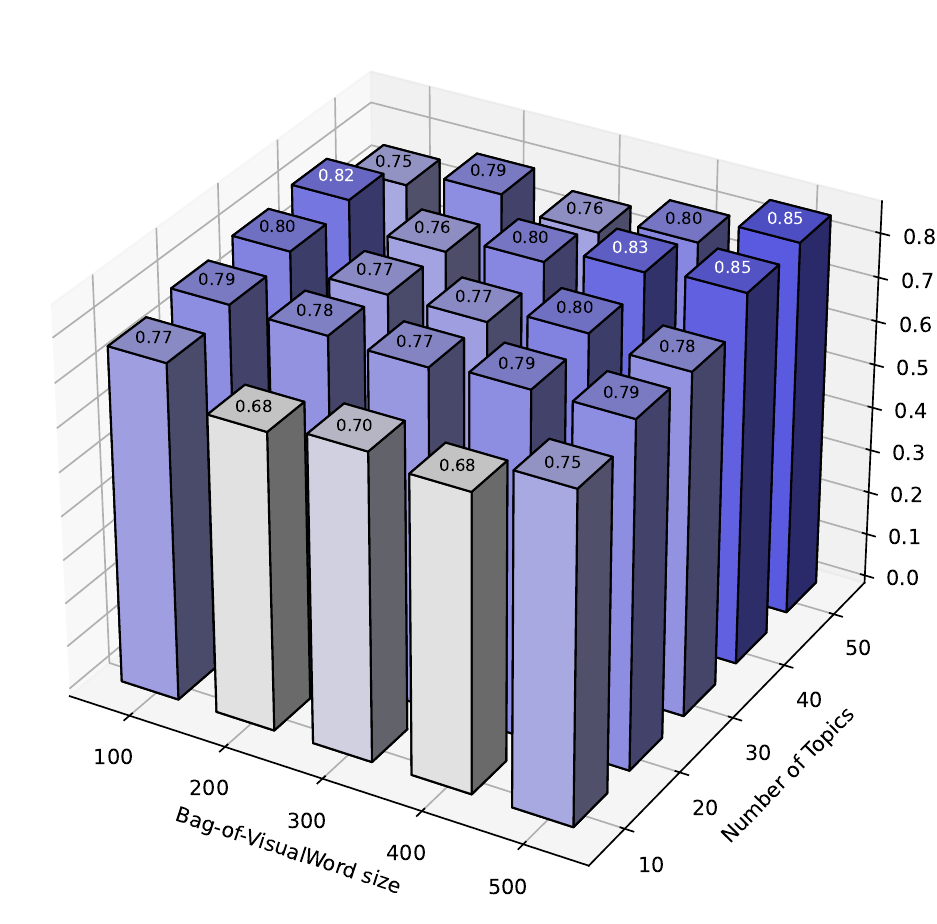}
        \caption{GDir-VI}
    \end{subfigure}
    \hspace{0.02\textwidth}
    \begin{subfigure}[b]{0.41\textwidth}
        \centering
        \includegraphics[width=\linewidth]{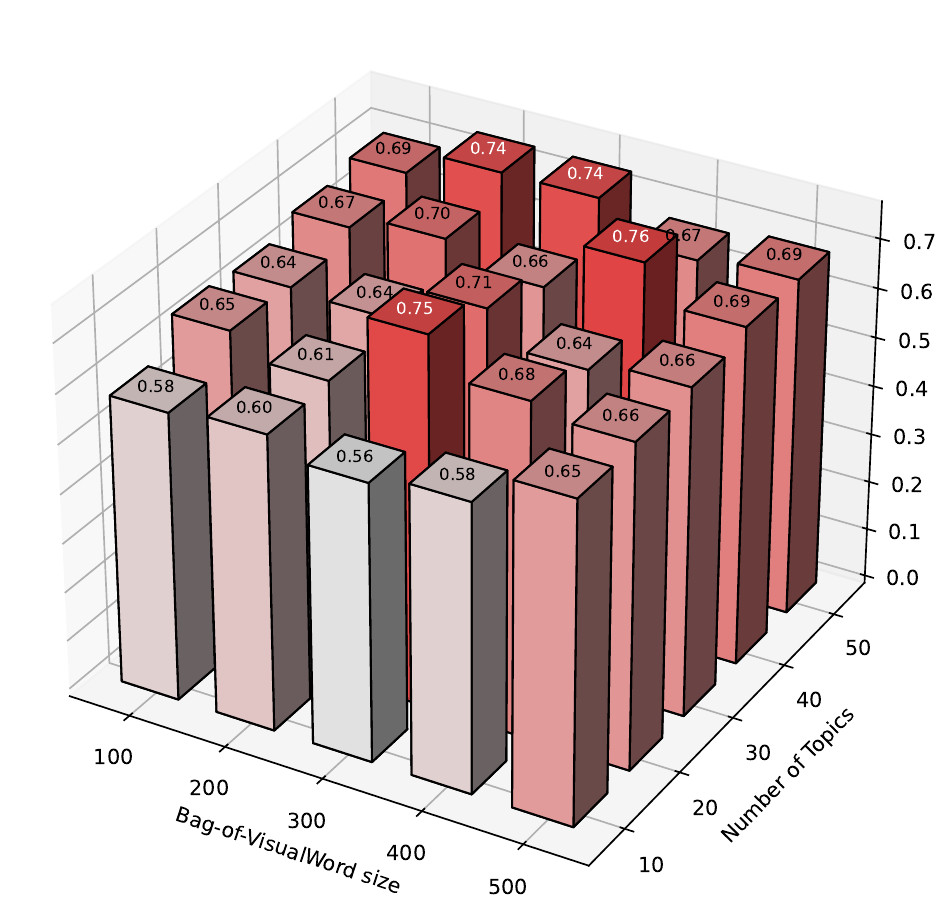}
        \caption{GDir-EP}
    \end{subfigure}

    \caption{Image classification accuracies}
    \label{fig:3dbar}
\end{figure}

\begin{figure}[htbp]
    \centering

    \begin{subfigure}[b]{0.48\textwidth}
        \centering
        \includegraphics[width=\linewidth]{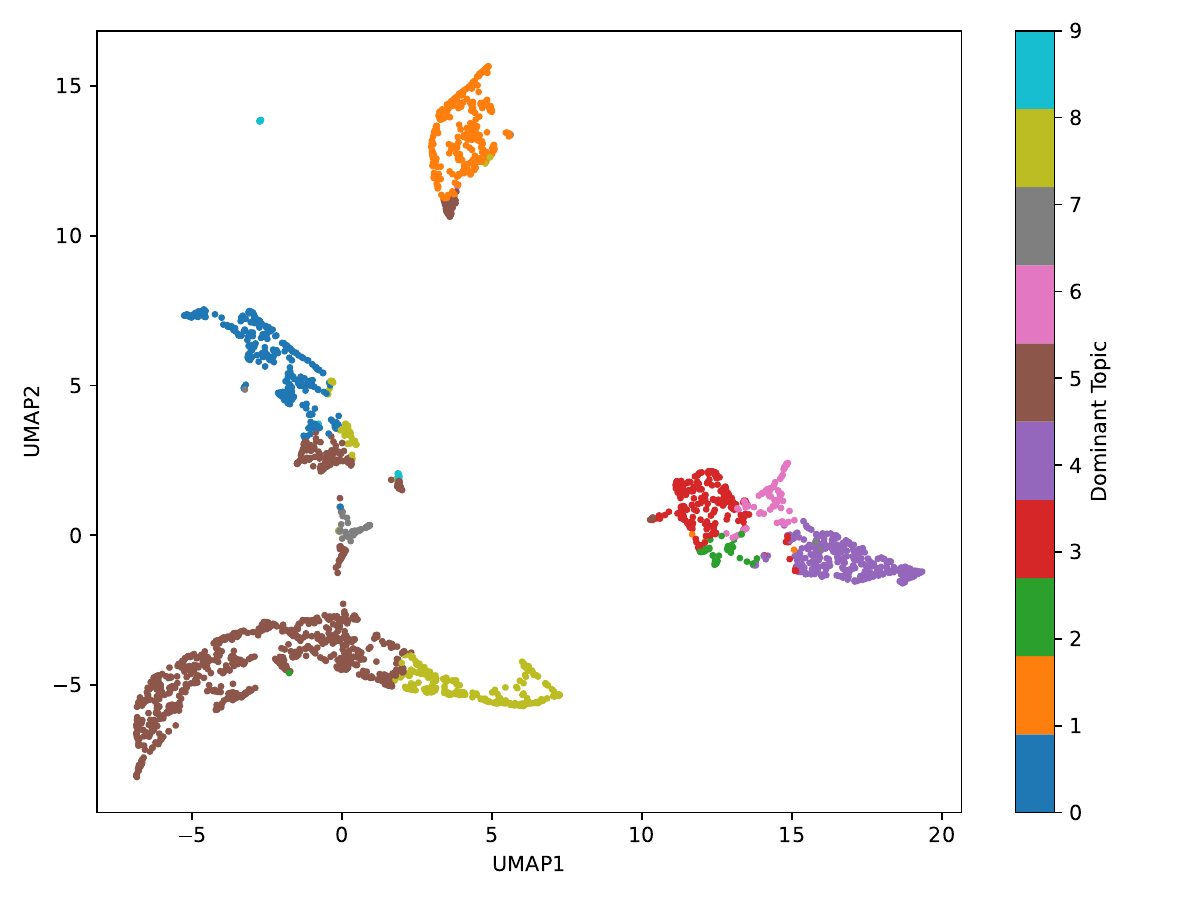}
        \caption{Dir-VI}
    \end{subfigure}
    \hspace{0.02\textwidth}
    \begin{subfigure}[b]{0.48\textwidth}
        \centering
        \includegraphics[width=\linewidth]{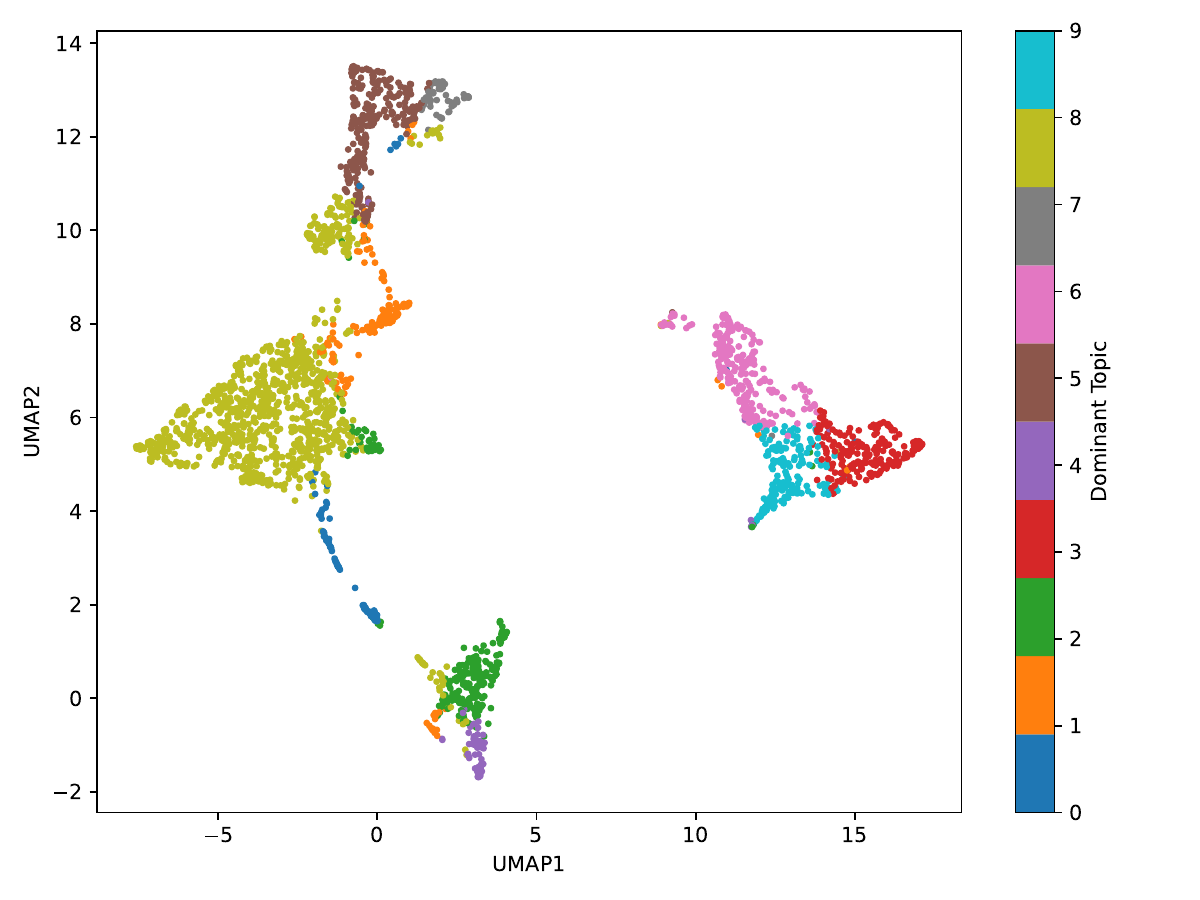}
        \caption{Dir-EP}
    \end{subfigure}

    \vspace{1em}

    \begin{subfigure}[b]{0.48\textwidth}
        \centering
        \includegraphics[width=\linewidth]{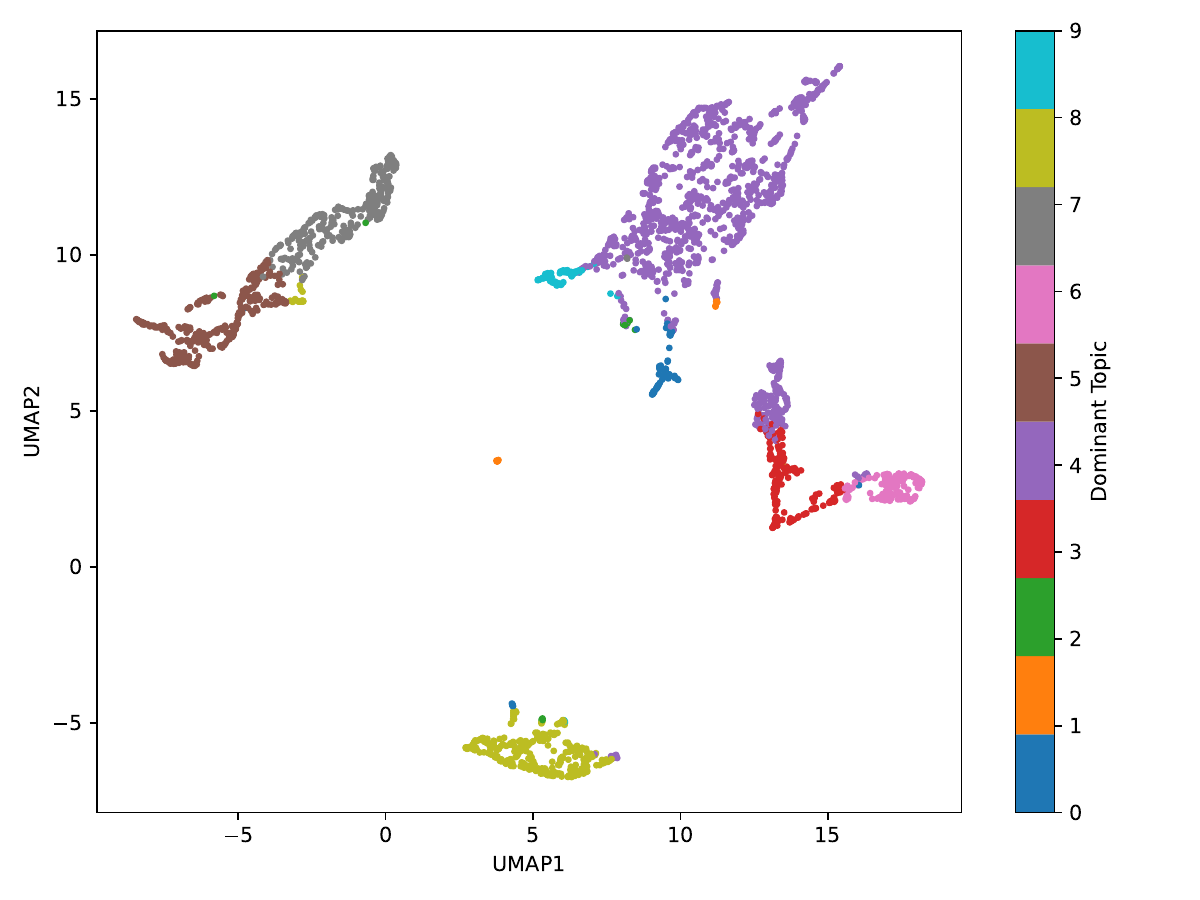}
        \caption{BL-VI}
    \end{subfigure}
    \hspace{0.02\textwidth}
    \begin{subfigure}[b]{0.48\textwidth}
        \centering
        \includegraphics[width=\linewidth]{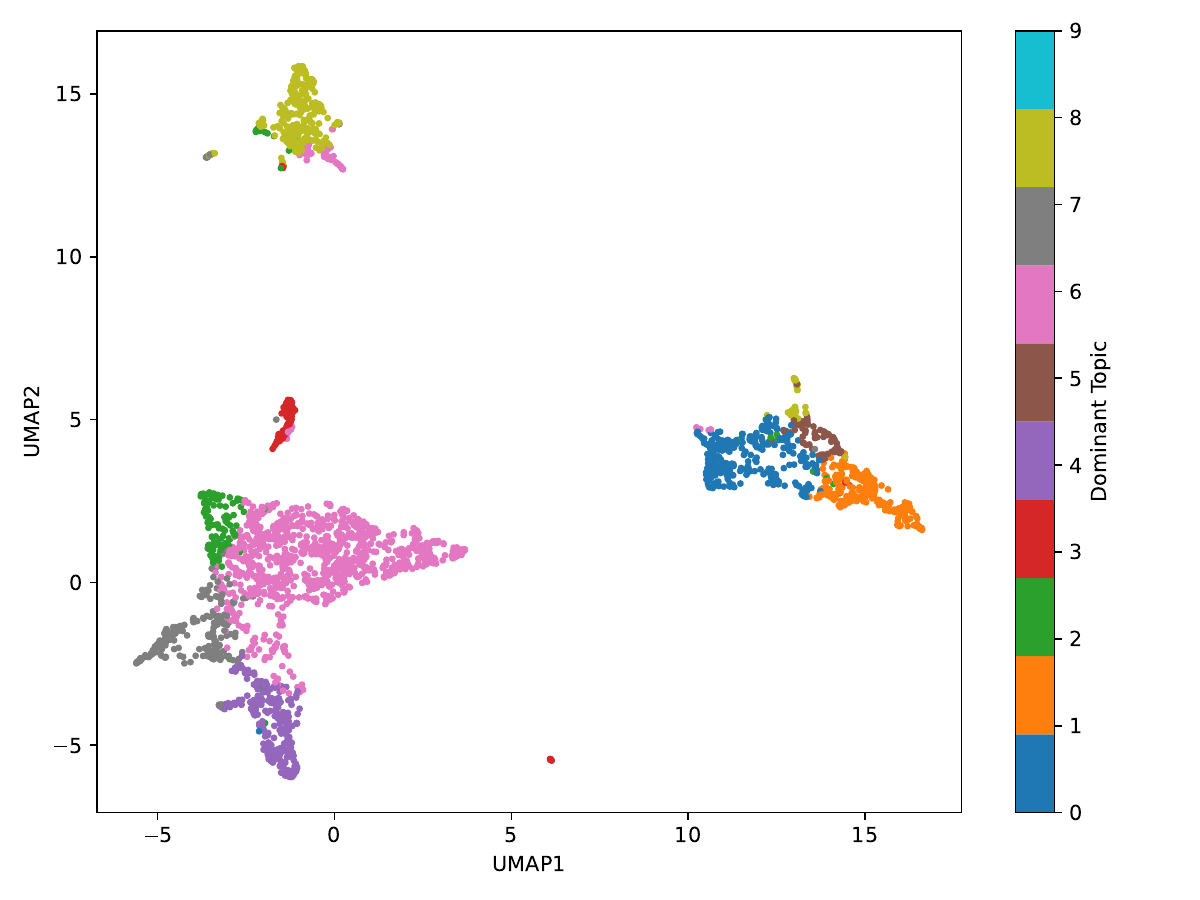}
        \caption{BL-EP}
    \end{subfigure}

    \vspace{1em}

    \begin{subfigure}[b]{0.48\textwidth}
        \centering
        \includegraphics[width=\linewidth]{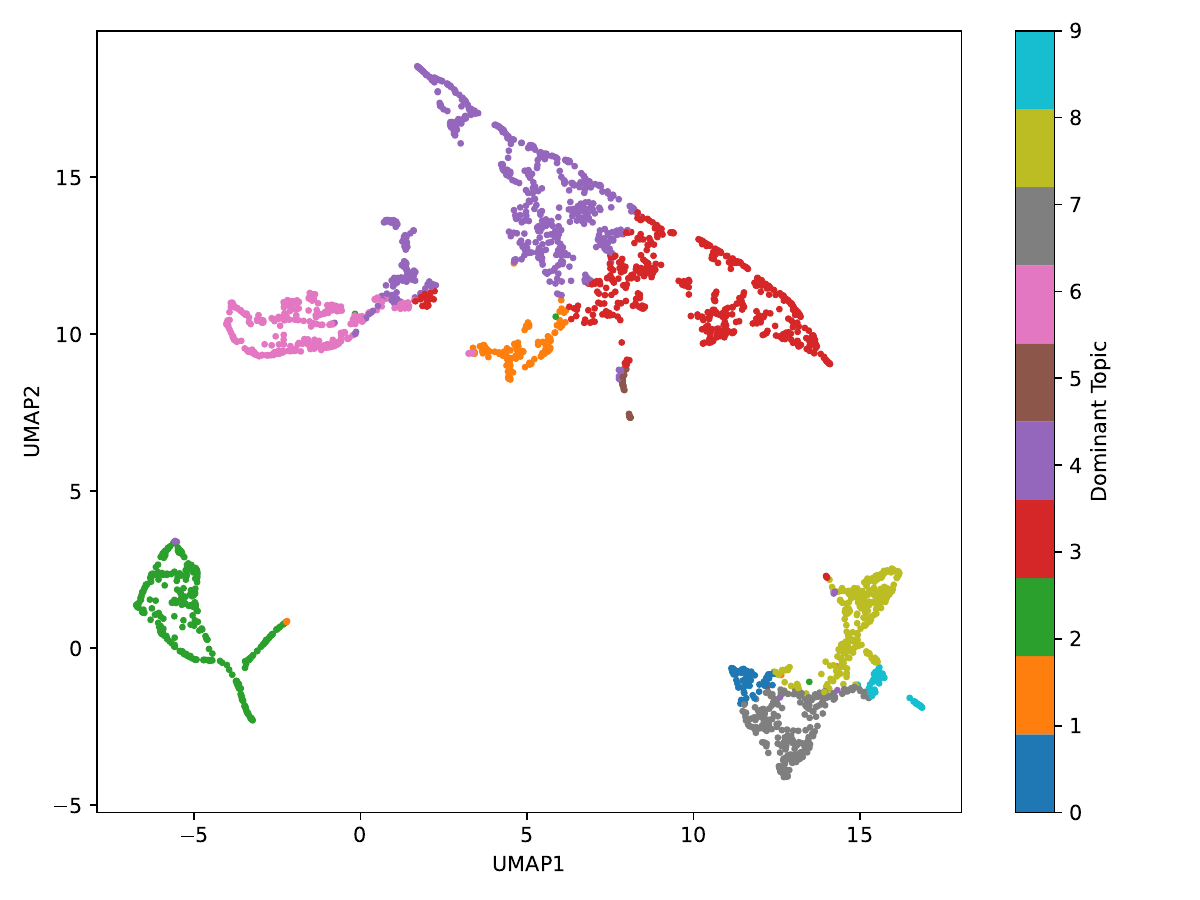}
        \caption{GDir-VI}
    \end{subfigure}
    \hspace{0.02\textwidth}
    \begin{subfigure}[b]{0.48\textwidth}
        \centering
        \includegraphics[width=\linewidth]{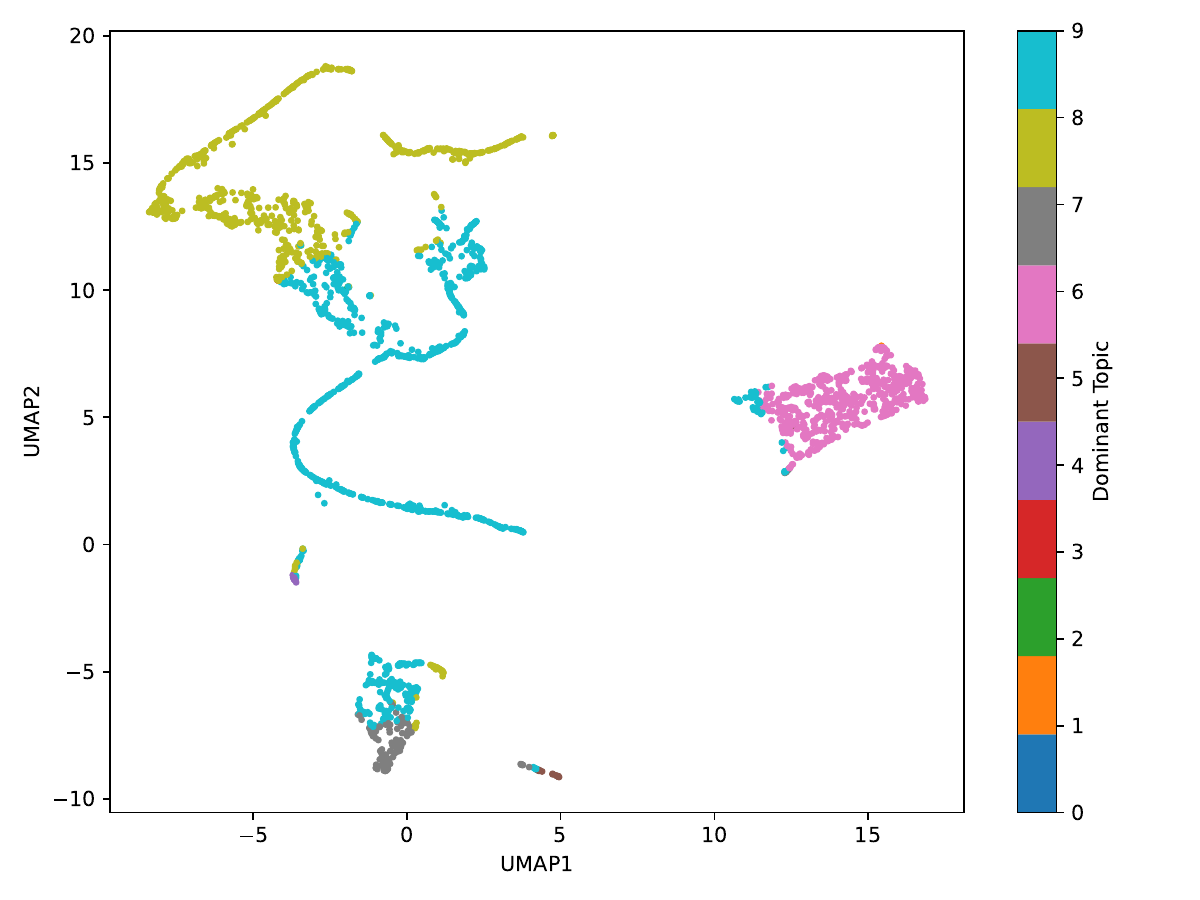}
        \caption{GDir-EP}
    \end{subfigure}

    \caption{UMAP of topic proportion of single-cells}
    \label{fig:pbmc3k}
\end{figure}

We further examine our models with an image classification task.
We choose 15 Scene Categories dataset, and pick five categories out of them: forest, highway, kitchen, office, and tallbuilding.

After loading the five categories of images, their feature descriptors are extracted with SIFT(Scale-Invariant Feature Transform), which are roughly distributed in a continuous space.
Thus, we use K-MEANS on all obtained descriptors to form bag-of-visualword (BoVW) vocabulary.
Then, a topic model with a certain prior-inference combination is trained on the BoVM, and the topic proportions of each image are obtained.
Subsequently, an SVM classifier is employed to perform classifications based on the topic proportions of all images.

In our experiment, we set the grid search based on the size of the BoVM and the topic number.
We choose from 100 to 500 for the size of the BoVW, and 10 to 50 for the topic number.
Finally, we obtain a total of 150 models of different settings, and compute their classification accuracies.
Figure \ref{fig:3dbar} is the 3D-plot of the image classification accuracy results.
We achieve a maximum accuracy of 85\% from our models.

\subsection{An Application in Bioinformatics}

In this section, we test our models using a well-known task in bioinformatics, namely single-cell RNA sequencing (scRNA-seq).
We use the famous dataset of Peripheral Blood Mononuclear Cells (PBMC) from 10X Genomics, which contains 2700 single cells that were sequenced on the Illumina NextSeq 500.
Each cell is equivalent to a document in our textual experiments, and the collection of genes of a single cell is equivalent to the collection of words in a document.
We firstly preprocess the PBMC dataset, and train multiple models with number of topics set to 10, as it is reported as the effective setting.
Figure \ref{fig:pbmc3k} is the UMAPs of topic proportion of all single-cells obtained from our models.

\section{Discussion}

This paper focuses on developing a general framework, called Latent Dirichlet-Tree Allocation (LDTA), to extend the conventional Latent Dirichlet Allocation (LDA) model both theoretically and practically.
Despite being overwhelmingly successful in modeling all kinds of discrete data, LDA employs the classical Dirichlet distribution as the prior for the topic proportions, which has intrinsic limitations such as strictly negative covariance matrix.
This shortcoming consequently imposes a structural restriction on LDA's capability of capturing the correlations and hierarchical relationships among topics.
To alleviate this restriction, we generalize the Dirichlet distribution to a broader family of Dirichlet-Tree distributions.
While maintaining the properties of conjugacy and exponential form, the Dirichlet-Tree's expressiveness is significantly enhanced thanks to its flexible hierarchical structure.
LDTA therefore extends the modeling capacity of LDA while preserving its generative semantic and interpretability.

In Section 2, we formally derive the Dirichlet-Tree distribution and demonstrate the key properties that underpin the subsequent derivation of Bayesian inference algorithm.
In particular, we prove the exponential form of Dirichlet-Tree for the first time, and we summarize the three equivalent representations: the node form, the general form, and the exponential form.
Under the Dirichlet-Tree framework, the classical Dirichlet, the Generalized Dirichlet, and the Beta-Liouville proposed in earlier literature are all unified as special cases of Dirichlet-Tree corresponding to different tree-structure.
To further prepare the development of the inference algorithm, we introduce several new concepts including Dirichlet selection operator, Bayesian operator and the collection of derived distributions.
These concepts provides new insight into the conjugacy and Bayesian updating while greatly simplifying the description of approximate Bayesian inference procedures.

In Section 3, 4 and 5, we introduce the generative process of LDTA, and derive and compare two universal inference algorithms, and Mean-Field Variational Inference (MFVI) and Expectation Propagation (EP).
We use the word ``universal" because the algorithms are derived and formulated for any arbitrary Dirichlet-Tree prior.
The two algorithms aim to approximate the intractable posteriors from different perspectives.
MFVI employs a global mean-field approximating distribution and optimizes the approximating parameters by maximizing the lower bound of evidence, while EP, in contrast, factorizes the target posterior into multiple terms and approximates them by a corresponding group of approximators step by step.
Throughout the formulation, we particularly highlight the vectorized nature of the two algorithms by expressing their computational steps using tensor operations.
And finally in Section 6, we implement LDTA with three different Dirichlet-Tree priors and both inference algorithms across three different application scenarios.
The applications and experiments are accelerated using Pytorch on GPU, and the experimental results demonstrate feasibility and scalability of our model.

Despite these advantages, LDTA also brings further challenges and opportunities for improvement. 
One limitation of the current work is that the tree structure must be specified in advance. 
While this gives the model users the flexibility to choose the desired tree structure, it remains an important question of how in general the tree should ``grow" itself given data where no obvious tree structure appears. 
Automatic structure learning, such as inferring the tree topology from data or using a flexible nonparametric prior, represents a promising direction for future research.
Another future direction calls the improvement and generalization of KL-divergence to measure the similarity between the approximating distributions and the target posteriors. 
Both MFVI and EP are based on KL-divergence which is known to be second-order approximation of Fisher-Rao metric \citep{Amari2016Information}.
Moreover, the global mean-field approximating distribution in MFVI can have more generalized and adaptive inner structure based on statistical optimal transport \citep{ChewiNilesWeedRigollet2025SOT, WuBlei2024_XiVI}.
In this viewpoint, the mean-field is a special case of a more general family of approximating distributions.
In summary, this thesis demonstrates a meaningful step in this research field, and more potential improvements are waiting to be discovered and developed.


\acks{The completion of this research was made possible thanks to the Natural Sciences and Engineering Research Council of Canada (NSERC).}


\newpage

\appendix

\section{Exponential Family and KL-Divergence}\label{app:properties_expo}

The exponential family of distributions has the following form \citep{bishop:06}:
\begin{equation*}
    p(\mathbf{x}|\boldsymbol{\eta})
    = h(\mathbf{x}) \mathrm{exp} \{ \boldsymbol{\eta}^{\top} \mathbf{u(x)} - \log g(\boldsymbol{\eta}) \}
\end{equation*}
where $\mathbf{u(x)}$ is the sufficient statistics; $\boldsymbol{\eta}$ is the natural parameter; $h(\mathbf{x})$ is the underlying measure; and $g(\boldsymbol{\eta})$ is the normalizer:
\begin{equation*}
        g(\boldsymbol{\eta}) = \int_{\mathbf{x}} \exp \{ \boldsymbol{\eta}^{\top} \mathbf{u(x)} \} h(\mathbf{x}) \,\mathrm{d}\mathbf{x}
\end{equation*}
\begin{theorem}[Expectation of sufficient statistics]
    \label{Moment on Sufficient Statistics}
    Let probability distribution $p(\mathbf{x}|\boldsymbol{\eta})$ belong to exponential family. The expectation of its sufficient statistics $\mathbf{u(x)}$ is equal to the gradient of its log-normalizer $\log g(\boldsymbol{\eta})$ w.r.t. its natural parameter $\boldsymbol{\eta}$.
    \begin{equation*}
        \nabla_{\boldsymbol{\eta}} \log g(\boldsymbol{\eta}) = \frac{1}{g(\boldsymbol{\eta})} \nabla_{\boldsymbol{\eta}} \,g(\boldsymbol{\eta}) = \int_{\mathbf{x}} \frac{1}{g(\boldsymbol{\eta})} \mathrm{exp} \{ \boldsymbol{\eta}^{\top} \mathbf{u(x)} \} \mathbf{u(x)} h(\mathbf{x}) \,\mathrm{d}\mathbf{x} = \mathbb{E}\left[ \mathbf{u(x)} \right]
    \end{equation*}
\end{theorem}

\begin{theorem}[ELBO Maximization]
    Suppose $p(\mathbf{z}|\boldsymbol{w})$ is a target posterior given observation $\boldsymbol{w}$; $q(\mathbf{z}|\boldsymbol{\eta})$ is the approximating distribution.
    The optimized approximating distribution $q(\mathbf{z}|\boldsymbol{\eta}^{*})$ to the target posterior $p(\mathbf{z}|\boldsymbol{w})$ measured by minimizing KL-Divergence $\mathrm{KL}(q(\mathbf{z}|\boldsymbol{\eta})||p(\mathbf{z}|\boldsymbol{w}))$ is the one that maximizes the Evidence Lower Bound (ELBO) given as:
    \begin{equation*}
        \mathrm{ELBO} = \int_{\mathbf{z}} q(\mathbf{z}|\boldsymbol{\eta})\log \frac{p(\boldsymbol{w}, \mathbf{z})}{q(\mathbf{z}|\boldsymbol{\eta})} \,\mathrm{d}\mathbf{z}
    \end{equation*}
\end{theorem}
\begin{proof}
    Given that $p(\mathbf{z}|\boldsymbol{w})$ is the target posterior and $q(\mathbf{z}|\boldsymbol{\eta})$ belongs to exponential family, the Kullback-Leibler divergence $\mathrm{KL}(q(\mathbf{z}|\boldsymbol{\eta})||p(\mathbf{z}|\boldsymbol{w}))$ w.r.t. $q(\mathbf{z}|\boldsymbol{\eta})$ has the following form:
    \begin{equation*}
    \begin{aligned}
        \mathrm{KL}(q(\mathbf{z}|\boldsymbol{\eta})||p(\mathbf{z}|\boldsymbol{w}))
        &= \int_{\mathbf{z}} q(\mathbf{z}|\boldsymbol{\eta}) \log \frac{q(\mathbf{z}|\boldsymbol{\eta})}{p(\mathbf{z}|\boldsymbol{w})} \,\mathrm{d}\mathbf{z} \\[1ex]
        &= \log p(\boldsymbol{w}) - \int_{\mathbf{z}} q(\mathbf{z}|\boldsymbol{\eta}) \log \frac{p(\boldsymbol{w}, \mathbf{z})}{q(\mathbf{z}|\boldsymbol{\eta})} \,\mathrm{d}\mathbf{z} \\[1ex]
        &= \log p(\boldsymbol{z}) - \mathrm{ELBO}
    \end{aligned}
    \end{equation*}
    where $\log p(\boldsymbol{w})$ is independent of $\boldsymbol{\eta}$, and $\mathrm{KL}(q(\mathbf{z}|\boldsymbol{\eta})||p(\mathbf{z}|\boldsymbol{w}))\ge 0$. 
    The theorem is thus proved.
\end{proof}

\begin{theorem}[Moment Matching]\label{thm:moment_matching}
    Let $q(\mathbf{x}|\boldsymbol{\eta})$ belong to the exponential family.
    The optimized approximating distribution $q(\mathbf{x}|\boldsymbol{\eta}^{*})$ to a target distribution $p(\mathbf{x})$ measured by minimizing KL-Divergence $\mathrm{KL}(p(\mathbf{x})||q(\mathbf{x}|\boldsymbol{\eta}))$ is the one under which the expectation of its own sufficient statistics is equal to the expectation of that sufficient statistics under the target distribution.
\end{theorem}
\begin{proof}
    Given that $p(\mathbf{x})$ is a target distribution and $q(\mathbf{x}|\boldsymbol{\eta})$ is the approximating distribution and $q(\mathbf{x}|\boldsymbol{\eta})$ belongs to the exponential family. 
    The Kullback-Leibler divergence $\mathrm{KL}(p(\mathbf{x})||q(\mathbf{x}|\boldsymbol{\eta}))$ w.r.t. $q(\mathbf{x}|\boldsymbol{\eta})$ has the following simple form:
    \begin{equation*}
    \begin{split}
        \mathrm{KL}(p(\mathbf{x})||q(\mathbf{x}|\boldsymbol{\eta}))
        & = \int_{\mathbf{x}} p(\mathbf{x}) \log \frac{p(\mathbf{x})}{q(\mathbf{x}|\boldsymbol{\eta})} \,\mathrm{d}\mathbf{x}  \\[1ex]
        & = \log g(\boldsymbol{\eta}) - \int_{\mathbf{x}} p(\mathbf{x}) \left\{\boldsymbol{\eta}^{\top}\mathbf{u(x)}\right\} \,\mathrm{d}\mathbf{x} + \int_{\mathbf{x}} p(\mathbf{x}) \log \frac{p(\mathbf{x})}{h(\mathbf{x})} \,\mathrm{d}\mathbf{x}  \\[1ex]
        & = \log g(\boldsymbol{\eta}) - \boldsymbol{\eta}^{\top} \mathbb{E}_{p(\mathbf{x})} \left[\mathbf{u(x)}\right] + \mathrm{const.}
    \end{split}
    \end{equation*}
    where the constant is independent of $\boldsymbol{\eta}$.
    The gradient is given as:
    \begin{equation*}
        \nabla_{\boldsymbol{\eta}} \mathrm{KL} (\boldsymbol{\eta}) = \mathbb{E}_{q(\mathbf{x})} \left[ \mathbf{u(x)} \right] - \mathbb{E}_{p(\mathbf{x})} \left[ \mathbf{u(x)} \right]
        \label{eq:moment_matching}
    \end{equation*}
    Consider minimizing the KL divergence w.r.t. $\boldsymbol{\eta}$ by setting the gradient to zero.
    This complete the proof of the theorem.
\end{proof}

\section{Details of Moment Matching}

In previous sections, we demonstrate that both maximization of ELBO with respect to $\boldsymbol{\xi}$ in mean-field variational inference and message passing in expectation propagation inference are reduced to moment matching.
More specifically, we want to ``recover" a Dirichlet distribution's parameters given its expectation of sufficient statistics.
\cite{minka:00} introduces the more general case of maximum likelihood estimate of a Dirichlet, and fix-point iteration and Newton-Ralphson algorithm for solving the problem.
\cite{sklar2014fast} tackles the problem with a fast Newton method.
In this appendix, we formulate the problem and introduce the fast Newton method that we used in the implementation of this paper, based on \cite{sklar2014fast} with slight differences in formulation and implementation.
In the end, we show that the same method can be generalized to any kind of Dirichlet-Tree.

Suppose that we have a Dirichlet distribution $\mathrm{Dir}(\boldsymbol{\theta}|\boldsymbol{\alpha})$ with unknown parameter $\boldsymbol{\alpha}$, and we want to compute the parameter from the known expectation of sufficient statistics $\boldsymbol{u}$.
For convenience, we define the function $\xi: \mathbb{R}_{<0}^{K} \to \mathbb{R}_{>0}^{K}$, such that $\boldsymbol{\alpha} = \xi(\boldsymbol{u})$.
This is equivalent to solving a system of non-linear equations:
\begin{equation*}
    \Biggl\{\psi(\alpha_k) - \psi \left( \sum_{l=1}^{K} \alpha_l \right) - u_k = 0 \Biggr\}_{k=1}^{K}
\end{equation*}
A natural idea is to use fixed-point iteration, as indicated in \citep{minka:00}.
Starting from an initial guess $\boldsymbol{\alpha}_0$, the optimized $\boldsymbol{\alpha}$ can be computed iteratively:
\begin{equation*}
    \Biggl\{\alpha_k^{(t+1)} = \psi^{-1}\left(  \psi \left( \sum_{l=1}^{K} \alpha_l^{(t)} \right) + u_k \right) \Biggr\}_{k=1}^{K}
\end{equation*}
This approach is computationally inefficient and may require many iterations in some cases. 
Therefore, we employ a fast Newton-Raphson method in our implementation.
We denote $\alpha_0 = \sum_{l=1}^{K} \alpha_l$ and:
\begin{align*}
    d(\boldsymbol{\alpha}) &= 
        \left(
            \begin{array}{c}
                d_1  \\
                d_2  \\
                \vdots  \\
                d_K
            \end{array}
        \right) = 
        \left(
            \begin{array}{c}
                \psi(\alpha_1) - \psi(\alpha_0) - u_1  \\
                \psi(\alpha_2) - \psi(\alpha_0) - u_2  \\
                \vdots  \\
                \psi(\alpha_K) - \psi(\alpha_0) - u_K
            \end{array}
        \right) \label{eq:grad} \\[1ex]
    J(\boldsymbol{\alpha}) &= 
        \left(
        \begin{array}{cccc}
            \frac{\partial}{\partial \alpha_1} d_1 & \frac{\partial}{\partial \alpha_1} d_2 & \cdots & \frac{\partial}{\partial \alpha_1} d_K \\
            \frac{\partial}{\partial \alpha_2} d_1 & \frac{\partial}{\partial \alpha_2} d_2 & \cdots & \frac{\partial}{\partial \alpha_2} d_K  \\
            \vdots & \vdots & \ddots & \vdots  \\
            \frac{\partial}{\partial \alpha_K} d_1 & \frac{\partial}{\partial \alpha_K} d_2 & \cdots & \frac{\partial}{\partial \alpha_K} d_K
        \end{array}
        \right) \notag \\[1ex]
    &= 
        \left(
        \begin{array}{cccc}
            \psi^{\prime}(\alpha_1)-\psi^{\prime}(\alpha_0) & -\psi^{\prime}(\alpha_0) & \cdots & -\psi^{\prime}(\alpha_0)  \\
            -\psi^{\prime}(\alpha_0) & \psi^{\prime}(\alpha_2)-\psi^{\prime}(\alpha_0) & \cdots & -\psi^{\prime}(\alpha_0)  \\
            \vdots & \vdots & \ddots & \vdots  \\
            -\psi^{\prime}(\alpha_0) & -\psi^{\prime}(\alpha_0) & \cdots & \psi^{\prime}(\alpha_K)-\psi^{\prime}(\alpha_0)
        \end{array}
        \right) \notag \\[1ex]
    &= \operatorname{diag}(\psi^{\prime}(\boldsymbol{\alpha})) - \psi^{\prime}(\alpha_0) \mathbf{1}_K \mathbf{1}_K^{\top}
\end{align*}
where $\psi^{\prime}(\cdot)$ is the trigamma function.
The Newton's method is given as:
\begin{equation*}
    \boldsymbol{\alpha}^{(t+1)} = \boldsymbol{\alpha}^{(t)} - J(\boldsymbol{\alpha}^{(t)})^{-1} d(\boldsymbol{\alpha}^{(t)})
\end{equation*}
Here, the key is to compute the inverse of Jacobian (or Hessian in the problem of MLE of a Dirichlet).
It should be noted that $J(\boldsymbol{\alpha})$ equals a diagonal matrix plus a constant matrix.
Sherman-Morrison formula, which is a special case of the Woodbury matrix inversion lemma, provides a perfect analytic solution for the inversion of this special kind of matrix.
According to the Sherman-Morrison formula, if $A \in \mathbb{R}^{K \times K}$ is an invertible matrix, $\boldsymbol{u}, \boldsymbol{v} \in \mathbb{R}^K$ are two vectors, we have:
\begin{equation*}
    \left( A + \boldsymbol{u}\boldsymbol{v}^{\top} \right)^{-1} = A^{-1} - \frac{A^{-1}\boldsymbol{u}\boldsymbol{v}^{\top}A^{-1}}{1 + \boldsymbol{v}^{\top}A^{-1}\boldsymbol{u}}
\end{equation*}
We replace $A$ with $\operatorname{diag}(\psi^{\prime}(\boldsymbol{\alpha}))$, and $\boldsymbol{u}\boldsymbol{v}^{\top}$ with $- \psi^{\prime}(\alpha_0) \mathbf{1}_K \mathbf{1}_K^{\top}$.
After a few steps, we arrive at the following expression:
\begin{equation*}
    J(\boldsymbol{\alpha})^{-1} = \operatorname{diag} \left( \boldsymbol{\pi}(\boldsymbol{\alpha}) \right) - \frac{\boldsymbol{\pi}(\boldsymbol{\alpha}) \boldsymbol{\pi}(\boldsymbol{\alpha})^{\top}}{\sigma(\boldsymbol{\alpha})}
\end{equation*}
where
\begin{align*}
    \boldsymbol{\pi}(\boldsymbol{\alpha}) &=
    \left( \frac{1}{\psi^{\prime}(\alpha_1)}, \; \frac{1}{\psi^{\prime}(\alpha_2)}, \; \cdots, \; \frac{1}{\psi^{\prime}(\alpha_K)}
    \right)^{\top} \notag \\[1ex]
    \sigma(\boldsymbol{\alpha}) &= \left( \sum_{l=1}^{K} \frac{1}{\psi^{\prime}(\alpha_l)} \right) - \frac{1}{\psi^{\prime}(\alpha_0)} \notag
\end{align*}
We proceed to derive the update in each iteration:
\begin{equation*}
    \boldsymbol{\alpha}^{(t+1)} = \boldsymbol{\alpha}^{(t)} - \left( d(\boldsymbol{\alpha}^{(t)}) - \frac{\boldsymbol{\pi}(\boldsymbol{\alpha}^{(t)})^{\top} d(\boldsymbol{\alpha}^{(t)})}{\sigma(\boldsymbol{\alpha}^{(t)})} \mathbf{1}_K \right) \odot \boldsymbol{\pi}(\boldsymbol{\alpha}^{(t)})
\end{equation*}
In Dirichlet-Tree distribution, the parameter $\{\xi_{t\mid s}\}_{c(s)},\;s\in\boldsymbol{\Lambda}$ under each internal node is treated as an independent sub-group of Dirichlet parameters that results in a sub-group of equations, and solved locally.

\section{Examples of Dirichlet-Tree}\label{app:examples_DT}

In this section, we introduce three typical Dirichlet-Tree distributions: the classical Dirichlet, the Beta-Liouville, and the Generalized Dirichlet.

\begin{figure}[htbp]
    \centering
    \includegraphics[width=1.0\linewidth]{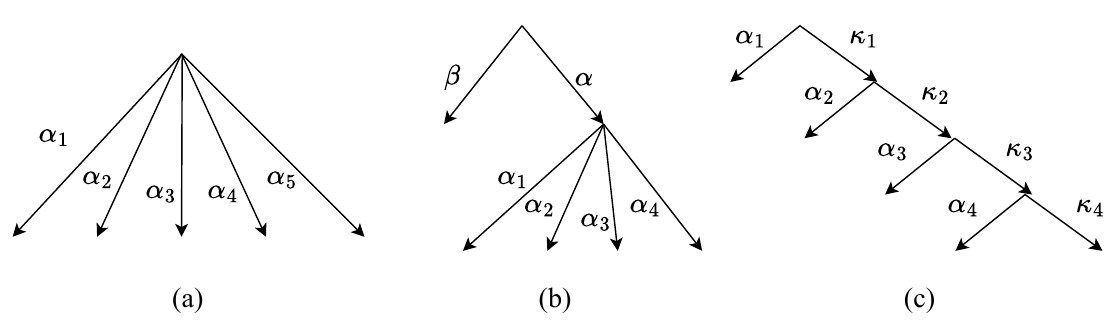}
    \caption{(a) Classical Dirichlet; (b) Beta-Liouville; (c) Generalized Dirichlet}
    \label{fig:d_bl_gd}
\end{figure}

\subsection{Dirichlet}
The Dirichlet distribution is the most basic Dirichlet-Tree (Figure~\ref{fig:d_bl_gd} (a)).
Given the parameter $\boldsymbol{\alpha} = (\alpha_1, \dots, \alpha_K)$, the PDF of Dirichlet is given as:
\begin{equation*}
    p(\boldsymbol{\theta}|\boldsymbol{\alpha}) = \frac{\Gamma\left(\sum_{l=1}^{K}\alpha_l\right)}{\prod_{k=1}^{K}\Gamma\left(\alpha_k\right)} \prod_{k=1}^{K} {\theta_k}^{\alpha_k -1}
\end{equation*}
Dirichlet's mean, variance, and covariance are given as:
\begin{equation*}
\begin{aligned}
    &\mathbb{E}[\theta_k] = \frac{\alpha_k}{\sum_{l=1}^{K}\alpha_l} \\[1ex]
    &\mathrm{Var}(\theta_k) = \frac{\alpha_k\left(\sum_{l=1}^{K}\alpha_l - \alpha_k\right)}{{\left(\sum_{l=1}^{K}\alpha_l\right)}^2\left(\sum_{l=1}^{K}\alpha_l+1\right)} \\[1ex]
    &\mathrm{Cov}(\theta_i, \theta_j) = -\,\frac{\alpha_i \alpha_j}{\left(\sum_{l=1}^{K}\alpha_l\right)^2\left(\sum_{l=1}^{K}\alpha_l+1\right)}
\end{aligned}    
\end{equation*}

\subsubsection{Exponential Form}

The natural parameters follows:
\begin{equation*}
    \eta_k = \alpha_k -1, \quad k = 1,\dots,K
\end{equation*}
The sufficient statistics:
\begin{equation*}
    u_k = \log \theta_k, \quad k=1, \dots, K
\end{equation*}
The log-normalizer:
\begin{equation*}
    \log g(\boldsymbol{\eta}) = \sum_{k=1}^{K} \log \Gamma(\alpha_k) - \log \Gamma \left( \sum_{l=1}^{K} \alpha_l \right)
\end{equation*}
The expectation of sufficient statistics:
\begin{equation*}
    \mathbb{E}\left[u_k\right] = \psi(\alpha_k) - \psi\left(\sum_{l=1}^{K}\alpha_l\right), \quad k = 1, \dots, K
\end{equation*}

\subsubsection{Conjugacy}

The Dirichlet distribution is conjugate to multinomial.
Given multinomial observation $\boldsymbol{n} = (n_1, \dots, n_K)$, the parameters of the Dirichlet posterior is updated as:
\begin{equation*}
    {\alpha_k}^{\prime} = \alpha_k + n_k, \quad k=1, \dots, K
\end{equation*}

\subsection{Beta-Liouville}

Given the parameter $\boldsymbol{\xi} = (\alpha, \beta, \alpha_1, \dots, \alpha_{K-1})$, the PDF of Beta-Liouville (Figure~\ref{fig:d_bl_gd} (b)) is given as:
\begin{equation*}
    p(\boldsymbol{\theta} | \boldsymbol{\xi}) = \frac{\Gamma(\sum_{l=1}^{K-1}\alpha_l)\Gamma(\alpha + \beta)}{\Gamma(\alpha)\Gamma(\beta)}\prod_{k=1}^{K-1}\frac{\theta_k^{\alpha_k - 1}}{\Gamma(\alpha_k)}
    \left(\sum_{l=1}^{K-1}\theta_l\right)^{\alpha-\sum_{l=1}^{K-1}\alpha_l}\left(1-\sum_{l=1}^{K-1}\theta_l\right)^{\beta-1}
\end{equation*}
Beta-Liouville also belongs to Liouville family of distributions whose general form is:
\begin{equation*}
    p(\boldsymbol{\theta}|\boldsymbol{\xi}) = \frac{f(u|\zeta)}{u^{\sum_{l=1}^{K-1}\alpha_l-1}} \frac{\Gamma(\sum_{l=1}^{K-1}\alpha_l)}{\prod_{k=1}^{K-1}\Gamma(\alpha_k)} \prod_{k=1}^{K-1}\theta_{k}^{\alpha_k-1}
\end{equation*}
where $u = \sum_{l=1}^{K-1} \theta_l$ and $f(u|\zeta)$ is the generating density with parameter $\zeta$.
Particularly, Beta-Liouville employs Beta distribution as $f(u|\zeta)$.
When $\alpha = \sum_{k=1}^{K-1} \alpha_k$ and $\beta = \alpha_K$, Beta-Liouville descends to the conventional Dirichlet distribution.
Beta-Liouville's mean, variance, and covariance matrix are given as:
\begin{align*}
    & \mathbb{E}\left[\theta_{k}\right] = \frac{\alpha}{\alpha + \beta} \frac{\alpha_{k}}{\sum_{l=1}^{K-1} \alpha_{l}}  \\[1ex]
    & \mathrm{Var}(\theta_{k}) = {\left( \frac{\alpha}{\alpha + \beta} \right)}^2 \frac{\alpha_{k} (\alpha_{k}+1)}{(\sum_{l=1}^{K-1} \alpha_{l}) (\sum_{l=1}^{K-1}\alpha_{l}+1)} - \left(\frac{\alpha_k E(\theta_k)}{\sum_{l=1}^{K-1} \alpha_l}\right)^2 \\[1ex]
    & \mathrm{Cov}(\theta_p, \theta_q) = \frac{\alpha_p \alpha_q}{\sum_{l=1}^{K} \alpha_l} \left( \frac{\frac{\alpha + 1}{\alpha + \beta + 1} \frac{\alpha}{\alpha + \beta}}{\sum_{l=1}^{K-1}\alpha_l + 1} - \frac{\frac{\alpha}{\alpha + \beta}}{\sum_{l=1}^{K-1} \alpha_l} \right) 
\end{align*}

\subsubsection{Exponential Form}

The natural parameters are as follows:
\begin{equation*}
\begin{aligned}
    &\eta_k = \alpha_k -1, \quad k = 1,\dots,K-1 \\[1ex]
    &\eta_K = \alpha - K + 1 \\[1ex]
    &\eta_{K+1} = \beta - 1 \\[1ex]
\end{aligned}
\end{equation*}
The sufficient statistics:
\begin{align*}
    & u_k = \log \theta_k - \log \left( \sum_{l=1}^{K-1} \theta_l \right), \quad k=1, \dots, K-1  \\[1ex]
    & u_K = \log \left( \sum_{l=1}^{K-1} \theta_l \right)    \\[1ex]
    & u_{K+1} = \log \left( 1 - \sum_{l=1}^{K-1} \theta_l \right)
\end{align*}
The log-normalizer:
\begin{equation*}
    \log g(\boldsymbol{\eta}) = \sum_{k=1}^{K-1} \log \Gamma(\alpha_k) + \log \Gamma(\alpha) + \log \Gamma(\beta) - \log \Gamma \left( \sum_{l=1}^{K-1} \alpha_l \right) - \log \Gamma(\alpha + \beta)
\end{equation*}
The expectation of sufficient statistics:
\begin{equation*}
\begin{aligned}
    &\mathbb{E}\left[u_k\right] = \psi(\alpha_k) - \psi\left(\sum_{l=1}^{K-1}\alpha_l\right), \quad k = 1, \dots, K-1 \\[1ex]
    &\mathbb{E} \left[ u_K \right] = \psi(\alpha) - \psi(\alpha + \beta)    \\[1ex]
    &\mathbb{E} \left[ u_{K+1} \right] = \psi (\beta) - \psi (\alpha + \beta)
\end{aligned}
\end{equation*}
And:
\begin{equation*}
\begin{aligned}
    &\mathbb{E} \left[ \log \theta_k \right]= \psi(\alpha_k) - \psi \left( \sum_{l=1}^{K-1} \alpha_l \right) + \psi (\alpha) -\psi (\alpha + \beta), \quad k=1,\dots, K-1  \\[1ex]
    &\mathbb{E} \left[ \log \theta_K \right] = \psi(\beta) - \psi(\alpha+\beta)
\end{aligned}
\end{equation*}

\subsubsection{Conjugacy}

The Beta-Liouville distribution is conjugate to multinomial.
Given multinomial observation $\boldsymbol{n} = (n_1, \dots, n_K)$, the parameters of the Beta-Liouville posterior is updated as:
\begin{align*}
    & {\alpha_k}^{\prime} = \alpha_k + n_k, \quad k=1, \dots, K-1  \\[1ex]
    & \alpha^{\prime} = \alpha + \sum_{l=1}^{K-1} n_l   \\[1ex]
    & \beta^{\prime} = \beta + n_K
\end{align*}

\subsection{Generalized-Dirichlet}

Given the parameters $\boldsymbol{\alpha}=(\alpha_1, \dots, \alpha_{K-1})$ and $\boldsymbol{\kappa} = (\kappa_1, \dots, \kappa_{K-1})$, the PDF of the Generalized-Dirichlet (Figure~\ref{fig:d_bl_gd} (c)) is given as:
\begin{equation*}
    p(\boldsymbol{\theta}|\boldsymbol{\alpha}, \boldsymbol{\kappa}) = \prod_{k=1}^{K-1} \frac{\Gamma(\alpha_k + \kappa_k)}{\Gamma(\alpha_k)\Gamma(\kappa_k)} {\theta_k}^{\alpha_k -1} \left(1-\sum_{l=1}^{k}\theta_l\right)^{\gamma_k}
\end{equation*}
where $\gamma_k = \kappa_k - \alpha_{k+1} - \kappa_{k+1}$ for $k=1,\dots,K-2$ and $\gamma_{K-1} = \kappa_{K-1}-1$.
Generalized-Dirichlet's mean, variance, and covariance are given as:
\begin{equation*}
\begin{aligned}
    &\mathbb{E}\left[\theta_k\right] = \frac{\alpha_k}{\alpha_k + \kappa_k} \prod_{l=1}^{k-1} \frac{\kappa_l}{\alpha_l + \kappa_l} \\[1ex]
    &\mathrm{Var}(\theta_k) = \mathbb{E}[\theta_k] \left( \frac{\alpha_k+1}{\alpha_k+\kappa_k+1} \prod_{l=1}^{k-1} \frac{\kappa_l+1}{\alpha_l+\kappa_l+1} - \mathbb{E}[\theta_k] \right)  \\[1ex]
    &\mathrm{Cov}(\theta_m, \theta_n) = \mathbb{E}[\theta_n] \left( \frac{\alpha_m}{\alpha_m + \kappa_m + 1} \prod_{l=1}^{m-1} \frac{\kappa_l + 1}{\alpha_l + \kappa_l + 1} - \mathbb{E}[\theta_m] \right)
\end{aligned}
\end{equation*}

\subsubsection{Exponential Form}

The natural parameters are as follows:
\begin{equation*}
\begin{aligned}
    &\eta_{1k} = \alpha_k - 1,\quad k = 1,\dots, K-1 \\[1ex]
    &\eta_{2k} = \kappa_k - K + k, \quad k = 1,\dots, K-1
\end{aligned}
\end{equation*}
The sufficient statistics:
\begin{equation*}
\begin{aligned}
    &u_{1k} = \log \theta_k - \log \sum_{l=k}^{K}\theta_l,\quad k = 1,\dots,K-1 \\[1ex]
    &u_{2k} = \log \sum_{l=k+1}^{K} \theta_l - \log \sum_{j=k}^{K} \theta_j, \quad k=1,\dots,K-1
\end{aligned}
\end{equation*}
The log-normalizer:
\begin{equation*}
    \log g(\boldsymbol{\eta}) = \sum_{k=1}^{K-1}\log\Gamma(\alpha_k) + \sum_{k=1}^{K-1}\log\Gamma(\kappa_k) - \sum_{k=1}^{K-1}\log\Gamma(\alpha_k + \kappa_k)
\end{equation*}
The expectation of sufficient statistics:
\begin{equation*}
\begin{aligned}
    &\mathbb{E}[u_{1k}] = \psi(\alpha_k) - \psi(\alpha_k+\kappa_k), \quad k = 1, \dots, K-1 \\[1ex]
    &\mathbb{E}[u_{2k}] = \psi(\kappa_k) - \psi(\alpha_k+\kappa_k), \quad k = 1,\dots, K-1
\end{aligned}
\end{equation*}
And:
\begin{equation*}
\begin{aligned}
    &\mathbb{E}[\log \theta_1] = \psi(\alpha_1) - \psi(\alpha_1+\kappa_1) \\[1ex]
    &\mathbb{E}[\log \theta_k] = \psi(\alpha_k) - \psi(\alpha_k+\kappa_k) + \sum_{l=1}^{k-1}\psi(\kappa_l) - \sum_{j=1}^{k-1}\psi(\alpha_j + \kappa_j),\quad k = 2,\dots,K-1 \\[1ex]
    &\mathbb{E}[\log \theta_K] = \sum_{l=1}^{K-1} \psi(\kappa_l) - \sum_{j=1}^{K-1}\psi(\alpha_j + \kappa_j)
\end{aligned}
\end{equation*}

\subsubsection{Conjugacy}
The Generalized-Dirichlet distribution is conjugate to the multinomial.
Given multinomial observation $\boldsymbol{n} = (n_1, \dots, n_K)$, the parameters of the Generalized-Dirichlet posterior are updated as:
\begin{equation*}
\begin{aligned}
    &\alpha_k^{\prime} = \alpha_k + n_k, \quad k =1, \dots, K-1  \\[1ex]
    &\kappa_k^{\prime} = \kappa_k + \sum_{l=k+1}^{K} n_l, \quad k=1, \dots, K-1
\end{aligned} 
\end{equation*}

\section{Matrix form of Bayesian Theorem}

The matrix form of Bayesian theory provides an intuition for Expectation-Maximization algorithm.
The basic form of Bayesian theorem is defined over binary sample space $\{A, A^{\mathrm{C}}\}$ and $\{B, B^{\mathrm{C}}\}$:
\begin{multline*}
        \left(
        \begin{array}{cc}
            P(A|B) & P(A|B^\mathrm{C}) \\
            P(A^\mathrm{C}|B) & P(A^\mathrm{C}|B^\mathrm{C})
        \end{array}
        \right)
        \left(
        \begin{array}{cc}
            P(B) & 0 \\
            0 & P(B^\mathrm{C})
        \end{array}
        \right) =  \\[1ex]
        \left(
        \begin{array}{cc}
            P(AB) & P(AB^{\mathrm{C}}) \\
            P(A^{\mathrm{C}}B) & P(A^{\mathrm{C}}B^{\mathrm{C}})
        \end{array}
        \right) = \\[1ex]
        \left(
        \begin{array}{cc}
            P(A) & 0 \\
            0 & P(A^\mathrm{C})
        \end{array}
        \right)
        \left(
        \begin{array}{cc}
            P(B|A) & P(B|A^\mathrm{C}) \\
            P(B^\mathrm{C}|A) & P(B^\mathrm{C}|A^\mathrm{C})
        \end{array}
        \right)^{\mathrm{T}}
\end{multline*}
Therefore, we have:
\begin{multline*}
        \left(
        \begin{array}{cc}
            P(A|B) & P(A|B^\mathrm{C}) \\
            P(A^\mathrm{C}|B) & P(A^\mathrm{C}|B^\mathrm{C})
        \end{array}
        \right) = \\[1ex]
        \left(
        \begin{array}{cc}
            P(A) & 0 \\
            0 & P(A^\mathrm{C})
        \end{array}
        \right)
        \left(
        \begin{array}{cc}
            P(B|A) & P(B|A^\mathrm{C}) \\
            P(B^\mathrm{C}|A) & P(B^\mathrm{C}|A^\mathrm{C})
        \end{array}
        \right)^{\mathrm{T}}
        \left(
        \begin{array}{cc}
            P(B) & 0 \\
            0 & P(B^\mathrm{C})
        \end{array}
        \right)^{-1}
\end{multline*}

Now, we generalize Bayesian theorem to multiple sample space. 
Suppose we have two finite sample spaces, $\mathbb{Z} = \{ z_1, \ldots, z_k, \ldots, z_K \}$, and $\mathbb{W} = \{ w_1, \ldots, w_v, \ldots, w_V \}$; and we assign a probabilistic distribution to each of the sample spaces: 

\begin{equation*}
        \boldsymbol{\theta} = (\theta_k)_K = 
        \left(
        \begin{array}{c}
             \theta_1   \\
             \vdots   \\
             \theta_k   \\
             \vdots   \\
             \theta_K
        \end{array}
        \right) = 
        \left(
        \begin{array}{c}
            P(z_1)   \\
            \vdots   \\
            P(z_k)   \\
            \vdots   \\
            P(z_K)
        \end{array}
        \right)
        \quad\;\;\;\;
        \boldsymbol{t} = (t_v)_V =
        \left(
        \begin{array}{c}
            t_1 \\
            \vdots \\
            t_v \\
            \vdots \\
            t_V
        \end{array}
        \right) = 
        \left(
        \begin{array}{c}
            P(w_1) \\
            \vdots \\
            P(w_v) \\
            \vdots \\
            P(w_V)
        \end{array}
        \right)
\end{equation*}
Let $M(\cdot)$ be the transform from a vector to a square diagonal matrix with each of its diagonal component being the corresponding component of the vector, and other components being 0.
\begin{equation*}
        M(\boldsymbol{\theta}) = 
        \left(
        \begin{array}{cccc}
            \theta_1 & 0 & \cdots & 0  \\
            0 & \theta_2 & \cdots & 0  \\
            \vdots & \vdots & \ddots & \vdots \\
            0 & 0 & \cdots & \theta_K
        \end{array}
        \right)
        \quad\;\;\;\;
        M(\boldsymbol{t}) = 
        \left(
        \begin{array}{cccc}
            t_1 & 0 & \cdots & 0 \\
            0 & t_2 & \cdots & 0 \\
            \vdots & \vdots & \ddots & \vdots \\
            0 & 0 & \cdots & t_V
        \end{array}
        \right)
\end{equation*}
Let $\boldsymbol{\varphi} = \{\varphi_{vk}\}_{V \times K}$ and $\boldsymbol{\phi} = \{\phi_{kv}\}_{K \times V}$ be two conditional probability matrices where $\varphi_{vk} = P(w_v|z_k)$ and $\phi_{kv} = P(z_k|w_v)$ for $k = 1, 2, \ldots, K$ and $v = 1, 2, \ldots, V$; and let $\Phi = \{P(w_v, z_k)\}_{V \times K}$ be the joint probabilistic matrix where $P(w_v, z_k) = P(w_v|z_k)P(z_k)$. 
Similarly, $\Phi^{\mathrm{T}} = \{P(z_k, w_v)\}_{K \times V}$ where $P(z_k, w_v) = P(z_k|w_v)P(w_v)$. 
Obviously, we have:
\begin{align*}
        & \boldsymbol{t} = \boldsymbol{\varphi} \boldsymbol{\theta} \\[1ex]
        & \boldsymbol{\theta} = \boldsymbol{\phi} \boldsymbol{t} \\[1ex]
        & \Phi = \boldsymbol{\varphi} M(\boldsymbol{\theta}) \\[1ex]
        & \Phi^{\mathrm{T}} = \boldsymbol{\phi} M(\boldsymbol{t})
\end{align*}
Let $\mathbf{1}_n$ be an n-dimensional vector whose components all equal to 1.
Consequently, we have:
\begin{align*}
        & \Phi \;\; \mathbf{1}_K = \boldsymbol{t} \\[1ex]
        & \Phi^{\mathrm{T}} \mathbf{1}_V = \boldsymbol{\theta}
\end{align*}
\begin{theorem}[Matrix form of Bayesian theorem]
        Given the marginal probability distributions $\boldsymbol{\theta}$ and $\boldsymbol{t}$ respectively corresponding to the conditional probability matrices $\boldsymbol{\varphi}$ and $\boldsymbol{\phi}$, the following transforms exit:

        \begin{align*}
            & \boldsymbol{\varphi} = M(\boldsymbol{t}) \boldsymbol{\phi}^{\mathrm{T}} M(\boldsymbol{\theta})^{-1} \\[1ex]
            & \boldsymbol{\phi} = M(\boldsymbol{\theta}) \boldsymbol{\varphi}^{\mathrm{T}} M(\boldsymbol{t})^{-1}
        \end{align*}
\end{theorem}
\begin{proof}
        \begin{equation*}
            \Phi = \boldsymbol{\varphi} M(\boldsymbol{\theta}) = \left( \boldsymbol{\phi} M(\boldsymbol{t}) \right)^{\mathrm{T}}
        \end{equation*}
        therefore
        \begin{equation*}
            \boldsymbol{\varphi} = M(\boldsymbol{t}) \boldsymbol{\phi}^{\mathrm{T}} M(\boldsymbol{\theta})^{-1}
        \end{equation*}
        The proof of the other equation is analogous.
\end{proof}

\section{Details of Explicitization}

For the first step, we define a new operation on matrices.
\begin{definition}[Matrix Exponential]
        Given two matrices $\mathbf{A} = \{a_{ml}\}_{M \times L}$ and $\mathbf{B} = \{b_{ln}\}_{L \times N}$, define $\mathbf{A}$ to the power of $\mathbf{B}$ as $\mathbf{C} = \mathbf{A}^\mathbf{B} = \{c_{mn}\}_{M \times N}$, where:
        \begin{equation*}
            c_{mn} = \prod_{l=1}^{L} {a_{ml}}^{b_{ln}}
        \end{equation*}
\end{definition}
The operation has the following properties:
\begin{enumerate}
        \item $\mathbf{A}^{\mathbf{I}} = \mathbf{A}$
        \item $\mathbf{A}^{\mathbf{B} + \mathbf{C}} = \mathbf{A^B} \odot \mathbf{A^C}$
        \item $(\mathbf{A} \odot \mathbf{B})^{\mathbf{C}} = \mathbf{A^C} \odot \mathbf{B^C}$
        \item $\mathbf{\Phi}^{\mathbf{AB}} = (\mathbf{\Phi^A})^{\mathbf{B}}$ 
\end{enumerate}
where $\mathbf{I}$ is identity matrix, $\odot$ is Hadamard product. 
The properties are obvious except the last one. 
Now, we prove the last one:
\begin{proof}
        Without the loss of generality, consider a special case where $\mathbf{\Phi}$ is a row vector $\boldsymbol{\theta}^{\mathrm{T}} = (\theta_1, \ldots, \theta_k, \ldots, \theta_K)$, $\mathbf{B}$ is a column vector $\mathbf{n} = (n_v)_V$, and $\mathbf{A}$ is matrix $\boldsymbol{\phi} = \{\phi_{kv}\}_{K \times V}$. We only need to prove that the property is correct for this special case and it can be naturally generalized.

        According to the definition:

        \begin{multline*}
            \left(\boldsymbol{\theta}^{\mathrm{T}}\right)^{\boldsymbol{\phi n}} = \prod_{k=1}^{K} {\theta_k}^{\sum_{v=1}^{V} \phi_{kv} n_v} = \\
            \prod_{k=1}^{K} \left( \prod_{v=1}^{V} {\theta_k}^{\phi_{kv} n_v} \right) = \prod_{v=1}^{V} \left( \prod_{k=1}^{K} {\theta_k}^{\phi_{kv} n_v} \right) = \\
            \prod_{v=1}^{V} \left( \prod_{k=1}^{K} {\theta_k}^{\phi_{kv}} \right)^{n_v} = \left[ \left(\boldsymbol{\theta}^{\mathrm{T}}\right)^{\boldsymbol{\phi}} \right]^{\boldsymbol{n}}
        \end{multline*}
\end{proof}
The definition here is different from the traditional definition of matrix exponentiation in Lie groups and Lie algebra, which requires the base to be square matrix.
The proof explains the origin of implicit approximate posterior in Expectation Porpagation.

Matrix exponential can be viewed as some form of ``inner power", which is similar in spirit to the definition of inner product. 
Now, we can show that the explicitization is more essentially interpreted as a decoration of Bayesian operator caused by the matrix exponential:
\begin{equation*}
\begin{aligned}
        \mathbf{B_D}(\boldsymbol{\phi}\boldsymbol{n_{(v)}})
        &= \left[ \frac{g(\boldsymbol{\eta}(\boldsymbol{\xi}+\mathbf{D}\boldsymbol{\phi}\boldsymbol{n_{(v)}}))}{g(\boldsymbol{\eta}(\boldsymbol{\xi}))} \right]^{-1} \prod_{k=1}^{K} {\theta_{k}}^{\sum_{v=1}^{V} \phi_{kv} n_v} \\[1ex]
        &= \left[ \frac{g(\boldsymbol{\eta}(\boldsymbol{\zeta}_{\boldsymbol{n}}))}{g(\boldsymbol{\eta}(\boldsymbol{\xi}))} \right]^{-1} \left( \boldsymbol{\theta}^{\mathrm{T}} \right)^{\boldsymbol{\phi n}} \\[1ex]
        &= \left[ \frac{g(\boldsymbol{\eta}(\boldsymbol{\zeta}_{\boldsymbol{n}}))}{g(\boldsymbol{\eta}(\boldsymbol{\xi}))} \right]^{-1} \left( \left( \boldsymbol{\theta}^{\mathrm{T}} \right)^{\boldsymbol{\phi}} \right)^{\boldsymbol{n}} \\[1ex]
        &= \left[ \frac{g(\boldsymbol{\eta}(\boldsymbol{\zeta}_{\boldsymbol{n}}))}{g(\boldsymbol{\eta}(\boldsymbol{\xi}))} \right]^{-1} \prod_{v=1}^{V} \left( \prod_{k=1}^{K} {\theta_{k}}^{\phi_{kv}} \right)^{n_v}
\end{aligned}
\end{equation*}

\vskip 0.2in
\bibliography{zheng_ldta}

\end{document}